\documentclass{article}

\usepackage{microtype}
\usepackage{graphicx}
\usepackage{subcaption}
\usepackage{booktabs}

\usepackage{hyperref}

\usepackage[accepted]{icml2024}

\usepackage{amsmath}
\usepackage{amssymb}
\usepackage{mathtools}
\usepackage{amsthm}

\usepackage[capitalize,noabbrev]{cleveref}

\usepackage[english]{babel}
\usepackage[utf8]{inputenc}
\usepackage{parskip}
\usepackage{comment}       
\usepackage{setspace}

\usepackage{amsmath,amssymb,amsfonts,mathrsfs,bbm}
\usepackage{amsthm}
\usepackage{thmtools}
\usepackage{thm-restate}
\usepackage{bm}

\usepackage[capitalize]{cleveref}
\usepackage{xr}
\usepackage{url}

\usepackage{enumitem}
\usepackage{array}
\usepackage{mathtools}
\usepackage[h]{esvect}
\usepackage{array}
\usepackage{minitoc}

\definecolor{brightPurple}{RGB}{170, 51, 119}
\definecolor{vibrantBlue}{RGB}{0, 119, 187}
\definecolor{brightGreen}{RGB}{34, 136, 51}

\newcommand\calS{\mathcal{S}}
\newcommand\calA{\mathcal{A}}
\newcommand\calX{\mathcal{X}}

\newcommand\entropy{\mathcal{H}}
\newcommand\regret{\mathcal{R}}
\newcommand\calY{\mathcal{Y}}

\newcommand\reals{\mathbb{R}}
\newcommand\exptn{\mathbb{E}}
\newcommand\indic{\mathbbm{1}}
\newcommand\interior[1]{\text{int}\round{X}}
\newcommand\simplex[1]{\Delta\left(#1\right)}
\newcommand\proj[2]{\text{proj}_{#1}\left(#2\right)}
\newcommand\kl{\text{KL}}

\newcommand\bregman[3]{B_{#1}(#2,#3)}
\newcommand\bregmanPsi[2]{\bregman{\psi}{#1}{#2}}
\newcommand\biggVert{~\bigg\vert~}
\newcommand\dotProd[2]{\left\langle #1, #2 \right\rangle}
\newcommand\sigmaAlg[1]{\mathcal{F}_{#1}}

\newcommand\piBar{\bar{\pi}}
\newcommand\piStar{\pi^\star}
\newcommand\regPiStar{\pi_{\tau}^\star}
\newcommand\regPiStarH[1]{\pi_{\tau,#1}^\star}

\newcommand\lambdaStar{\bm{\lambda}^\star}
\newcommand\lambdaStarI[1]{\lambda_{#1}^\star}
\newcommand\regLambdaStar{\bm{\lambda}_{\tau}^\star}
\newcommand\regLambdaStarI[1]{\lambda_{\tau, #1}^\star}

\newcommand\lambdaMax{\lambda_{max}}

\newcommand\xStar{x^\star}
\newcommand\xTilde{\tilde{x}}
\newcommand\yStar{y^\star}

\newcommand\pBar{\bar{p}}

\newcommand\xBar{\bar{x}}
\newcommand\yBar{\bar{y}}
\newcommand\fBar{\bar{f}}
\newcommand\fUnderline{\underline{f}}

\newcommand\rHat{\hat{r}}
\newcommand\uHat{\hat{u}}
\newcommand\gHat{\hat{g}}
\newcommand\pHat{\hat{p}}
\newcommand\rBar{\bar{r}}
\newcommand\uBar{\bar{u}}
\newcommand\gBar{\bar{g}}
\newcommand\zHat{\hat{z}}
\newcommand\psiHat{\hat{\psi}}

\newcommand\upOne[1]{#1 \vee 1}
\newcommand\counter{n_{k-1,h}(s,a)}
\newcommand\counterH{n_{k-1,h}(s_h,a_h)}
\newcommand\counterHK{n_{k-1,h}(s_h^{k},a_h^{k})}
\newcommand\counterMax{\upOne{\counter}}
\newcommand\counterMaxH{\upOne{\counterH}}
\newcommand\counterMaxHK{\upOne{\counterHK}}

\newcommand\curly[1]{\left\{#1\right\}}
\newcommand\round[1]{\left(#1\right)}
\newcommand\rectangular[1]{\left[#1\right]}
\newcommand\norm[1]{\left\|#1\right\|}
\newcommand\abs[1]{\left|#1\right|}

\newcommand\stepSet{[H]}
\newcommand\conSet{[I]}

\newcommand\cTauLambda{C_{\eta,\tau,\Lambda}}

\newcommand\dTauLambda{D_{\eta,\tau,\Lambda}}
\newcommand\dTauLambdaPrime{D'_{\tau,\Lambda}}

\newcommand\lagrangian{\mathcal{L}}
\newcommand\regLagrangian{\mathcal{L}_{\tau}}
\newcommand\oTilde[1]{\tilde{O}\round{#1}}

\newcommand\val[3]{V_{#1}^{#2}(#3)}

\newcommand\qval[3]{Q_{#1}^{#2}(#3)}

\newcommand\valHat[3]{\hat{V}_{#1}^{#2}(#3)}
\newcommand\valHatVec[3]{\bm{\hat{V}}_{#1}^{#2}(#3)}

\newcommand\qvalHat[3]{\hat{Q}_{#1}^{#2}(#3)}

\newcommand\valStart[2]{V_{#1}^{#2}}
\newcommand\valStartVec[2]{\bm{V}_{#1}^{#2}}
\newcommand\valStartHat[2]{\hat{V}_{#1}^{#2}}
\newcommand\valStartHatVec[2]{\bm{\hat{V}}_{#1}^{#2}}

\newcommand\occ[3]{d_{#1}^{#2}(#3)}

\newcommand\occFn[2]{\bm{d}_{#1}^{#2}}

\renewcommand{\epsilon}{\ensuremath\varepsilon}

\renewcommand{\phi}{\ensuremath{\varphi}}

\newcommand\St{\mathcal{S}}
\newcommand\A{\mathcal{A}}

\newcommand\E{\mathbb{E}}

\newcommand\M{\mathcal{M}}

\newcommand\R{\mathbb{R}}
\newcommand\Reg{\mathcal{R}}
\newcommand\xbar{\bar{x}}
\newcommand\pibar{\bar{\pi}}

\newcommand\Simplex[1]{\Delta\left(#1\right)}

\newtheorem{remark}{Remark}[section]

\newtheorem{definition}{Definition}[section]
\newtheorem{lemma}{Lemma}[section]
\newtheorem{observation}{Observation}[section]

\newtheorem{assumption}{Assumption}[section]

\usepackage[textsize=tiny]{todonotes}

\icmltitlerunning{Truly No-Regret Learning in Constrained MDPs}

\begin{document}

\twocolumn[
\icmltitle{Truly No-Regret Learning in Constrained MDPs}

\icmlsetsymbol{equal}{*}

\begin{icmlauthorlist}
\icmlauthor{Adrian Müller}{EPFL}
\icmlauthor{Pragnya Alatur}{ETHZ}
\icmlauthor{Volkan Cevher}{EPFL}
\icmlauthor{Giorgia Ramponi}{UZH}
\icmlauthor{Niao He}{ETHZ}
\end{icmlauthorlist}

\icmlaffiliation{EPFL}{EPFL}
\icmlaffiliation{ETHZ}{ETH Zürich}
\icmlaffiliation{UZH}{University of Zürich}

\icmlcorrespondingauthor{Adrian Müller}{adrian.luis.mueller@gmail.com}

\icmlkeywords{Machine Learning, ICML}

\vskip 0.3in
]

\printAffiliationsAndNotice{}

\begin{abstract}
Constrained Markov decision processes (CMDPs) are a common way to model safety constraints in reinforcement learning. State-of-the-art methods for efficiently solving CMDPs are based on primal-dual algorithms. For these algorithms, all currently known regret bounds allow for \textit{error cancellations} --- one can compensate for a constraint violation in one round with a strict constraint satisfaction in another. This makes the online learning process unsafe since it only guarantees safety for the final (mixture) policy but not during learning. As \citet{efroni2020exploration} pointed out, it is an open question whether primal-dual algorithms can provably achieve sublinear regret if we do not allow error cancellations. In this paper, we give the first affirmative answer. We first generalize a result on last-iterate convergence of regularized primal-dual schemes to CMDPs with multiple constraints. Building upon this insight, we propose a model-based primal-dual algorithm to learn in an unknown CMDP. We prove that our algorithm achieves sublinear regret without error cancellations.
\end{abstract}

\section{Introduction}
\label{sec:introduction}

Classical reinforcement learning \citep[RL,][]{sutton2018reinforcement} aims to solve sequential decision-making problems under uncertainty. It involves learning a policy while interacting with an unknown Markov decision process \citep[MDP,][]{bellman1957}. However, in many real-world situations, RL algorithms need to solve the task while respecting certain safety constraints. For example, in autonomous driving and drone navigation, we must avoid collisions and adhere to traffic rules to ensure safe behavior \citep{brunke2022safe}. Such safety requirements are commonly described by constrained Markov decision processes \citep[CMDPs,][]{altman1999constrained}. In CMDPs, the goal is to maximize the expected cumulative reward while subject to multiple safety constraints, each modeled by a different expected cumulative reward signal that needs to lie above a respective threshold. We consider the finite-horizon setting, in which an algorithm chooses a policy in each episode, plays it for one episode, and observes the random transitions, rewards, and constraint rewards along its trajectory.

\looseness -1 In the literature, there are three standard approaches for finding an optimal policy in a known CMDP: linear programming \citep[LP,][]{altman1999constrained}, primal-dual \citep{paternain2022safe}, and dual algorithms \citep{paternain2019strongduality}. If the CMDP is unknown, a common approach to handle the uncertainty is the classical paradigm of optimism in the face of uncertainty \citep{auer2008near}. In their influential paper, \citet{efroni2020exploration} established comprehensive regret guarantees for all three types of optimistic algorithms in the online setup. In practice, especially primal-dual algorithms are preferred due to their high computational efficiency and flexibility for policy parameterization, thereby scaling to high-dimensional problems \citep{chow2017risk,achiam2017constrained,tessler2018reward}. Thus, it is important to rigorously understand the fundamental properties of this algorithm class. Indeed, there has been a large number of studies on primal-dual (and dual) approaches for CMDPs \citep[to list just a few]{ding2020natural,ding2022convergence,liu22global,ding2022policy,ghosh2022provably,ding2022provably,qiu2020upper,liu2021learning,bai2022achieving} since the work of \citet{efroni2020exploration}. 

\looseness -1 However, unlike for LP-based algorithms, the known bounds for primal-dual (and dual) algorithms suffer from the fundamental limitation pointed out by \citet[Section 2.2]{efroni2020exploration}: they concern a \textit{weaker}, less safe notion of regret. More precisely, the known guarantees bound the sum of the suboptimalities and the sum of the constraint violations across episodes, where one episode corresponds to one round of learning. However, a policy can have a negative constraint violation (by being very safe but obtaining a lower return than an optimal safe policy) or a positive constraint violation (by being unsafe but obtaining a higher return than an optimal safe policy). Thus, terms from these two cases can cancel each other out when summing the violations across episodes, a phenomenon referred to as \textit{error cancellations} \citep{efroni2020exploration}. An algorithm with sublinear weak regret may heavily violate safety constraints during learning. For example, if the policies alternate between the two cases above in every other episode, the algorithm may even obtain \textit{zero} regret despite being unsafe every second episode. This weak notion of regret falls short of capturing safety in a setup with no simulator and where the algorithm must adhere to constraints during learning. In fact, these cancellations are not a weakness in the analysis but rather due to oscillations of the underlying optimization method, which converges on average but not in the last-iterate \citep{efroni2020exploration,beck2017first}. Indeed, these oscillations are observed in practice \citep{stooke2020responsive,moskovitz2023reload}.

\looseness -1 We thus consider a stronger notion of regret that concerns the sum of the \textit{positive parts} of the error terms instead. This regret does not allow for error cancellations, and we refer to it as \textit{strong regret}. The results of this paper address the research question pointed out by \citet{efroni2020exploration}:
\begin{center}
    \onehalfspacing
    \textit{\textbf{Can we design an efficient primal-dual algorithm that achieves sublinear strong regret in an unknown CMDP?}}
\end{center}
We provide the first affirmative answer for tabular finite-horizon CMDPs. Specifically, we introduce a regularization framework inspired by the recent work of \citet{ding2023last} and derive guarantees in the online setup for a primal-dual algorithm that arises from this formulation.

\textbf{Contributions} Our main contributions are the following:
\begin{itemize}
    \item We first prove non-asymptotic policy last-iterate convergence (\cref{def:li-conv}) of a regularized primal-dual scheme for CMDPs despite the inherent non-concavity, assuming access to a value function oracle (\cref{sec:li-conv}). Our guarantee generalizes previous results for the strictly easier problem of CMDPs with only a single constraint. This is the first analysis that establishes last-iterate convergence of primal-dual algorithms in \emph{arbitrary} CMDPs.
    \item Combining this regularized primal-dual scheme with optimistic exploration, we propose an improved model-based primal-dual algorithm (\cref{algo:rpg-pd-finite-learn}) for online learning in CMDPs (\cref{sec:pd}). Our algorithm requires no prior knowledge of the CMDP and maintains value-optimism for the \textit{regularized} problem. 
    \item Finally, we establish that our algorithm achieves sublinear \textit{strong} regret when learning an unknown CMDP (\cref{sec:regret-bound}). This is the first primal-dual algorithm achieving a sublinear regret guarantee without allowing error cancellations, providing the first answer to the open question posed by \citet{efroni2020exploration}.
\end{itemize}
The latter is relevant due to the efficiency and practical importance of primal-dual algorithms, which are often preferred over LP-based algorithms in large-scale applications. Additionally, we provide numerical evaluations of our algorithm in simple environments. We illustrate that it exhibits sublinear regret when safety during learning is concerned, while unregularized algorithms do not. We conclude that error cancellations are not merely a hypothetical issue of existing algorithms but bear practical relevance.\looseness-1

\subsection{Related Work} \label{sec:related-work}

\looseness -1 Since \citet{efroni2020exploration} analyzed the vanilla primal-dual (and dual) algorithm, their analysis has been extended in various works, both for the case of an unknown or known CMDP \citep{ding2020natural,ding2022convergence,liu22global,ding2022policy,ghosh2022provably,ding2022provably,qiu2020upper,liu2021learning,bai2022achieving}. As \citet{calvo2023state} pointed out, even the works assuming full knowledge of the CMDP only establish convergence of the averaged iterates. Hence, none of the mentioned works provides a guarantee for the strong regret or is easily amendable to obtain one. 

\looseness -1 Very recently, \citet{ding2023last} were the first to provide a last-iterate convergence analysis for a primal-dual algorithm in a known discounted infinite-horizon CMDP closely related to our algorithm. However, their analysis is limited to the case of a single constraint, which is non-trivial to generalize to multiple constraints (\cref{sec:li-conv}). Moreover, the authors left it as an open question whether the algorithm can be generalized to achieve last-iterate convergence in the online setup, when the CMDP is unknown (\cref{sec:pd}). In addition, our analysis holds for an algorithm with closed-form updates (\cref{eq:reg-primal,eq:reg-dual}), while their algorithm involves Bregman projections for technical reasons (\cref{lem:potential-finite}). 

Prior, \citet{moskovitz2023reload} showed last-iterate convergence of a primal-dual scheme, but their analysis concerns a hypothetical algorithm whose implicit updates do not allow efficient implementation.
\citet{li2021faster} provided a dual (not primal-dual) algorithm based on regularization like ours but only proved convergence for a history-weighted mixture policy\footnote{By mixture policy we refer to a policy randomly drawn from all policy iterates.} in a known CMDP. Similarly, \citet{ying2022dual} derived a dual algorithm with last-iterate convergence but left it open whether an online version is possible.

\citet{muller2023cancellation} were the first to prove a sublinear regret guarantee without error cancellations for a dual algorithm in the online setup. However, their algorithm, which is based on the augmented Lagrangian method, lacks the desired computational efficiency. Very recently, \citet{ghosh2024towards} proposed a different primal-dual algorithm to achieve sublinear strong regret, but their algorithm requires an exponential runtime to enjoy this guarantee. We refer to \cref{sec:extra-lit} for further comparison with prior results.

\section{Problem Formulation} \label{sec:background}

\textbf{Notation} For $n\in\mathbb{N}$, we use $[n]$ to refer to the set of integers $\curly{1, \dots ,n}$. For a finite set $X$, we denote the probability simplex over $X$ as $\Simplex{X} = \{v \in [0,1]^{X} | \sum_{x \in X} v_x = 1\}$. For $a\in\reals$, we set $[a]_+ := \max\{0, a\}$ to be the positive part of $a$. $\| \bm{b} \|$ denotes the $\ell_2$-norm of a vector $\bm{b}\in\reals^n$. $\tilde{O}$-notation refers to asymptotics up to poly-log factors.

\textbf{Constrained MDPs} A finite-horizon CMDP with state and action spaces $\calS$, $\calA$ (with finite cardinalities $S$ and $A$) and horizon $H>0$ is defined by a tuple
$\mathcal{M} = (\calS, \calA, H, p, r, \bm{u}, \bm{c})$.
Every episode consists of $H$ steps and starts from an initial state $s_1\in\calS$.\footnote{It is straightforward to extend this to any initial distribution $\mu$.} At every step $h$, $p_h(s' | s,a)$ denotes the probability of transitioning to state $s'$ if the current state and action are $s$ and $a$. Moreover, $r_h \colon \calS\times\calA \to [0,1]$, $(s,a) \mapsto r_h(s,a)$ denotes the reward function at step $h\in[H]$. Similarly, $\bm{u}_h \colon \calS\times\calA \to [0,1]^I$, $(s,a) \mapsto \bm{u}_{h}(s,a) = (u_{1,h}(s,a), \dots, u_{I,h}(s,a))^T \in [0,1]^I$ refers to the $I$ constraint reward functions, and $\bm{c} \in [0, H]^I$ are the respective thresholds $c_i$ for the $i$-th constraint $(i\in[I])$. 

\looseness -1 The algorithm interacts with the CMDP by playing a policy $\pi\in\Pi$, where
\begin{align*}
    \Pi := \bigg\{&(\pi_1, \dots, \pi_H) \biggVert
     \forall h~\forall s\in\calS \colon \pi_h(\cdot | s) \in \simplex{\calA} \bigg\}.
\end{align*}
\looseness -1 For any $\pi \in \Pi$, we consider the Markov process given by $a_h \sim \pi_h(\cdot | s_h)$, $s_{h+1} \sim p_h(\cdot|s_h,a_h)$ for $h = 1, \dots, H$. For any function $r' \colon \stepSet\times\calS\times\calA \to \reals$, $(h,s,a) \mapsto r'_h(s,a)$, every $(s, h) \in \calS\times\stepSet$ and $\pi\in\Pi$, consider the value functions
\begin{align*}
    \val{r',h}{\pi}{s} :=& \exptn_{\pi} \rectangular{ \sum_{h'=h}^H r'_{h'}(s_{h'},a_{h'}) \biggVert s_h = s },\\
    \qval{r',h}{\pi}{s,a} :=& \exptn_{\pi} \rectangular{ \sum_{h'=h}^H r'_{h'}(s_{h'},a_{h'}) \biggVert s_h = s,~ a_h=a }.
\end{align*} 
For notational convenience, we drop the indices for the step and state if we refer to $h=1$ and $s_1$ and write $\valStart{r'}{\pi} = \val{r',1}{\pi}{s_1}$. 

In the CMDP setting, we are interested in solving the following optimization problem:
\begin{align}
    \max_{\pi \in \Pi} \quad \valStart{r}{\pi} \quad \text{s.t.} \quad \valStart{u_i}{\pi} \geq c_i \quad (\forall i\in \conSet),\label{eq:cmdp-finite}
\end{align}
and we fix an \emph{optimal solution} $\piStar \in \Pi$ for \cref{eq:cmdp-finite}. 
Among all policies that are feasible with respect to the $I$ safety constraints $\valStart{u_i}{\pi} \geq c_i$, the goal is to find one that maximizes $\valStart{r}{\pi}$. 

We consider the stochastic reward setting, in which the algorithm observes rewards sampled from random variables $R_h(s,a) \in [0,1]$ and $\bm{U}_h(s,a) \in [0,1]^I$ such that $\E[R_h(s,a)] = r_h(s,a)$ and $\E[U_{i,h}(s,a)] = u_{i,h}(s,a)$ for all $i\in [I]$ when taking action $a$ in state $s$ at step $h$. Throughout, we make the following assumption, which is standard in the context of CMDPs \citep{altman1999constrained,efroni2020exploration,li2021faster,ying2022dual,ding2022convergencenat,paternain2022safe,ding2023last}.  
\begin{assumption}[Slater policy] \label{ass:slater-finite}
    There exists $\piBar \in \Pi$ and $\bm{\xi} \in \reals_{> 0}^I$ such that $\valStart{u_i}{\piBar} \geq c_i + \xi_i$ for all $i\in[I]$. Set the \emph{Slater gap}
    \begin{align*}
        \Xi := \min_{i\in[I]}~ \xi_i.
    \end{align*}
\end{assumption}
\looseness -1 This assumption asserts that there exists a policy that strictly satisfies the constraints.

\textbf{Problem Formulation} The algorithm interacts with the unknown CMDP over a fixed number of $K>0$ episodes. Prior to every episode $k\in [K]$, the algorithm selects a policy $\pi_k\in\Pi$ and plays it for one run of the CMDP. The goal is to simultaneously minimize its two \textit{strong regrets}:
\begin{align*}
    \Reg(K; r) &:= \sum_{k\in[K]} \left[ \valStart{r}{\piStar} - \valStart{r}{\pi_k} \right]_+, \tag{Objective}\\ 
    \Reg(K; \bm{u}) &:= \max_{i\in [I]} \sum_{k\in[K]} \left[ c_i - \valStart{u_i}{\pi_k} \right]_+. \tag{Constraints} 
\end{align*}
Only when a policy has a suboptimal objective or violates the constraints, this counts to the respective regret. All existing works on primal-dual (and dual) algorithms \citep[e.g.,][]{liu2021learning, efroni2020exploration, bai2022achieving, ding2022convergencepol, ding2022convergencenat} only prove sublinear guarantees on a \textit{weaker} notion:
\begin{align*}
        \Reg_{\text{weak}}(K; r) &:= \sum_{k\in[K]} \left( \valStart{r}{\piStar} - \valStart{r}{\pi_k} \right), \\ 
        \Reg_{\text{weak}}(K; \bm{u}) &:= \max_{i\in [I]} \sum_{k\in[K]} \left( c_i - \valStart{u_i}{\pi_k} \right). 
\end{align*}
The weak regrets allow for the aforementioned \textit{error cancellations} as positive and negative terms count toward each of the regrets. Even if they are sublinear in $K$ (in fact, even if they are \textit{zero}), the algorithm may continue compensating for a constraint violation in one episode with strict constraint satisfaction in another. On the other hand, a sublinear bound on the \textit{stronger} notion of regret guarantees that the algorithm achieves a low constraint violation in most episodes (see \cref{sec:str-vs-weak}). This is crucial for many practical applications where we do not have access to a simulator, but we have to learn our optimal policy in an \textit{online} fashion. In the example of navigating an autonomous vehicle or drone, one would want to avoid crossing the boundaries of a specified track in each episode during learning. It is not helpful to compensate for crashing the vehicle into a wall by driving overly safely in the next episode. However, from a theoretical perspective, it is strictly more challenging to provide a guarantee for the strong regret than for the weaker notion.\footnote{For practical purposes, one may consider the strong regret only for the constraint violations and the weak one for the objective. We refer to \cref{sec:further-exp} for a discussion of the differences. However, this relaxation does not improve our theoretical results.}

\section{Primal-Dual Scheme} \label{sec:duality-lagrangian}

\textbf{Vanilla Scheme} Primal-dual and dual algorithms arise from the equivalent Lagrangian formulation \citep{altman1999constrained} of \cref{eq:cmdp-finite}:
\begin{align}
    \max_{\pi \in \Pi} \min_{\bm{\lambda} \in \reals_{\geq 0}^I}~& \lagrangian(\pi, \bm{\lambda}), \label{eq:primal-finite}
\end{align}
where 
\begin{align*}
    \lagrangian (\pi, \bm{\lambda}) := \valStart{r}{\pi} + \sum_{i \in [I]} \lambda_i (\valStart{u_i}{\pi} - c_i ) = \valStart{r + \bm{\lambda}^T (\bm{u}-H^{-1}\bm{c}) }{\pi}
\end{align*}
is the \emph{Lagrangian}. \citet{paternain2019strongduality} showed that CMDPs exhibit strong duality, by which \cref{eq:primal-finite} is equivalent to finding a saddle point $(\piStar,\lambdaStar)$ of the Lagrangian. Primal-dual algorithms solve this saddle point problem via iterated play between two no-regret dynamics for $\pi$ and $\bm{\lambda}$. Typically, as considered by \citet{efroni2020exploration} in the regret minimization setting,
\begin{align} 
    \pi_{k+1,h}(a | s) \propto& ~ \pi_{k,h} (a | s) \exp \round{\eta \qval{r + \bm{\lambda_k}^T\bm{u}, h}{\pi_k}{s,a}},\label{eq:vanilla-prim} \\
    \bm{\lambda}_{k+1} =& \proj{\Lambda}{\bm{\lambda}_k - \eta (\valStartVec{\bm{u}}{\pi_k} - \bm{c})}, \label{eq:vanilla-dual}
\end{align}
where $\text{proj}_{\Lambda}$ refers to the projection onto a predefined $\Lambda = [0, \lambdaMax]^I$, which amounts to truncating the coordinates. We refer to \cref{eq:vanilla-prim,eq:vanilla-dual} as \textit{vanilla primal-dual scheme}. The mixture policy of the iterates is guaranteed to converge to an optimal solution pair of the min-max problem. However, the last iterate is not guaranteed to converge. Instead, the method oscillates around an optimal solution, which results in the weak regret bounds of previous primal-dual algorithms (\cref{sec:experiments}).

\textbf{Regularized Scheme} The key idea of the regularization is to induce strict concavity in the primal variable (to be precise, in the state-action occupancy measure $\occ{h}{\pi}{s,a} := P_{\pi}[ s_h=s, a_h=a]$ and not in the policy) and strong convexity in the dual variable $\bm{\lambda}$. This enables us to establish convergence to the unique solution of the regularized problem. We then show how to retrieve an error bound for the original, unregularized problem by carefully choosing the amount of regularization.

For $\tau > 0$, we define the \emph{regularized Lagrangian} $\regLagrangian \colon \Pi \times \reals^I \to \reals$ as 
\begin{align*}
    \regLagrangian (\pi, \bm{\lambda}) := \lagrangian (\pi, \bm{\lambda}) + \tau \left( \entropy(\pi) + \frac{1}{2}\| \bm{\lambda}\|^2 \right),
\end{align*}
where $\entropy (\pi) := - \exptn_{\pi} [ \sum_{h=1}^{H} \log (\pi_h(a_h | s_h)) ]$ is the \emph{entropy} of a policy $\pi$. Then, consider the following \emph{regularized CMDP problem}:
\begin{align}
    \max_{\pi \in \Pi} \min_{\bm{\lambda} \in \Lambda}~& \regLagrangian(\pi, \bm{\lambda}), \label{eq:reg-primal-finite}
\end{align}
The domain of the dual variable $\bm{\lambda}$ is now a compact set $\Lambda := [0,\lambdaMax]^I$, with $\lambdaMax \geq H\Xi^{-1}$ to be specified (crucially, we will choose it depending on the number of episodes $K$). Thanks to strong duality of the unregularized problem, any saddle point $(\piStar,\lambdaStar)$ of $\lagrangian$ satisfies $\norm{\lambdaStar} \leq H\Xi^{-1}$ (e.g., \cite{ying2022dual} for infinite horizon), which will allow us to constrain the dual variable as above. We denote the regularized primal and dual optimizers as follows:
\begin{align*}
    \regPiStar =& \arg \max_{\pi\in\Pi} \min_{\bm{\lambda} \in \Lambda} ~\regLagrangian(\pi, \bm{\lambda}), \\
    \regLambdaStar =& \arg \min_{\bm{\lambda} \in \Lambda} \max_{\pi \in \Pi} ~\regLagrangian(\pi, \bm{\lambda}).
\end{align*}
Regularization preserves strong duality (\cref{app:lagrangian}), by which we are equivalently looking for a saddle point $(\regPiStar,\regLambdaStar)$ of the regularized Lagrangian $\regLagrangian$. \citet{ding2023last} proposed to perform the ascent-descent scheme in \cref{eq:vanilla-prim,eq:vanilla-dual} on $\regLagrangian$ rather than $\lagrangian$ in the discounted infinite-horizon setting given a value function oracle:
\begin{align}
    \pi_{k+1,h}(a | s) \propto& \pi_{k,h} (a | s) \exp \round{\eta \qval{r + \bm{\lambda_k}^T\bm{u} + \tau \psi_{k},h}{\pi_k}{s,a}}, \label{eq:reg-primal}\\
    \bm{\lambda}_{k+1} =& \proj{\Lambda}{ (1-\eta\tau)\bm{\lambda}_k - \eta (\valStartVec{\bm{u}}{\pi_k} -\bm{c})}, \label{eq:reg-dual}
\end{align}
where $\psi_{k,h}(s,a) := -\log(\pi_{k,h}(a|s))$. We refer to \cref{eq:reg-primal,eq:reg-dual} as \textit{regularized primal-dual scheme}. In fact, our scheme above is a simplification of \citet{ding2023last}'s algorithm, since their policy update would read $\pi_{k+1,h}( \cdot | s) = \arg\max_{\pi_h(\cdot|s)\in\widehat{\Delta}(\calA)} \langle \pi_{h} ( \cdot | s), \qval{r + \bm{\lambda_k}^T\bm{u} + \tau \psi_{k},h}{\pi_k}{s,\cdot}\rangle - \frac{1}{\eta}\text{KL}(\pi_h(\cdot | s) || \pi_{k,h}(\cdot|s))$, where $\widehat{\Delta}(\calA) := \{ \pi_h(\cdot|s) \in \Simplex{\calA} ~\vert~ \forall a \in \calA \colon \pi_h(a|s) \geq \epsilon_0 /A \}$ for some $\epsilon_0 > 0$ is a \textit{restricted} probability simplex for technical reasons stemming from the analysis. While this update can be performed via Bregman projections \citep{orabona2019modern} of the KL divergence onto $\widehat{\Delta}(\calA)$, this requires solving a convex program at every iteration $k$ of the scheme. In contrast, our scheme admits a closed form of the policy update in \cref{eq:reg-primal} due to our modified analysis (see discussion of \cref{lem:potential-finite}).

\section{Last-Iterate Convergence} \label{sec:li-conv}

In this section, we prove last-iterate convergence of the regularized primal-dual scheme (\cref{eq:reg-primal,eq:reg-dual}) with an exact value function oracle (e.g., via policy evaluation if the true model is known) for an arbitrary number of constraints. We define last-iterate convergence as follows.

\begin{definition}[Last-iterate convergence] \label{def:li-conv}
    A method producing policy iterates $\pi_k \in \Pi$ ($k=1,2,\dots$) is \emph{last-iterate convergent} if 
    \begin{align*}
        V_r^{\pi^\star}-V_r^{\pi_k} \to 0 \quad \text{and} \quad [c_i-V_{u_i}^{\pi_k}]_+ \to 0 \quad (\forall i \in [I])
    \end{align*}
     as $k \to \infty$.
\end{definition}

The main technical challenges we overcome to show last-iterate convergence are: (a) to prove ascent properties for the primal update (\cref{eq:reg-primal}), which optimizes a nonconcave objective with surrogate gradients $\qval{r + \bm{\lambda_k}^T\bm{u} + \tau \psi_{k},h}{\pi_k}{s,a}$ that are unbounded in general; and (b) to bound all \textit{unregularized} constraint violations of the last iterate $\pi_k$ in the presence of more than one constraint. We provide all proofs for this section in \cref{app:convergence}.

\textbf{Regularized Optimizers} Our first step is to show that the iterates $(\pi_k, \bm{\lambda}_k)$ converge to the regularized optimizers $(\regPiStar, \regLambdaStar)$. Indeed, we formalize this by showing that the \textit{potential function}
\begin{align*}
    \Phi_k := & \sum_{s,h} P_{\regPiStar}[s_h=s] \kl_{k,h}(s) + \frac{1}{2} \norm{\regLambdaStar - \bm{\lambda}_k}^2
\end{align*}
approaches zero, if we choose the regularization parameter $\tau$ and the step size $\eta$ sufficiently small. Here, $\kl_{k,h}(s) := \kl(\regPiStarH{h}(\cdot|s), \pi_{k,h}(\cdot|s))$ refers to the Kullback-Leibler divergence between the optimal and the $k$-th policy, and $P_{\regPiStar}$ refers to the probability distribution under policy $\regPiStar$.

\begin{restatable}[Regularized convergence]{lemma}{potentialFinite} \label{lem:potential-finite} 
    Let $\eta, \tau < 1$ and $\lambdaMax \geq H\Xi^{-1}$. The iterates in \cref{eq:reg-primal,eq:reg-dual} satisfy 
    \begin{align*}
        \Phi_{k+1} \leq (1 - \eta \tau)^k \Phi_{1} + \tilde{O} \round{ \eta\tau^{-1} \cTauLambda },
    \end{align*}
    where
    \begin{align*}
        \cTauLambda =& \lambdaMax^2H^3A^{1/2}I^2\exp\round{\eta H \round{1 + \lambdaMax I +  \log(A)}} \\
        &+ I \round{ H + \tau \lambdaMax }^2.
    \end{align*}
\end{restatable}

\looseness -1 Despite the exponential term, we can control the factor $\cTauLambda$ to be constant of order $\text{poly}(A,H,I,\Xi^{-1})$ by choosing $\eta < (H\lambdaMax I \log(A))^{-1}$. For the remaining part, $\eta$ and $\tau$ need to be traded off to have fast linear convergence $(1 - \eta \tau)^k \Phi_{1}$ and a small bias term $\eta\tau^{-1} \cTauLambda$ simultaneously. \citet{ding2023last} showed a similar result with a different constant $C_{\eta,\tau}$ for their update rule that constrains the policies to the restricted probability simplex $\widehat{\Delta}(\calA) := \{ \pi_h(\cdot|s) \in \Simplex{\calA} ~\vert~ \forall a \in \calA \colon \pi_h(a|s) \geq \epsilon_0 /A \}$ by solving a convex problem in every iteration. They introduce this restriction as their proof requires a uniform bound of $\qval{r + \bm{\lambda_k}^T\bm{u} + \tau \psi_{k},h}{\pi_k}{s,a}$ and thus of $\tau \psi_{k,h} = - \tau \log(\pi_{k,h}(s,a))$, which may be unbounded outside of $\widehat{\Delta}(\calA)$. Our modified proof overcomes this challenge by leveraging a mirror descent (MD) lemma with local norms rather than the standard online MD lemma \citep{orabona2019modern}. While the standard norm of the regularized $Q$-values may be unbounded outside of $\widehat{\Delta}(\calA)$, we are able to bound their local norms to arrive at \cref{lem:potential-finite}, even though our policy updates (\cref{eq:reg-primal}) are not restricted to $\widehat{\Delta}(\calA)$ and thus closed-form. We refer to \cref{app:convergence} for the proof.

\textbf{Unregularized Error Bounds} While the bound in \cref{lem:potential-finite} depends on the choice of $\eta < 1$, $\tau < 1$ and $\lambdaMax \geq H\Xi^{-1}$, we show that it is possible though not obvious to choose them (depending on the desired approximation) such that $\Phi_k$ decays to zero. Prior to this, we show that this will allow us to upper-bound both the constraint violation and the objective suboptimality in the original problem. 

\begin{restatable}[Error bounds]{lemma}{potToRegretFinite}\label{lem:pot-to-regret-finite}
    For any sequence $(\pi_k)_{k\in[K]}$,
    \begin{align*}
        \rectangular{\valStart{r}{\piStar} - \valStart{r}{\pi_k}}_+ \leq& H^{3/2} (2\Phi_k)^{1/2} + \tau H\log(A),\\
        \max_{i\in \conSet} \rectangular{c_i-\valStart{u_i}{\pi_k}}_+ \leq& H^{3/2} (2\Phi_k)^{1/2} + \tau \lambdaMax \\
        &+ \lambdaMax^{-1} \round{H^2\Xi^{-1} + \tau H\log(A)}.
    \end{align*}
\end{restatable}
A similar result was provided by \citet{ding2023last} for the case of a single constraint ($I=1$). Generalizing this is technically challenging as the standard way of showing that approximate saddle points have small constraint violation \citep[Theorem 3.60]{beck2017first} does not apply in the case of regularized saddle points. Simultaneously, the technique of \citet[Corollary 1]{ding2023last} leverages the fact that only one constraint is present. We overcome this by choosing the domain $\Lambda=[0,\lambdaMax]^I$ larger than standard primal-dual algorithms, making it possible to extract bounds on the individual constraint violations from the approximate saddle points. See \cref{app:convergence} for the proof. This novel approach yields the rather uncommon inverse dependency on the diameter of $\Lambda$ in \cref{lem:pot-to-regret-finite}, which needs to be chosen such that both terms $\tau \lambdaMax$ and $\lambdaMax^{-1}$ are $O(\epsilon)$ to obtain an $\epsilon$-close solution. 

\cref{lem:pot-to-regret-finite} tells us that we can bound the objective suboptimality and all constraint violations by controlling the terms $\propto (\Phi_k)^{1/2}$ via \cref{lem:potential-finite}, and the remaining terms by appropriately choosing the regularization and domain diameter, in terms of $\epsilon$.

\textbf{Last-iterate Convergence} We are now ready to establish last-iterate convergence to the unregularized optimal policy of the regularized primal-dual scheme. 

\begin{restatable}[Last-iterate convergence]{theorem}{lastIterate} \label{lem:last-iterate}
    Let $\epsilon \in (0,1)$. Then, with appropriate choices of $\eta \propto \epsilon^{6}$, $\tau \propto \epsilon^{2}$, $\lambdaMax \propto \epsilon^{-1}$, for $k = \Omega(\text{poly}(A,H,I,\Xi^{-1}) \cdot \epsilon^{-10})$ we have 
    \begin{align*}
        \rectangular{\valStart{r}{\piStar} - \valStart{r}{\pi_k}}_+ \leq \epsilon, \quad \rectangular{c_i-\valStart{u_i}{\pi_k}}_+ \leq \epsilon \quad (\forall i \in [I]).
    \end{align*}
\end{restatable}
Here, we only highlight the explicit dependency on the desired approximation. The dependency on the CMDP size is (low-degree) polynomial and detailed in \cref{app:convergence}. The only problem-dependent constant in this bound is the Slater gap $\Xi$, which is shared by all primal-dual analyses of our knowledge. While the provided rate is slow and may be improved in the future, all other known rates of primal-dual algorithms for CMDPs with arbitrary constraints ($I>1$) only hold for the averaged and not the last iterate. More importantly, the technique leading to this result will allow us to achieve sublinear strong regret in the following section.

\section{Online Setup} \label{sec:pd}

Recall the regularized primal-dual scheme from \cref{eq:reg-primal,eq:reg-dual}. In our online learning setup, the true value functions are not known as we are learning the unknown CMDP. Thus, we are required to explore the CMDP and respect safety \textit{during} exploration. Replacing the value functions by optimistic estimates \citep{shani2020optimistic,auer2008near} allows us to turn the primal-dual scheme into an online learning algorithm for finite-horizon CMDPs (see \cref{algo:rpg-pd-finite-learn}). Importantly, we need to be optimistic with respect to the regularization term $\tau \entropy(\pi)$ too, rather than just the classical mixture value $\valStart{r + \bm{\lambda_k}^T\bm{u}}{\pi}$. The main technical challenge is to incorporate the model uncertainty into our primal-dual analysis from \cref{sec:li-conv}.

\subsection{Optimistic Model} For all $s,a,h$ and $k\in[K]$, let $\counter := \sum_{l=1}^{k-1} \indic_{\{ s_h^l=s,~ a_h^l = a \}}$ count the number of times that the state-action pair $(s,a)$ has been visited at step $h$ before episode $k$. Here, ($s_h^l$, $a_h^l$) denotes the state-action pair visited at step $h$ in episode $l$. First, we compute the empirical averages of the reward and transition probabilities as follows:
\begin{align}
    \rBar_{k-1,h}(s,a) :=& \frac{\sum_{l=1}^{k-1} R^{l}_{h}(s,a) \indic_{\{ s_h^l=s,~ a_h^l = a \}}}{\counterMax}, \nonumber\\
    \uBar_{k-1,i,h}(s,a) :=& \frac{\sum_{l=1}^{k-1} U_{i,h}^{l}(s,a) \indic_{\{ s_h^l=s,~ a_h^l = a \}}}{\counterMax} , \label{eq:counters} \\ 
    \pBar_{k-1,h}(s'|s,a) :=& \frac{\sum_{l=1}^{k-1} \indic_{\{ s_h^l=s,~ a_h^l = a,~ s_{h+1}^l=s' \}}}{\counterMax}, \nonumber
\end{align}
where $a \vee b := \max\{a,b\}$ and $\indic_{A}$ is the indicator function of an event $A$. We consider optimistic estimates
\begin{align}
    \rHat_{k,h}(s,a) :=& \rBar_{k-1,h}(s,a) + b_{k-1,h}(s,a), \nonumber\\
    \uHat_{k,i,h}(s,a) :=& \uBar_{k-1,i,h}(s,a) + b_{k-1,h}(s,a), \label{eq:model}\\
    \psiHat_{k,h}(s,a) :=& -\log(\pi_{k,h}(a|s)) + b^p_{k-1,h}(s,a) \log(A), \nonumber\\
    \pHat_{k,h}(s' | s,a) :=& \pBar_{k-1,h}(s' | s,a),\nonumber
\end{align}
where $b_{k-1,h}(s,a) = b^r_{k-1,h}(s,a) + b^p_{k-1,h}(s,a)$, and for any $\delta \in (0,1)$, we specify the correct values for
\begin{align*} 
    b^r_{k-1,h}(s,a) =& O \round{\sqrt{\frac{\log\round{ SAHIK\delta^{-1} }}{\counterMax}}}, \\
    b^p_{k-1,h}(s,a) =& O \round{H \sqrt{\frac{S + \log\round{ SAHK\delta^{-1} }}{\counterMax}}},
\end{align*} 
in \cref{app:model} to obtain our regret guarantees with probability at least $1- \delta$. The optimistic model guarantees that, with high probability, the obtained value functions overestimate the true ones and simultaneously allows us to control the estimation error. While optimistic exploration is standard, we here also take the entropy term in the objective into account via $\psiHat_k$. Let
\begin{align}
    \zHat_k := \rHat_k + \bm{\lambda}_k^T \bm{\uHat}_{k} + \tau \psiHat_k \label{eq:z}
\end{align}
be the optimistic reward function mimicking the $\pi$-dependency of the regularized Lagrangian at $(\pi_k, \lambda_k)$. Consider the \emph{truncated value functions}
\begin{align*}
    (h,s,a) \mapsto \qvalHat{\zHat_k, h}{k}{s,a}, \quad \quad \valStartHatVec{\bm{\uHat}_k}{k} = \valHatVec{\bm{\uHat}_k,1}{k}{s_1}
\end{align*}
that we compute via truncated policy evaluation (by dynamic programming) of $\pi_k$ with respect to the optimistic model. We refer to \cref{algo:trun-eval} in \cref{app:model}, where we also establish the relevant properties of the model. 

\textbf{Algorithm} Combining the truncated policy estimation under our learned model with the regularized primal-dual scheme (\cref{eq:reg-primal,eq:reg-dual}) yields \cref{algo:rpg-pd-finite-learn}. The computational cost of the algorithm amounts to evaluating a policy $O(I)$ times per episode, which matches the complexity of standard primal-dual algorithms and is more efficient than running dual or LP-based algorithms. Projecting onto $\Lambda$ is immediate since $\Lambda$ is a product of intervals.
\begin{algorithm}[t!] 
	\begin{algorithmic}
        \REQUIRE{$\Lambda = [0,\lambdaMax]^I$, stepsize $\eta>0$, regularization parameter $\tau > 0$, number of episodes $K$, initial policy $\pi_{1,h}(a | s) = 1/A$ ~($\forall s,a,h$), $\bm{\lambda}_1 := \bm{0} \in \reals^I$}
        \vspace{0.2cm}
        \FOR{$ k=1,\dots, K$}
        \STATE{Update $\rHat_{k}$, $\bm{\uHat}_{k}$, $\pHat_{k}$, $\psiHat_{k}$ via \cref{eq:model}}
        \vspace{0.2cm}
        \STATE{Truncated policy evaluation (\cref{algo:trun-eval}) w.r.t. $\zHat_k$ (\cref{eq:z}) and $\bm{\uHat}_k$:}
        \begin{align*}
            \qvalHat{\zHat_k}{k}{\cdot}, \valStartHatVec{{\bm\uHat}_{k}}{k} :=& \textsc{Eval}(\pi_k, \lambda_k, \rHat_k, \bm{\uHat}_k, \psiHat_k, \pHat_k).
        \end{align*}
        \STATE{Update primal variables for all $h$, $s$, $a$:}
            \begin{align*}
                \pi_{k+1,h}(a | s) \propto& \pi_{k,h} (a | s) \exp \round{\eta \qvalHat{h,\zHat_k}{k}{s,a}}
            \end{align*}
        \STATE{Update dual variables:}
            \begin{align*}
                \bm{\lambda}_{k+1} =& \proj{\Lambda}{(1-\eta\tau)\bm{\lambda}_k - \eta (\valStartHatVec{\bm{\uHat}_k}{k} - \bm{c}) }.
            \end{align*}
        \STATE{Play $\pi_k$ for one episode, update $\rBar_{k}$, $\bm{\uBar}_{k}$, $\gBar_{k}$, $\pBar_{k}$ via \cref{eq:counters}}
        \ENDFOR 
	\caption{Regularized Primal-Dual Algorithm with Optimistic Exploration} 
    \label{algo:rpg-pd-finite-learn}
	\end{algorithmic}	
\end{algorithm}

\subsection{Regret Analysis}\label{sec:regret-bound}
We now provide the key steps of our regret analysis, showing that \cref{algo:rpg-pd-finite-learn} indeed achieves sublinear strong regret for both the constraint violations and the objective. We defer all proofs for this section to \cref{app:pd}.

\begin{restatable}[Regularized convergence]{lemma}{potentialFiniteLearn} \label{lem:potential-finite-learn} 
    Let $\eta, \tau < 1$ and $\lambdaMax \geq H\Xi^{-1}$. With probability at least $1-\delta$, the iterates of \cref{algo:rpg-pd-finite-learn} satisfy 
    \begin{align*}
        \Phi_{k+1} \leq& (1-\eta \tau )^k \Phi_{1} + \tilde{O} \bigg( \eta\tau^{-1} \cTauLambda \\
        &+ \eta \lambdaMax \round{ISA^{1/2}H^2 k^{1/2} + IS^{3/2}AH^2} \bigg),
    \end{align*}
    where $\cTauLambda$ is the same constant as in \cref{lem:potential-finite}. 
\end{restatable} 
Here, we use $\tilde{O}$-notation for asymptotics up to polylogarithmic factors in $S$, $A$, $H$, $I$, $K$, $\Xi^{-1}$, and $\delta^{-1}$. This result is similar to our \cref{lem:potential-finite}, but now we obtain an additional term corresponding to the model uncertainty (estimation error), which we control when choosing the step size $\eta$.

\textbf{Regret Bound} In a final step, we can leverage \cref{lem:pot-to-regret-finite} to turn \cref{lem:potential-finite-learn} into a sublinear regret bound for \cref{algo:rpg-pd-finite-learn}, when summing up the error terms and choosing $\eta$, $\tau$, and $\lambdaMax \geq H\Xi^{-1}$ optimally depending on $K$ given our bounds. This yields our main result.
\begin{restatable}[Regret bound]{theorem}{rateFiniteLearn} \label{lem:rate-finite-learn} 
    Let $\tau = K^{-1/7}$, \quad $\eta = (H^2 I)^{-1}\Xi K^{-5/7}$, $\lambdaMax = H\Xi^{-1} K^{1/14}$. Then with probability at least $1-\delta$, \cref{algo:rpg-pd-finite-learn} obtains a \emph{strong} regret of
    \begin{align*}
        \Reg(K; r)  \leq C_{r} K^{0.93}, \quad
        \Reg(K; \bm{u}) \leq C_u K^{0.93},
    \end{align*}
    where $C_r,~C_u = \text{poly}(S,A,H,I,\Xi^{-1},\log(1/\delta),\log(K))$ and $K$ is the number of episodes.
\end{restatable}
Here, we only highlight the leading term in $K$. The dependency on the CMDP parameters is (low-degree) polynomial and detailed in \cref{app:pd}. Again, $\Xi$ is the only problem-dependent constant (and unavoidable).

\begin{remark}
    We remark that our proof of \cref{lem:rate-finite-learn}, in fact, shows last-iterate convergence in the online setup, which is strictly stronger than a regret bound in general.
\end{remark}

Our strong regret bound of $\tilde{O} (K^{0.93})$ is less tight than the $\tilde{O}(K^{1/2})$ that the vanilla primal-dual algorithm achieves for the \textit{weak} regret, for which there exist well-known lower bounds \citep{jin2018q,domingues2021episodic}. Nevertheless, \cref{algo:rpg-pd-finite-learn} is the first primal-dual algorithm for CMDPs provably achieving sublinear strong regret. It is thus the first algorithm of its kind for which we can guarantee that it cannot keep violating constraints indefinitely. While LP-based approaches achieve strong regret of $\tilde{O}(K^{1/2})$, most modern (deep) safe RL algorithms for CMDPs follow primal-dual schemes \citep{chow2017risk,tessler2018reward,stooke2020responsive}. We believe that it might be possible to tighten our analysis, although this will require novel ideas. Indeed, our numerical evaluations show that the parameter choices in \cref{lem:rate-finite-learn} are overly pessimistic.

\subsection{Strong vs. Weak Regret and Safety at any Time} \label{sec:str-vs-weak}
We allude to the strong regret several times by saying that a sublinear bound guarantees safety \emph{during learning} or \textit{in most episodes}. As our algorithm does not guarantee safety in \emph{every} episode, one may wonder in which sense safety during learning is formally guaranteed by the strong regret compared to the weak one. Indeed, this is an important theme in CMDPs. In an unknown CMDP, there is no way to explore it without constraint violations unless further assumptions are made (as the constraint rewards $u_i$ and transitions $P$ are unknown, we cannot know that an action is unsafe without trying at least once). However, this is a limitation of the CMDP model with exploration rather than our algorithm. We can thus only argue about safety in \emph{most episodes}. 

\textbf{Approximate Safety in Most Episodes} Unlike any previous primal-dual algorithm, our method guarantees that for any fixed $\epsilon > 0$, the fraction of episodes in which our policy is not $\epsilon$-safe vanishes to $0$ as the number of episodes $K$ grows. Not being $\epsilon$-safe here means to violate at least one constraint by at least $\epsilon$. We make this formal in the following remark.
\begin{remark}
    Fix $\epsilon > 0$ and suppose $\mathcal{R}(K;\mathbf{u}) \leq \tilde{O}(K^{\alpha})$ for some $\alpha \in (0,1)$. Then there exist at most $\tilde{O}(K^{\alpha} / \epsilon)$ episodes with a constraint violation of at least $\epsilon$. In other words, only a small fraction $\tilde{O}(K^{\alpha - 1}/\epsilon) = o(1)$ of the iterates is not $\epsilon$-safe. In comparison, this is by no means guaranteed by a sublinear bound on $\mathcal{R}_{\text{weak}}(K;\mathbf{u})$.
\end{remark}
Hence, our algorithm is approximately safe in \emph{most} episodes, while being safe in \emph{every} episode is not possible by design. This is a remarkable result since previous works on primal-dual algorithms can only guarantee safety of the average policy and not in most of the episodes (such algorithms can be fully unsafe in, e.g., half of the episodes). 

\textbf{Strict Safety in Most Episodes} Furthermore, $\epsilon$-safety can be strengthened by a simple reduction to ensure \emph{strict} safety in most episodes. This is possible by increasing the true constraint thresholds by a small shift of $\epsilon = O(K^{(\alpha-1)/2})$ to be more conservative and applying our results. For the formal details of this reduction, see, e.g., Appendix C.4 in \citet{ding2023last}. We thus established that the fraction of unsafe episodes is vanishing (in terms of $K$).

\section{Simulation} \label{sec:experiments}

We perform numerical simulations of our algorithms and compare them to their unregularized counterparts \citep{efroni2020exploration}. We find that the vanilla primal-dual and dual algorithms can suffer linear strong regret while our regularized counterparts do not, illustrating that \textit{error cancellations} are not merely a hypothetical issue. We provide further details in \cref{sec:further-exp}. 

\subsection{Baselines and Environment}

We compare our regularized primal-dual algorithm (\cref{algo:rpg-pd-finite-learn}) to the vanilla primal-dual algorithm of \citet{efroni2020exploration}, which corresponds to \cref{eq:vanilla-prim,eq:vanilla-dual} with optimistic exploration. We also include the vanilla dual algorithm of \citet{efroni2020exploration} as a baseline and our regularized dual algorithm (below), which arises from the same regularization framework as \cref{algo:rpg-pd-finite-learn}. We test each algorithm for the same total number (6) of hyperparameter configurations and report the best results for each.

\textbf{Dual Algorithm} Leveraging the framework we introduced, it is immediate to also derive a \textit{dual} algorithm for finite-horizon CMDPs. Dual algorithms amount to performing projected dual descent \citep{beck2017first,paternain2019strongduality} on the Lagrangian, where one can again use the optimistic model to estimate the unknown CMDP. \citet{efroni2020exploration} proved that this algorithm achieves a sublinear \textit{weak} regret. Instead, we perform dual descent on the \textit{regularized} Lagrangian $\regLagrangian$. Explicitly,
\begin{align}
    \pi_{k} &= \arg\max_{\pi \in \Pi} \round{ \valStart{\rHat_k + \bm{\lambda}_k^T\bm{\uHat}_k}{\pHat_k, \pi} + \tau \hat{\entropy}_k(\pi) }, \label{eq:reg-dual1}\\
    \bm{\lambda}_{k+1} &= \proj{\Lambda}{(1-\eta\tau)\bm{\lambda}_k - \eta (\valStartVec{\bm{\uHat}_k}{\pHat_k, \pi_k} - \bm{c}) }, \label{eq:reg-dual2}
\end{align}
where $\valStart{}{\pHat_k, \pi}$ refers to the value functions under transition model $\pHat_k$. These updates are similar to the ones proposed by \citet{li2021faster,ying2022dual}, yet both assume a value function oracle. We can compute the first update via regularized dynamic programming, and the second one is the same as before. The dual approach has a higher computational complexity as the primal update requires a planning subroutine rather than just policy evaluation, but shows similar numerical performance. See \cref{sec:dual} for the full description of the regularized dual algorithm.

\begin{figure}[t]
    \centering
    \includegraphics[width=0.49\textwidth]{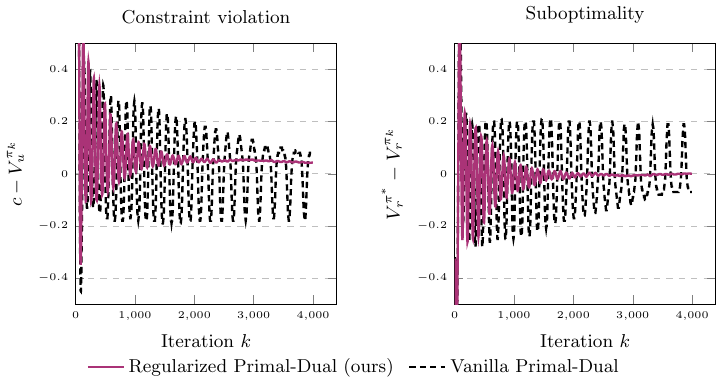}
    \caption{Constraint violation and objective suboptimality of the \textbf{vanilla} primal-dual algorithm \citep[cf. \cref{eq:vanilla-prim,eq:vanilla-dual}]{efroni2020exploration} and our {\color{brightPurple}\textbf{regularized}} version (\cref{algo:rpg-pd-finite-learn}). We present the values of the individual policies in each episode while learning the CMDP.}
    \label{fig:iterates}
\end{figure}

\textbf{Environment} We consider a randomly generated CMDP with deterministic rewards and unknown transitions. We draw the reward function $r$, constraint thresholds $\bm{c}$, and transitions $p$ uniformly at random. In order for oscillations (and thus error cancellations) to occur, the objective must be conflicting with the constraints \citep{moskovitz2023reload}, as they can otherwise easily be satisfied. However, by concentration of measure, two random vectors in high dimension are nearly orthogonal with high probability \citep{blum2020foundations}. Uniformly sampling the constraints would thus not yield interesting CMDPs, which is why we invoke a negative correlation between reward and constraint function. We sample the constraint function as $(1-r) + \beta \zeta$, where $\zeta \in \reals^{HSA}$ is Gaussian with zero mean and identity covariance matrix. We consider $S=A=H=5$, $\beta = 0.1$, and focus on the case of one constraint for visualization purposes. We refer to \cref{sec:further-exp} for further details.

\begin{figure*}[h]
    \centering
    \begin{subfigure}[b]{0.49\textwidth}
        \caption{Strong regrets}
        \label{fig:strong}
        \includegraphics[width=1\textwidth]{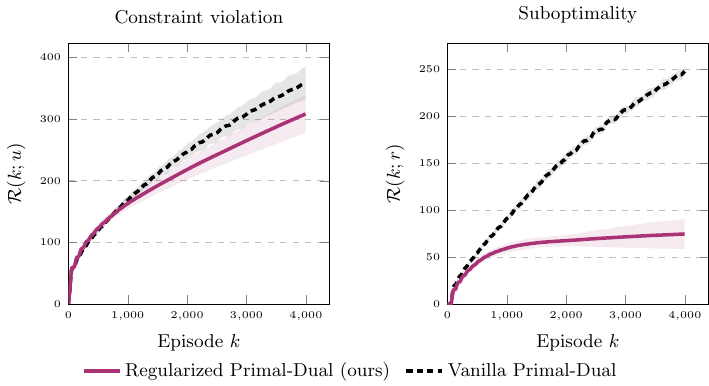}
    \end{subfigure}
    \begin{subfigure}[b]{0.49\textwidth}
        \caption{Weak regrets}
        \label{fig:weak}
        \includegraphics[width=1\textwidth]{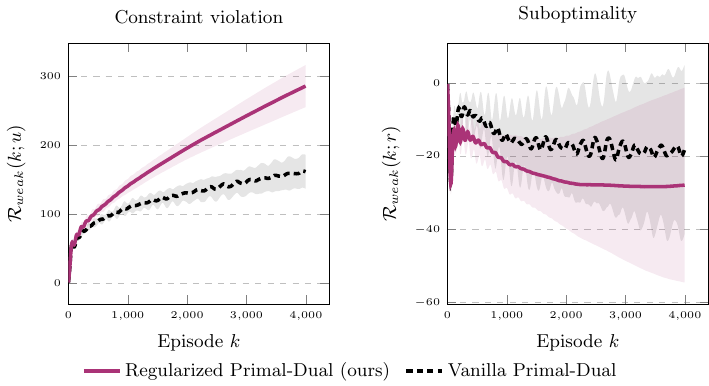} 
    \end{subfigure}
    \label{fig:regrets}
    \caption{\textbf{Vanilla} primal-dual algorithm \citep[cf. \cref{eq:vanilla-prim,eq:vanilla-dual}]{efroni2020exploration} and our {\color{brightPurple}\textbf{regularized}} version (\cref{algo:rpg-pd-finite-learn}). \cref{fig:strong} shows the strong regret; \cref{fig:weak} shows the weak regret. The weak regret regarding the objective can be negative, illustrating that the iterates are superoptimal but unsafe on average. Y-axes differ across plots. All results are averaged over $n=5$ independent runs, with plotted confidence intervals.}
\end{figure*}
\begin{figure*}[t]
    \centering
    \begin{subfigure}[b]{0.49\textwidth}
        \caption{Strong regrets}
        \label{fig:strong-dual}
        \includegraphics[width=1\textwidth]{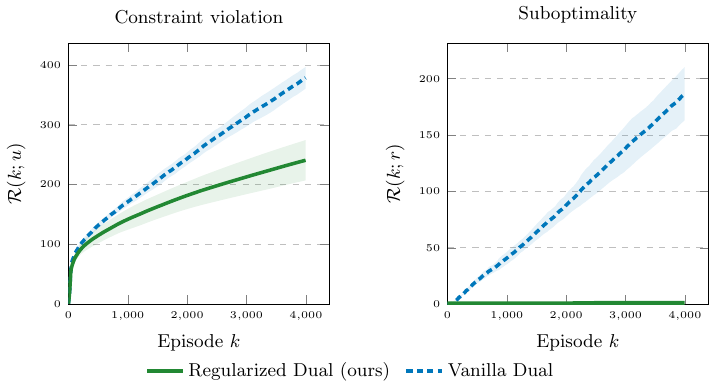} 
    \end{subfigure}
    \begin{subfigure}[b]{0.49\textwidth}
        \caption{Weak regrets}
        \label{fig:weak-dual}
        \includegraphics[width=1\textwidth]{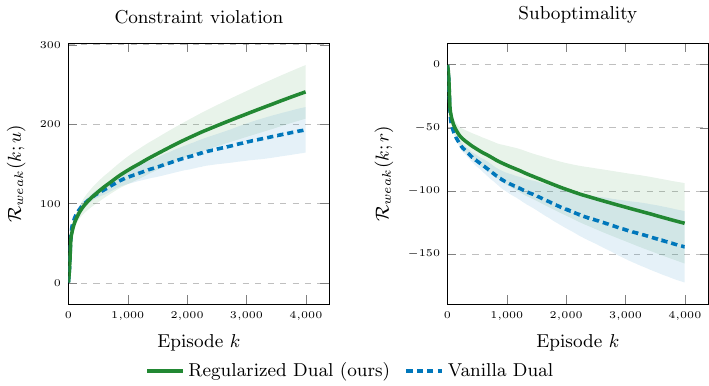} 
    \end{subfigure}
    \label{fig:regrets-dual}
    \caption{{\color{vibrantBlue}\textbf{Vanilla}} dual algorithm \citep{efroni2020exploration} and our {\color{brightGreen}\textbf{regularized}} version (\cref{eq:reg-dual1,eq:reg-dual2}). \cref{fig:strong-dual} shows the strong regret; \cref{fig:weak-dual} shows the weak regret. Y-axes differ across plots. All results are averaged over $n=5$ independent runs, with plotted confidence intervals.}
\end{figure*}

\subsection{Results}

The constraint violation and suboptimality of the iterates in each episode show the oscillatory behavior of the vanilla primal-dual algorithm as opposed to ours (\cref{fig:iterates}). While the on-average errors across episodes are sublinear, the vanilla algorithm keeps violating the constraints indefinitely as the number of episodes grows. In comparison, the oscillations of the regularized method are dampened, thus allowing it to converge to an optimal safe policy.

With respect to the \emph{weak} regret, the vanilla algorithms perform better (\cref{fig:weak,fig:weak-dual}, even constant for the suboptimality). However, with respect to the \emph{strong} regret, the regularized algorithms outperform the unregularized ones, as they achieve sublinear regret \emph{without} allowing for error cancellations (\cref{fig:strong,fig:strong-dual}). While the strong regrets for the vanilla algorithms may look sublinear, a second look at their iterates (\cref{fig:iterates}) reveals that their regret will indeed grow linearly due to the persisting oscillations. This confirms our key point that a sublinear bound on the weak regret is not informative whenever we do not allow compensating for an unsafe episode with a safe one. 

The vanilla algorithms will suffer linear strong regret even with a potentially better learning rate scheduling. We observed that the learning rate influences the oscillation frequency: With a larger learning rate, the vanilla methods oscillate faster. However, changing the learning rate does not dampen the oscillation magnitude. Hence, the strong regret is still linear. Indeed, we observe a change of magnitude only via the regularization parameter rather than the learning rate.

\textbf{Comparison with Guarantee} With the theoretically derived stepsize $\eta$, regularization $\tau$, and exploration from \cref{lem:rate-finite-learn}, we need many episodes to observe a benefit, due to the slowly vanishing gap between regularized and unregularized problem. Setting hyperparameters empirically, we observe a better regret than the theory suggests. Therefore, the plots in this section refer to the empirical choice.

\section{Conclusion}

In this paper, we gave the first answer to the open question of \citet{efroni2020exploration} whether primal-dual algorithms can achieve sublinear strong regret in finite-horizon CMDPs. While our answer is affirmative, it remains open in how far it is possible to lower the gap to the desired $\tilde{O}(K^{1/2})$ regret bound. We hope that our first analysis inspires further research on truly no-regret learning in CMDPs, including improvements in the analysis of our algorithm, incorporating function approximation, algorithms for the infinite-horizon average reward setup, and showing provable benefits of related approaches such as optimistic gradients.

\section*{Impact Statement}

This paper presents work whose goal is to advance the field of Machine Learning. There are many potential societal consequences of our work, none of which we feel must be specifically highlighted here.

\section*{Acknowledgements}

We thank the anonymous reviewers for their valuable feedback. P.A. is supported by the AI Center and ETH Foundations of Data Science (ETH-FDS) initiative. V.C. is supported by the Hasler Foundation Program: Hasler Responsible AI (project number 21043). V.C. is supported by the Swiss National Science Foundation (SNSF) under grant number 200021\_205011. N.H. is supported by Swiss National Science Foundation Project Funding No. 200021-207343 and SNSF Starting Grant.

\bibliography{refs}
\bibliographystyle{icml2024}

\newpage
\appendix
\onecolumn

\section{Summary of Notation} 
\label{app:notation}
\vspace{2cm}
The following table summarizes our general CMDP notation.
{{\renewcommand{\arraystretch}{1.75} \centering
\begin{table}[H]
    \fontsize{10pt}{10pt}\selectfont
    \centering
    \begin{tabular}{|>{\bfseries}l|l|}
        \hline
        State space& $\St$, with cardinality $S$\\
        Action space& $\A$, with cardinality $A$\\
        \# of constraints & $I$ \\
        Time horizon &$H$\\
        Transition probability & $p_h(s' | s,a) = P[s_{h+1}=s' \mid s_h=s, a_h=a]$\\
        Initial state & $s_1 \in \calS$\\
        Slater gap of $\piBar$ & $\Xi = \min_{i\in[I]} (\valStart{u_i}{\piBar} - c_i)$\\ 
        Number of episodes& $K$\\\hline
        Objective reward & Random variable $R_{h}(s,a) \in [0,1]$,
        with $\E[R_h(s,a)]=r_h(s,a)$\\
        Constraint rewards & Random variable $U_{i,h}(s,a)$ with $\E[U_{i,h}(s,a)]=u_{i,h}(s,a)$ \\
        Constraint thresholds & $\bm{c} \in \reals^I$, with $c_i \in [0,H]$ \\ 
        Constraint functions & $g_{i,h}(s,a) = u_{i,h}(s,a) - \frac{1}{H}c_i$ \\\hline
        Policy & $\pi \in \Pi$ with $(h,s,a) \mapsto \pi_h(a|s) $ (non-stationary)\\ 
        Value functions & $\val{r',h}{\pi}{s} = \E_{\pi}[\sum_{h'=h}^H r'_{h'}(s_{h'}, a_{h'}) \mid s_h=s]$ \\
        \normalfont{(shorthand)} & $\valStart{r'}{\pi} = \val{r',1}{\pi}{s_1}$ \\
        \normalfont{(vector-valued)} & $\valStartVec{\bm{u'}}{\pi} = (\valStart{u'_1}{\pi}, \dots, \valStart{u'_I}{\pi})^T \in \reals^I$ \\
        $Q$-values & $\qval{r',h}{\pi}{s,a} = \exptn_{\pi} [ \sum_{h'=h}^H r'_{h'}(s_{h'},a_{h'}) \mid s_h = s,~ a_h=a ]$\\
        Occupancy measures & $\occ{h}{\pi}{s,a} = P_{\pi}[s_h=s,a_h=a]$\\
        & $\occ{h}{\pi}{s} = P_{\pi}[s_h=s]$\\\hline
        Lagrangian & $\lagrangian (\pi, \bm{\lambda}) = \valStart{r}{\pi} + \sum_{i \in [I]} \lambda_i (\valStart{u_i}{\pi} - c_i ) = \valStart{r + \bm{\lambda}^T \bm{g} }{\pi}$\\
        Optimal policy & $\piStar \in \arg \max_{\pi\in\Pi} \min_{\bm{\lambda} \in \reals_{\geq 0}^I} ~\lagrangian(\pi, \bm{\lambda})$ \\ 
        Dual optimizer & $\lambdaStar \in \arg \min_{\bm{\lambda} \in \reals_{\geq 0}^I} \max_{\pi \in \Pi} ~\lagrangian(\pi, \bm{\lambda})$\\\hline
        Confidence level & $1-\delta$ \\
        Objective regret & $\Reg(K; r) = \sum_{k\in[K]} \left[ \valStart{r}{\piStar} - \valStart{r}{\pi_k} \right]_+$\\
        Constraint regret & $\Reg(K; \bm{u}) = \max_{i\in [I]} \sum_{k\in[K]} \left[ c_i - \valStart{u_i}{\pi_k} \right]_+$\\\hline
    \end{tabular}
    \label{aug-tab:mdp}
\end{table}
}

\newpage

The following table summarizes the notation specific to the algorithm.
{{\renewcommand{\arraystretch}{1.75} \centering
\begin{table}[H]
    \fontsize{10pt}{10pt}\selectfont
    \centering
    \begin{tabular}{|>{\bfseries}l|l|}
        \hline
        Step size & $\eta > 0$ (hyperparameter)\\
        Regularization parameter & $\tau > 0$ (hyperparameter)\\
        Dual threshold & $\lambdaMax > 0$ (hyperparameter)\\
        Dual domain & $\Lambda = \left[0, \lambdaMax \right]^I$ \\\hline
        Entropy & $\entropy (\pi) = - \exptn_{\pi} \left[ \sum_{h=1}^{H} \log (\pi_h(a_h | s_h)) \right]$\\
        Regularized Lagrangian & $\regLagrangian (\pi, \bm{\lambda}) = \valStart{r}{\pi} + \sum_{i \in [I]} \lambda_i (\valStart{u_i}{\pi} - c_i ) + \tau \left( \entropy(\pi) + \frac{1}{2}\|\bm{\lambda}\|^2 \right)$\\
        Regularized optimal policy & $\regPiStar \in \arg \max_{\pi\in\Pi} \min_{\bm{\lambda} \in \Lambda} ~\regLagrangian(\pi, \bm{\lambda})$\\
        Regularized dual optimizer & $\regLambdaStar \in \arg \min_{\bm{\lambda} \in \Lambda} \max_{\pi \in \Pi} ~\regLagrangian(\pi, \bm{\lambda})$\\\hline
        Auxiliary function & $\psi_{k,h}(s,a) = -\log(\pi_{k,h}(a | s))$\\
        KL divergence & $\kl(q,q') = \sum_{a \in \calA} q(a) \log\round{\frac{q(a)}{q'(a)}}$ ($q,q' \in \simplex{\calA}$)\\
        & $\kl_{k,h}(s) = \kl(\regPiStarH{h}(\cdot|s), \pi_{k,h}(\cdot|s))$\\
        & $\kl_k = \sum_{h} \sum_{s} \occ{h}{\regPiStar}{s} \kl_{k,h}(s)$\\ 
        Potential function & $\Phi_k = \kl_k + \frac{1}{2} \norm{\regLambdaStar - \bm{\lambda}_k}^2 \quad (k \geq 1)$\\
        Visitation counter & $\counter = \sum_{l=1}^{k-1} \indic_{\{ s_h^l=s,~ a_h^l = a \}}$\\
        Averages & $\rBar_{k-1,h}(s,a)$, $\uBar_{k-1,h}(s,a)$, $\pBar_{k-1,h}(s'|s,a)$\\
        Exploration bonuses & $ b_{k-1,h}(s,a) = b^r_{k-1,h}(s,a) + b^p_{k-1,h}(s,a) $\\
        Optimistic estimates & $\rHat_k$, $\bm{\uHat}_k$, $\bm{\gHat}_k = \bm{\uHat}_k - \frac{1}{H}\bm{c}$, $\psiHat_k$, $\pHat_k$\\
        Regularized reward function & $z_k = r + \bm{\lambda}_k^T \bm{u} + \tau \psi_k$, $\zHat_k = \rHat_{k} + \bm{\lambda}_k^T\bm{\uHat}_k + \tau \psiHat_k$\\
        Success event & $G$\\\hline
        Truncated value functions & $\qvalHat{\zHat_k,h}{k}{s,a} = \qvalHat{\rHat_k,h}{k}{s,a} + \sum_{i} \lambda_{k,i} \qvalHat{\uHat_{k,i},h}{k}{s,a}$\\
        &~~~~~~~~~~~~~~~~~~$+ \tau \qvalHat{\psiHat_k,h}{k}{s,a} $\\
        & $\valHat{\zHat_k,h}{k}{s} = \dotProd{\pi_{k,h}(\cdot|s)}{\qvalHat{\zHat_k,h}{k}{s,\cdot}}$\\\hline
    \end{tabular}
    \label{aug-tab:regpd}
\end{table}
} 
}

\section{Extended Related Work} \label{sec:extra-lit}

In this section, we review further related work and provide a technical comparison with prior works.

\textbf{Constrained MDPs} \citet{efroni2020exploration} provided the first regret analysis for LP-based (\textsc{OptLP}), primal-dual (\textsc{OptPrimalDual}), and dual algorithms (\textsc{OptDual}). \textsc{OptLP} achieves the optimal strong regret of $\tilde{O}(K^{1/2})$, yet most modern CMDP algorithms are based on primal-dual schemes rather than LP. \textsc{OptPrimalDual} is akin to our \cref{algo:rpg-pd-finite-learn} but without regularization. It guarantees a weak regret of $\tilde{O}(K^{1/2})$ but no bound on the strong regret, which is left as an open question that we addressed in \cref{sec:pd}. The same holds regarding the guarantees for \textsc{OptDual}, for which the question about strong regret bounds is still unanswered.

Since \citet{efroni2020exploration} analyzed the vanilla primal-dual (and dual) algorithm, their analysis has been extended in various works, both for the case of an unknown or known CMDP. Specifically, the algorithms have been extended to natural policy gradient methods with policy parameterization \citep{ding2020natural,ding2022convergence,liu22global}, function approximation in the linear MDP setup \citep{ding2022policy, ghosh2022provably}, CMDPs with time-varying characteristics \citep{ding2022provably,qiu2020upper}, and have even been shown to achieve bounded on-average constraint violation \citep{liu2021learning,bai2022achieving}. However, all these works only established convergence of the \textit{averaged} iterates or a sublinear \textit{weak} regret. In practice, recent works do show empirical success (using optimistic gradients \citep{moskovitz2023reload} and PID control \citep{stooke2020responsive}) but without the desired theoretical guarantees. 

\textbf{Comparison with Prior Results} \citet{ding2023last} analyzed two algorithms, \textsc{RPG-PD} and \textsc{OPG-PD}, for last-iterate convergence assuming a value function oracle. \textsc{RPG-PD} follows the same scheme as our \cref{eq:reg-primal,eq:reg-dual}. However, the analysis is tailored for a single constraint. It is not straightforward (as far as we know) how Corollary 1 in \cite{ding2023last} can be extended to multiple constraints (which we achieve in \cref{sec:li-conv}). Our \cref{lem:pot-to-regret-finite} generalizes the analysis to deal with multiple constraints. However this extensions leads to a worse iteration complexity in \cref{lem:last-iterate} of $O(\epsilon^{-10})$ rather than $\tilde{O}(\epsilon^{-6})$. In addition to this, \citet{ding2023last}'s policy update differs from ours in that it does not allow a closed-form solution but requires projection onto a \textit{restricted} probability simplex for technical reasons (see discussion of \cref{lem:potential-finite}, providing an analysis for our closed-form updates). 

Moreover, \citet{ding2023last} left it as an open question whether the algorithm can be generalized to achieve last-iterate convergence in the online setup, when the CMDP is unknown; we addressed this point in \cref{sec:pd}.

The other algorithm of \citet{ding2023last}, \textsc{OPG-PD}, is based on optimistic gradient updates and requires the restrictive assumption that the optimal state-visitation distribution (i.e., occupancy measure) is unique and introduces an extra problem-dependent constant. Moreover, it assumes a uniform lower bound on the state-visitation frequency in the discounted infinite-horizon setting, an assumption that cannot be guaranteed in the finite-horizon setting.

\citet{moskovitz2023reload} showed last-iterate convergence of a primal-dual scheme using optimistic gradient updates given a known CMDP, but their analysis concerns an algorithm operating over occupancy measures rather than policies (different from the practical implementation). Its implicit updates are constrained over the set of occupancy measures (i.e., the Bellman flow polytope), making them at least as computationally expensive as solving the CMDP directly via an LP in the first place.

\citet{calvo2023state} considered a rather different approach to overcome the problem that CMDPs cannot be modeled by a single (mixture) reward weighted by Lagrange multipliers (sometimes referred to as \textit{scalarization fallacy}). They proposed a state-augmentation technique that addresses this related problem without guaranteeing last-iterate convergence.

\citet{ghosh2024towards} proposed a rather different model-free primal-dual algorithm for the linear MDP setting. Their algorithm achieves $\tilde{O}(K^{1/2})$ strong regret if it is allowed to take $\Omega(d^{H-1}K^{1.5H+0.5}\log(A)^H)$ computational steps in every episode. This is needed because their algorithm searches for an optimal dual variable in each episode by making incremental steps of $\eta=1/(d^{H-1}K^{1.5H}\log(A)^H)$, potentially until reaching $K^{1/2}$ (see their Algorithm 1). In our work, we focus on polynomial-time algorithms that achieve a strong regret guarantee.

\textbf{Dual Algorithms} \citet{li2021faster} provided a dual (not primal-dual) algorithm based on the same regularization scheme as ours but considered an accelerated dual update and only proved convergence for a history-weighted mixture policy in a known CMDP. Similarly, \citet{ying2022dual} derived a dual (not primal-dual) algorithm with last-iterate convergence but left it open if a sample-based version is possible. Moreover, their analysis covers the discounted infinite-horizon setting and requires a uniform lower bound on the state-visitation frequency, an assumption that cannot be guaranteed in the finite-horizon setting.

\textbf{Constrained Bandits} In the simpler bandit setup where there is only a single state, there are mainly three setups in the literature: Knapsack bandits \citep{agrawal2016linear,badanidiyuru2018bandits} consider reward maximization over time as long a some global budget is not used up yet. Conservative bandits \citep{wu2016conservative,kazerouni2017conservative} concern algorithms whose cumulative reward performs sufficiently well relative to some pre-defined baseline policy. Finally, there is a line of research on stage-wise constrained bandits \citep{amani2019linear,pacchiano2021stochastic}, which require algorithms that obtain a reward and a cost associated with an action, where the latter should stay below a threshold in each round. While these settings may inspire related research in CMDPs, they are rather different from ours: They consider hard thresholds in the single-state setup, while exploration in CMDPs is generally stateful and commonly aims at simultaneous minimization of reward and constraint regrets.

\section{Properties of the Lagrangian Formulation} \label{app:lagrangian}

The results in this section are not novel by themselves, but we re-establish them here for finite-horizon CMDPs for completeness. We refer to \cref{app:convopt} for the relevant convex optimization background. To view the CMDP as a convex optimization problem, we will express it via the common notion of occupancy measures \citep{borkar1988convex}.
\begin{definition}
    The \emph{state-action occupancy measure} $d^{\pi}$ of a policy $\pi$ for a CMDP $\M$ is defined as 
    \begin{align*}
        \occ{h}{\pi}{s,a} := \E \left[ \indic_{\{ s_h = s, a_h = a \}} \mid s_1; p, \pi \right] = P[ s_h=s, a_h=a \mid s_1; p, \pi ],
    \end{align*}
    for $s\in\St$, $a\in\A$, $h\in[H]$. We denote the stacked vector of these values as $\occFn{ }{\pi} \in \R^{HSA}$, with the element at index $(h,s,a)$ being $\occ{h}{\pi}{s,a}$. Similarly, we define
    \begin{align*}
        \occ{h}{\pi}{s} := P[ s_h=s \mid s_1; p, \pi ] = \sum_{a} \occ{h}{\pi}{s,a}
    \end{align*}
    for $s\in\calS$.
\end{definition}
We can now define
\begin{align*}
    Q(p) := \left\{ \occFn{ }{\pi} \in \R^{HSA} \mid \pi \in \Pi \right\}
\end{align*}
as the \textit{state-action occupancy measure polytope}. Note that $Q(p)$ is indeed a polytope \citep{puterman2014markov}. Moreover, we have a surjective map $\pi \mapsto \occFn{ }{\pi}$ between $\Pi$ and $Q(p)$, for which we can explicitly compute an element in the pre-image of $\bm{d} \in Q(p)$ via $\pi_{h}(a|s) = d_h(s,a) / (\sum_{a'} d_h(s,a'))$.

We can stack the expected rewards $r_h(s,a)$ and constraint rewards $u_{i,h}(s,a)$ in the same way as $\occ{h}{\pi}{s,a}$ to obtain vectors $\bm{r} \in \R^{HSA}$ and $\bm{u}_i \in \R^{HSA}$. Note that we then have $\valStart{r}{\pi} = \sum_{h,s,a} \occ{h}{\pi}{s,a} r_h(s,a) = \bm{r}^T \occFn{ }{\pi}$ by linearity of expectation. Similarly, for all $i\in [I]$, we have $\valStart{u_i}{\pi} = \bm{u}_i^T \occFn{ }{\pi}$. Moreover, if we stack $\bm{U} = (\bm{u}_i)_{i\in[I]} \in \R^{I \times HSA}$ and $\bm{c} = (c_i)_{i\in[I]}\in \R^I$ as
\begin{align*}
    \bm{U} := \left(\begin{matrix}
    \bm{u}_1^T\\
    \vdots\\
    \bm{u}_I^T
    \end{matrix}\right),\hspace{1cm}
    \bm{c} := \left( \begin{matrix}
        c_1\\
        \vdots \\
        c_I
    \end{matrix} \right),
\end{align*}
we obtain $\valStartVec{\bm{u}}{\pi} = \bm{U} \occFn{ }{\pi} \in [0,H]^I$ for the vector of the constraint value functions. We can thus write
\begin{align*}
    \pi^* \in \arg \max_{\pi \in \Pi} \quad \valStart{r}{\pi} \quad \text{s.t.} \quad \valStart{u_i}{\pi} \geq c_i \quad (\forall i \in [I])
\end{align*}
equivalently as
\begin{align}
    \occFn{ }{\piStar} \in \arg \max_{\occFn{ }{\pi} \in Q(p)} \quad \bm{r}^T \occFn{ }{\pi} \quad \text{s.t.} \quad \bm{U} \occFn{ }{\pi} \geq \bm{c}, \label{eq:cmdp-finite-occ}
\end{align}
where $\geq$ is understood element-wise. This is a linear program (LP). In particular, by compactness of the state-action occupancy polytope, there exists an optimal solution $\pi^*$ as we assume feasibility.

\begin{restatable}[Strong duality CMDP \citep{paternain2019strongduality}]{lemma}{SdCmdpFinite} \label{lem:SD-cmdp-finite}
    We have 
    \begin{align*}
        \max_{\pi \in \Pi} \min_{\bm{\lambda} \in \reals_{\geq 0}^{I}}~ \lagrangian(\pi, \bm{\lambda}) = \min_{ \bm{\lambda} \in \reals_{\geq 0}^I } \max_{\pi \in \Pi}~ \lagrangian(\pi, \bm{\lambda}),
    \end{align*}
    and both optima are attained.
\end{restatable}

\begin{proof}
    Note that, under Assumption \ref{ass:slater-finite}, we can view \cref{eq:cmdp-finite} as the convex optimization problem in \cref{eq:cmdp-finite-occ} over $Q(p)$ that satisfies all parts of Assumption \ref{aug-ass:841} from \cref{aug-sec:conv-prelim}. Indeed, 
    \begin{itemize}
        \item[(a)] $X:=Q(p)$ is a polytope and thus convex
        \item[(b)] the objective $f(\cdot) := -\bm{r}^T (\cdot)$ is affine and thus convex
        \item[(c)] the constraints $g_i(\cdot) := c_i - \bm{u}_{i}^T (\cdot)$ are affine and thus convex
        \item[(d)] by Assumption $\ref{ass:slater-finite}$, \cref{eq:cmdp-finite-occ} is feasible, and thus its optimum is attained (since the domain is compact and the objective continuous)
        \item[(e)] a Slater point exists by Assumption \ref{ass:slater-finite}, namely $\occFn{ }{\piBar}$
        \item[(f)] all dual problems have an optimal solution since the domain $X$ is compact and the objective $f(\cdot) + \bm{\lambda}^T\bm{g}(\cdot)$ is continuous,
    \end{itemize} 
    where $Q(p) \subset \reals^{HSA}$, $\bm{r} \in \reals^{HSA}$ and $\bm{u}_i \in \reals^{SAH}$ are defined as above. The claim now readily follows from \cref{aug-thm:duality}.
\end{proof}

\begin{restatable}[e.g., \citet{ying2022dual}]{lemma}{dualBoundFinite} \label{lem:dual-bound-finite}
    We have $\norm{\lambdaStar}_1 \leq \frac{H}{\Xi}$,
\end{restatable}

\begin{proof}
    As in the proof of \cref{lem:SD-cmdp-finite}, under Assumption \ref{ass:slater-finite}, we can view the CMDP problem as a convex optimization problem in the occupancy measure, in the same setup as \cref{aug-sec:conv-prelim}.
    Specifically, we have $\valStart{r}{\pi} = \bm{r}^T \occFn{ }{\pi}$ and $\valStartVec{\bm{u}}{\pi} = \bm{U} \occFn{ }{\pi}$. Then, set $X=Q(p)$, $\xbar = \occFn{ }{\pibar}$, $f(\cdot) = -\bm{r}^T(\cdot)$ and $g_i(\cdot) = c_i - \bm{u}_{i}^T(\cdot)$. Plugging this into \cref{aug-thm:conv-dual} indeed yields 
    \begin{align*}
        \norm{\lambdaStar}_1 \leq \frac{\valStart{r}{\piStar} - \valStart{r}{\piBar}}{\min_{i\in[I]}(\valStart{u_i}{\piBar} - c_i)} \leq \frac{H}{\Xi}.
    \end{align*}
\end{proof}

\begin{restatable}[Saddle point CMDP]{lemma}{spCmdpFinite} \label{lem:sp-cmdp-finite}
    Let $\pi \in \Pi$ and $\bm{\lambda} \in \reals_{\geq 0}^I$. Then 
    \begin{align*}
        \lagrangian(\pi, \lambdaStar) \leq \lagrangian(\piStar, \lambdaStar) \leq \lagrangian(\piStar, \bm{\lambda}).
    \end{align*}
\end{restatable}

\begin{proof}
    By \cref{lem:SD-cmdp-finite}, this immediately follows from \cref{lem:minmax-to-sp} in \cref{app:minmax}.
\end{proof}

\begin{restatable}[Strong duality regularized CMDP \citep{ding2023last}]{lemma}{SdRegCmdpFinite} \label{lem:SD-reg-cmdp-finite}
    We have 
    \begin{align*}
        \max_{\pi \in \Pi} \min_{\bm{\lambda} \in \Lambda}~ \regLagrangian(\pi, \bm{\lambda}) = \min_{ \bm{\lambda} \in \Lambda } \max_{\pi \in \Pi}~ \regLagrangian(\pi, \bm{\lambda}),
    \end{align*}
    and both primal and dual optimum are attained.
\end{restatable}

\begin{proof} 
    For all $\pi\in\Pi$, $\bm{\lambda}\in\Lambda$, we have
    \begin{align*}
        \regLagrangian(\pi, \bm{\lambda}) =& \valStart{r+\bm{\lambda}^T\bm{g}}{\pi} + \tau\round{\entropy(\pi) + \frac{1}{2}\norm{\bm{\lambda}}^2}\\
        =& \sum_{s,a,h} (r_h(s,a) + \sum_{i} \lambda_i g_{i,h}(s,a)) \occ{h}{\pi}{s,a} \\
        &+ \tau \round{-\sum_{s,a,h} \occ{h}{\pi}{s,a} \log\round{\frac{\occ{h}{\pi}{s,a}}{\sum_{a'} \occ{h}{\pi}{s,a'}}} + \frac{1}{2}\norm{\bm{\lambda}}^2}\\
        =:& \regLagrangian^{occ} (\occFn{ }{\pi}, \bm{\lambda}),
    \end{align*}
    where $g_{i,h}(s,a) = u_{i,h}(s,a) - \frac{1}{H}c_i$ and where we used the definition of the occupancy measures and the polytope $Q(p)$. Consider the problem
    \begin{align}
        \max_{\bm{d} \in Q(p)} \min_{\bm{\lambda} \in \Lambda}~ \regLagrangian^{occ} (\bm{d}, \bm{\lambda}). \label{eq:reg-primal-finite-occ}
    \end{align}
    For any $\pi \in \Pi$ that is optimal for \cref{eq:reg-primal-finite}, $\occFn{ }{\pi}$ is also optimal for \cref{eq:reg-primal-finite-occ}. Conversely, for every $\bm{d} \in Q(p)$ that is optimal for \cref{eq:reg-primal-finite-occ}, we have that any $\pi$ given by $\pi_h(a|s) := \frac{d_h(s,a)}{\sum_{a'\in\calA} d_h(s,a')}$ for $s$ with $\sum_{a'\in\calA} d_h(s,a) > 0$, and arbitrary otherwise, is optimal for \cref{eq:primal-finite}. 
    
    Note that $\regLagrangian^{occ}$ is continuous. We further claim that $\regLagrangian^{occ}$ is 1-strongly convex in $\bm{\lambda} \in \Lambda$ and concave in $\bm{d} \in Q(p)$. Indeed, while the former claim is immediate, we can see the latter via the log-sum inequality (e.g., \citet[Theorem 2.7.1]{cover1999elements}) with $n=2$: For non-negative $a_i$, $b_i$,
    \begin{align*}
        -\round{\sum_{i=1} a_i}\log\round{\frac{\sum_{i=1} a_i}{\sum_{i=1} b_i}} \geq - \sum_{i=1} a_i\log\round{\frac{a_i}{b_i}} 
    \end{align*}
    and equality if and only if $a_i/b_i$ is the same for all $i$. Only considering the nonlinear term in $\regLagrangian(\bm{\lambda}, \cdot)$, for $\bm{d}_1, \bm{d}_2 \in Q(p)$ and $\alpha\in (0,1)$ we have 
    \begin{align*}
        &-\sum_{s,a,h} \round{\alpha d_{1,h}(s,a) + (1-\alpha)d_{2,h}(s,a)} \log\round{\frac{\alpha d_{1,h}(s,a) + (1-\alpha)d_{2,h}(s,a)}{\sum_{a'} \round{\alpha d_{1,h}(s,a') + (1-\alpha)d_{2,h}(s,a')} }}\\
        =& \sum_{s,a,h} - \round{\alpha d_{1,h}(s,a) + (1-\alpha)d_{2,h}(s,a)} \log\round{\frac{\alpha d_{1,h}(s,a) + (1-\alpha)d_{2,h}(s,a)}{\sum_{a'} \alpha d_{1,h}(s,a') + \sum_{a'}(1-\alpha)d_{2,h}(s,a') }}\\
        \geq& \sum_{s,a,h} - \alpha d_{1,h}(s,a) \log\round{\frac{\alpha d_{1,h}(s,a)}{\sum_{a'} \alpha d_{1,h}(s,a')}} \\
        &+ \sum_{s,a,h} - (1-\alpha)d_{2,h}(s,a) \log\round{\frac{(1-\alpha)d_{2,h}(s,a)}{\sum_{a'}(1-\alpha)d_{2,h}(s,a') }}\\
        =& - \alpha \sum_{s,a,h} d_{1,h}(s,a) \log\round{\frac{d_{1,h}(s,a)}{\sum_{a'} d_{1,h}(s,a')}} \\
        & - (1-\alpha)\sum_{s,a,h} d_{2,h}(s,a) \log\round{\frac{d_{2,h}(s,a)}{\sum_{a'}d_{2,h}(s,a') }}.
    \end{align*}
    with equality if and only if $\frac{d_{1,h}(s,a)}{\sum_{a'} d_{1,h}(s,a')} = \frac{d_{2,h}(s,a)}{\sum_{a'} d_{2,h}(s,a')}$ for all $s,h$. By \cref{lem:exist-minmax}, we thus have
    \begin{align*}
        \max_{\bm{d} \in Q(p)} \min_{\bm{\lambda} \in \Lambda}~ \regLagrangian^{occ} (\bm{d}, \bm{\lambda}) = \min_{\bm{\lambda} \in \Lambda} \max_{\bm{d} \in Q(p)}~ \regLagrangian^{occ} (\bm{d}, \bm{\lambda}),
    \end{align*}
    and primal and dual optimizers exist. This implies the same for the original problem \cref{eq:primal-finite} by converting the occupancy measures back into policies via $\pi_h(a|s) = d_h(s,a)/(\sum_{a'} d_h(s,a'))$.
\end{proof}

\begin{restatable}[Saddle point regularized CMDP]{lemma}{spRegCmdpFinite} \label{lem:sp-reg-cmdp-finite-main}
    Let $\pi \in \Pi$ and $\bm{\lambda} \in \Lambda$. Then 
    \begin{align*}
        \regLagrangian(\pi, \regLambdaStar) \leq \regLagrangian(\regPiStar, \regLambdaStar) \leq \regLagrangian(\regPiStar, \bm{\lambda}).
    \end{align*}
\end{restatable}

\begin{proof}
    By \cref{lem:SD-reg-cmdp-finite}, this follows from \cref{lem:minmax-to-sp}.
\end{proof}

\begin{restatable}{lemma}{spRegCmdpFiniteCor} \label{lem:sp-reg-cmdp-finite}
    Let $\pi \in \Pi$ and $\bm{\lambda} \in \Lambda$. Then 
    \begin{align*}
        \valStart{r+(\regLambdaStar)^T \bm{g}}{\pi} - \tau \entropy( \regPiStar ) \leq \valStart{r + (\regLambdaStar)^T \bm{g}}{\regPiStar} \leq \valStart{r + \bm{\lambda}^T \bm{g}}{\regPiStar} + \frac{\tau}{2} \norm{\bm{\lambda}}^2,
    \end{align*}
    where $\bm{g} = \bm{u} - \frac{1}{H}\bm{c}$.
\end{restatable}

\begin{proof}
    Plugging the definition of $\regLagrangian$ into \cref{lem:sp-reg-cmdp-finite-main} proves the claim, after using that $\entropy(\pi) \geq 0$ and $\norm{\bm{\lambda}}^2 \geq 0$.
\end{proof}

\section{Last-Iterate Convergence} \label{app:convergence}

In this section, we provide the proofs for all results in \cref{sec:li-conv}, resulting in the proof of last-iterate convergence of the regularized primal-dual scheme (\cref{eq:reg-primal,eq:reg-dual}).

We first establish the convergence of the aforementioned potential function $\Phi_k$.

\potentialFinite*

\begin{proof}
    We first decompose the $k$-th \emph{primal-dual gap} as follows:
    \begin{align}
        \regLagrangian(\regPiStar, \bm{\lambda}_k) - \regLagrangian(\pi_k, \regLambdaStar) = \underbrace{\regLagrangian(\regPiStar, \bm{\lambda}_k) - \regLagrangian(\pi_k, \bm{\lambda}_k)}_{(i)} + \underbrace{\regLagrangian(\pi_k, \bm{\lambda}_k) - \regLagrangian(\pi_k, \regLambdaStar)}_{(ii)}. \label{eq:pd-decomp-finite}
    \end{align}
    We first bound term (i):
    \begin{align}
        (i) =& \regLagrangian(\regPiStar, \bm{\lambda}_k) - \regLagrangian(\pi_k, \bm{\lambda}_k)\nonumber\\
        =& \valStart{r + \bm{\lambda}_k^T\bm{g}}{\regPiStar} - \valStart{r + \bm{\lambda}_k^T\bm{g}}{\pi_k} \tag{as $\occ{h}{\pi}{s,a}=\occ{h}{\pi}{s}\pi_h(a|s)$, and cancel $\norm{\bm{\lambda}_k}^2$}\nonumber\\
        &- \tau \sum_{s,a,h} \occ{h}{\regPiStar}{s} \regPiStarH{h}(a|s) \log(\regPiStarH{h}(a|s)) + \tau \sum_{s,a,h} \occ{h}{\pi_k}{s} \pi_{k,h}(a|s) \log(\pi_{k,h}(a|s))\nonumber\\
        =& \valStart{r + \bm{\lambda}_k^T\bm{g} + \tau \psi_k}{\regPiStar} - \valStart{r + \bm{\lambda}_k^T\bm{g} + \tau \psi_k}{\pi_k} \quad \tag{since $\psi_{k,h}(s,a) = -\log(\pi_{k,h}(a|s))$} \nonumber\\
        &+ \tau \sum_{s,a,h} \occ{h}{\regPiStar}{s} \regPiStarH{h}(a|s) \log(\pi_{k,h}(a|s)) - \tau \sum_{s,a,h} \occ{h}{\pi_k}{s} \pi_{k,h}(a|s) \log(\pi_{k,h}(a|s))\nonumber\\
        &- \tau \sum_{s,a,h} \occ{h}{\regPiStar}{s} \regPiStarH{h}(a|s) \log(\regPiStarH{h}(a|s)) + \tau \sum_{s,a,h} \occ{h}{\pi_k}{s} \pi_{k,h}(a|s) \log(\pi_{k,h}(a|s))\nonumber\\
        =& \valStart{r + \bm{\lambda}_k^T\bm{g} + \tau \psi_k}{\regPiStar} - \valStart{r + \bm{\lambda}_k^T\bm{g} + \tau \psi_k}{\pi_k} \nonumber\\
        &+ \tau \sum_{s,a,h} \occ{h}{\regPiStar}{s} \regPiStarH{h}(a|s) \log(\pi_{k,h}(a|s)) \nonumber\\
        &- \tau \sum_{s,a,h} \occ{h}{\regPiStar}{s} \regPiStarH{h}(a|s) \log(\regPiStarH{h}(a|s)) \nonumber\\
        =& \valStart{r + \bm{\lambda}_k^T\bm{g} + \tau \psi_k}{\regPiStar} - \valStart{r + \bm{\lambda}_k^T\bm{g} + \tau \psi_k}{\pi_k} \nonumber\\
        &- \tau \sum_{s,h} \occ{h}{\regPiStar}{s} \sum_{a} \regPiStarH{h}(a|s) \log\round{\frac{\regPiStarH{h}(a|s)}{\pi_{k,h}(a|s)}}\nonumber\\
        =& \valStart{r + \bm{\lambda}_k^T\bm{g} + \tau \psi_k}{\regPiStar} - \valStart{r + \bm{\lambda}_k^T\bm{g} + \tau \psi_k}{\pi_k} - \tau\sum_{s,h} \occ{h}{\regPiStar}{s} \kl_{k,h}(s) \nonumber\\
        =& \valStart{r + \bm{\lambda}_k^T\bm{g} + \tau \psi_k}{\regPiStar} - \valStart{r + \bm{\lambda}_k^T\bm{g} + \tau \psi_k}{\pi_k} - \tau \kl_k \nonumber\\
        =& \valStart{r + \bm{\lambda}_k^T\bm{u} + \tau \psi_k}{\regPiStar} - \valStart{r + \bm{\lambda}_k^T\bm{u} + \tau \psi_k}{\pi_k} - \tau \kl_k \tag{as $\bm{g} = \bm{u} - \frac{1}{H}\bm{c}$}\nonumber\\
        =&\valStart{z_k}{\regPiStar} - \valStart{z_k}{\pi_k} - \tau \kl_k \nonumber\\
        =& \sum_{s,h} \occ{h}{\regPiStar}{s} \left\langle \qval{z_k,h}{\pi_k}{s,\cdot},  \regPiStarH{h}(\cdot|s) - \pi_{k,h}(\cdot|s) \right\rangle - \tau \kl_k \tag{by \cref{lem:pdl-finite}} \nonumber
    \end{align}
    Note that for all $s,h$,
    \begin{align*}
        &\left\langle \qval{z_k,h}{\pi_k}{s,\cdot},  \regPiStarH{h}(\cdot|s) - \pi_{k,h}(\cdot|s) \right\rangle\\
        \leq& \frac{\kl_{k,h}(s) - \kl_{k+1,h}(s)}{\eta} + \frac{\eta}{2}\sum_{a} \pi_{k,h}(a|s)\exp\round{\qval{z_k,h}{\pi_k}{s,a}} \qval{z_k,h}{\pi_k}{s,a}^2 \tag{\cref{lem:md-descent-kl}}\\
        \leq&\frac{\kl_{k,h}(s) - \kl_{k+1,h}(s)}{\eta} \tag{\cref{lem:val-bounds-finite}} \\
        &+ \frac{\eta}{2} A^{1/2}\exp\round{\eta H \round{1 + \lambdaMax I + \tau\log(A)}} \round{ 2H^2 \round{1 + I \lambdaMax + \tau \log(A)}^2 + 2\tau^2 (64/e^2) } \\
        =& \frac{\kl_{k,h}(s) - \kl_{k+1,h}(s)}{\eta} + \frac{\eta}{2} \frac{1}{H}\dTauLambda, 
    \end{align*}
    with 
    \begin{align*}
        \dTauLambda = HA^{1/2}\exp\round{\eta H \round{1 + \lambdaMax I + \tau\log(A)}} \round{ 2H^2 \round{1 + I \lambdaMax + \tau \log(A)}^2 + 2\tau^2 (64/e^2) }
    \end{align*}
    and where we were able to apply \cref{lem:md-descent-kl} by \cref{lem:update-md-implicit} and since $\qval{z,h}{\pi_k}{s,a}\geq0$. Hence, 
    \begin{align*}
        \sum_{s,h} \occ{h}{\regPiStar}{s} \left\langle \qval{z_k,h}{\pi_k}{s,\cdot},  \regPiStarH{h}(\cdot|s) - \pi_{k,h}(\cdot|s) \right\rangle
        \leq& \sum_{s,h} \occ{h}{\regPiStar}{s} \round{\frac{\kl_{k,h}(s) - \kl_{k+1,h}(s)}{\eta} + \frac{\eta}{2} \frac{1}{H}\dTauLambda}\\
        =& \frac{\kl_k - \kl_{k+1}}{\eta} + \frac{\eta}{2} \dTauLambda.
    \end{align*}
    Plugging in, we thus find 
    \begin{align}
        (i) = \valStart{z_k}{\regPiStar} - \valStart{z_k}{\pi_k} - \tau \kl_k \leq& \frac{\kl_k - \kl_{k+1}}{\eta} + \frac{\eta}{2} \dTauLambda - \tau \kl_k = \frac{(1-\eta\tau)\kl_k - \kl_{k+1}}{\eta} + \frac{\eta}{2} \dTauLambda. \label{eq:rpg-i-bound-finite}
    \end{align}
    We now bound term (ii):
    \begin{align}
        (ii) =& \regLagrangian(\pi_k, \bm{\lambda}_k) - \regLagrangian(\pi_k, \regLambdaStar)\nonumber\\
        =& \valStart{r+\bm{\lambda}_k^T \bm{g}}{\pi_k} - \valStart{r+(\regLambdaStar)^T \bm{g}}{\pi_k} + \frac{\tau}{2} \norm{\bm{\lambda}_k}^2 - \frac{\tau}{2} \norm{\regLambdaStar}^2 \tag{cancel $\entropy(\pi_k)$}\nonumber\\
        =& \sum_{i} (\lambda_{k,i} - \regLambdaStarI{i}) \valStart{g_i}{\pi_k} + \frac{\tau}{2} \norm{\bm{\lambda}_k}^2 - \frac{\tau}{2} \norm{\regLambdaStar}^2\nonumber\\
        =& \sum_{i} (\lambda_{k,i} - \regLambdaStarI{i}) (\valStart{u_i}{\pi_k} - c_i + \tau \lambda_{k,i}) - \frac{\tau}{2} \norm{\bm{\lambda}_k - \regLambdaStar}^2 \nonumber\\
        \leq& \frac{\norm{\regLambdaStar - \bm{\lambda}_k}^2 - \norm{\regLambdaStar - \bm{\lambda}_{k+1}}^2}{2\eta} +\frac{\eta}{2}\norm{\valStartVec{u_{k}}{\pi_k} - \bm{c} + \tau \bm{\lambda}_k}^2 \tag{\cref{lem:md-descent-proj}}\\
        \leq& \frac{\norm{\regLambdaStar - \bm{\lambda}_k}^2 - \norm{\regLambdaStar - \bm{\lambda}_{k+1}}^2}{2\eta} + \frac{\eta}{2} \dTauLambdaPrime, \tag{\cref{lem:val-bounds-finite}}
    \end{align}
    with $\dTauLambdaPrime = I(H + \tau \lambdaMax)^2$ and where we were able to apply \cref{lem:md-descent-proj} by \cref{lem:update-md-implicit}. Plugging in, we find 
    \begin{align}
        (ii) =& \sum_{i} (\lambda_{k,i} - \regLambdaStarI{i} ) (\valStart{u_i}{\pi_k} - c_i + \tau \lambda_{k,i}) - \frac{\tau}{2} \norm{\bm{\lambda}_k - \regLambdaStar}^2 \nonumber\\
        \leq& \frac{\norm{\regLambdaStar - \bm{\lambda}_k}^2 - \norm{\regLambdaStar - \bm{\lambda}_{k+1}}^2}{2\eta} + \frac{\eta}{2} \dTauLambdaPrime
        - \frac{\tau}{2} \norm{\bm{\lambda}_k - \regLambdaStar}^2\nonumber\\
        =& \frac{(1-\eta\tau)\norm{\regLambdaStar - \bm{\lambda}_k}^2 - \norm{\regLambdaStar - \bm{\lambda}_{k+1}}^2}{2\eta} + \frac{\eta}{2} \dTauLambdaPrime. \label{eq:rpg-ii-bound-finite}
    \end{align}
    From \cref{lem:sp-cmdp-finite} (with $\pi=\pi_k$, $\bm{\lambda}=\bm{\lambda}_k$), we have $0 \leq \regLagrangian(\regPiStar, \bm{\lambda}_k) - \regLagrangian(\pi_k, \regLambdaStar)$. Moreover, recall $\Phi_k = \kl_k + \frac{1}{2} \norm{\bm{\lambda}_k - \regLambdaStar}^2$, thus by \cref{eq:rpg-i-bound-finite,eq:rpg-ii-bound-finite},
    \begin{align*}
        \Phi_{k+1} =& \kl_{k+1} + \frac{1}{2} \norm{\bm{\lambda}_{k+1} - \regLambdaStar}^2\\
        \leq& (1-\eta\tau) \kl_k + \frac{\eta^2}{2} \dTauLambda - \eta(i) + (1-\eta\tau) \frac{\norm{\bm{\lambda}_k - \regLambdaStar}^2}{2} + \frac{\eta^2}{2} \dTauLambdaPrime - \eta(ii) \tag{\cref{eq:rpg-i-bound-finite,eq:rpg-ii-bound-finite}}\\
        \leq& (1-\eta\tau) \Phi_k + \eta^2(\dTauLambda + \dTauLambdaPrime)\tag{Def. $\Phi_k$} - \eta \round{(i) + (ii)} \\
        \leq& (1-\eta\tau) \Phi_k + \eta^2(\dTauLambda + \dTauLambdaPrime) - \eta \round{\regLagrangian(\regPiStar, \bm{\lambda}_k) - \regLagrangian(\pi_k, \regLambdaStar)} \tag{\cref{eq:pd-decomp-finite}}\\
        \leq& (1-\eta\tau) \Phi_k + \eta^2(\dTauLambda + \dTauLambdaPrime) \tag{as $\regLagrangian(\regPiStar, \bm{\lambda}_k) - \regLagrangian(\pi_k, \regLambdaStar) \geq 0$}.
    \end{align*}
    Finally, the claimed bound follows by noting that 
    \begin{align*}
        &\dTauLambda + \dTauLambdaPrime \\
        =& HA^{1/2}\exp\round{\eta H \round{1 + \lambdaMax I + \tau\log(A)}} \round{ 2H^2 \round{1 + I \lambdaMax + \tau \log(A)}^2 + 2\tau^2 (64/e^2) } + I(H + \tau \lambdaMax)^2\\
        \leq& \tilde{O} \round{\lambdaMax^2 H^3A^{1/2}I^2 \exp\round{\eta H \round{1 + \lambdaMax I + \log(A)}} + I(H + \tau \lambdaMax)^2},
    \end{align*}
    as $\tau \leq 1$ and $\lambdaMax\geq H\Xi^{-1}\geq 1$.
\end{proof}

We can use the following result to turn the convergence of the potential function into an error bound. We will then choose the optimal values for $\lambdaMax$, $\tau$, and $\eta$.

\potToRegretFinite*

\begin{proof}
    (1) We bound the objective optimality gap. First, decompose it as
    \begin{align}
        \valStart{r}{\piStar} - \valStart{r}{\pi_k} = \underbrace{ \valStart{r}{\piStar} -  \valStart{r}{\regPiStar}}_{(i)} + \underbrace{\valStart{r}{\regPiStar} - \valStart{r}{\pi_k}}_{(ii)}.
    \end{align}
    We bound (ii) as follows:
    \begin{align*}
        (ii) =& \valStart{r}{\regPiStar} - \valStart{r}{\pi_k}\\
        =& \sum_{s,a,h} \occ{h}{\regPiStar}{s} \round{ \regPiStarH{h}(a|s) - \pi_{k,h}(a|s) } \qval{r,h}{\pi_k}{s,a}
        \tag{\cref{lem:pdl-finite}}\\
        \leq& H \sum_{s,h} \occ{h}{\regPiStar}{s} \norm{\regPiStarH{h}(\cdot|s) - \pi_{k,h}(\cdot|s)}_1\\
        \leq& H \sum_{s,h} \occ{h}{\regPiStar}{s} \sqrt{2\kl_{k,h}(s)}
        \tag{by Pinsker's} \\ 
        \leq& H^2 \sqrt{2 \sum_{s,h} \frac{1}{H}\occ{h}{\regPiStar}{s}\kl_{k,h}(s)}
        \tag{by Jensen's} \\ 
        =& H^{3/2} \sqrt{2 \kl_k}.
    \end{align*}
    We next bound term (i). By \cref{lem:sp-reg-cmdp-finite} with $\pi=\piStar$ we have 
    \begin{align*}
        \valStart{r}{\piStar} - \tau \entropy(\regPiStar) \leq \valStart{r}{\regPiStar} + \sum_{i} \regLambdaStarI{i} \round{\valStart{g_i}{\regPiStar} - \valStart{g_i}{\piStar}}.
    \end{align*}
    By \cref{lem:sp-reg-cmdp-finite} with $\bm{\lambda}=\bm{0}$ we have 
    \begin{align*}
        \sum_{i} \regLambdaStarI{i} \valStart{g_i}{\regPiStar} \leq 0.
    \end{align*}
    Moreover, $\valStart{g_i}{\piStar} \geq 0$ by feasibility and $\regLambdaStarI{i} \geq 0$. Combing these inequalities, we find 
    \begin{align}
        (i) =& \valStart{r}{\piStar} -  \valStart{r}{\regPiStar} \leq \tau \entropy(\regPiStar) \leq \tau H\log(A),
    \end{align}
    which concludes the proof for the objective optimality gap.

    (2) Let $i \in [I]$. We now bound the $i$-th constraint violation. First, decompose it as
    \begin{align}
        c_i - \valStart{u_i}{\pi_k} = - \valStart{g_i}{\pi_k} = \underbrace{- \valStart{g_i}{\regPiStar}}_{(iii)} + \underbrace{\valStart{g_i}{\regPiStar} - \valStart{g_i}{\pi_k}}_{(iv)}
    \end{align}
    We first bound (iv). The same calculation as for the objective optimality gap (1) shows
    \begin{align}
        (iv) =& \valStart{g_i}{\regPiStar} - \valStart{g_i}{\pi_k} \leq H^{3/2} \sqrt{2\kl_k}.
    \end{align}
    We next bound term (iii). Recall $\Lambda = [0,\lambdaMax]^I$. 
    \cref{lem:sp-reg-cmdp-finite} with $\pi = \piStar$ and $\bm{\lambda} \in \Lambda$ as
    \begin{align*}
        \lambda_j := \begin{cases}
            0 \quad (j \neq i)\\ 
            \lambdaMax \quad (j = i)
        \end{cases}
    \end{align*}
    yields
    \begin{align*}
         \valStart{r}{\piStar} + \sum_{j} \regLambdaStarI{j} \valStart{g_j}{\piStar} \leq& \valStart{r}{\regPiStar} + \lambdaMax \valStart{g_i}{\regPiStar} + \frac{\tau}{2} \lambdaMax^2 + \tau \entropy(\regPiStar)
    \end{align*}
    From \cref{lem:sp-cmdp-finite} (with $\pi=\regPiStar$) we get
    \begin{align*}
        \valStart{r}{\regPiStar} - \valStart{r}{\piStar} \leq \sum_{j} \lambdaStarI{j} \round{ \valStart{g_j}{\piStar} - \valStart{g_j}{\regPiStar} }.
    \end{align*}
    Adding the two previous inequalities and canceling terms, we get 
    \begin{align*}
        0 \leq \sum_{j} \regLambdaStarI{j} \valStart{g_j}{\piStar} \leq& \lambdaMax \valStart{g_i}{\regPiStar} + \frac{\tau}{2} \lambdaMax^2 + \sum_{j} \lambdaStarI{j} \round{ \valStart{g_j}{\piStar} - \valStart{g_j}{\regPiStar} } + \tau \entropy(\regPiStar),
    \end{align*}
    where the first inequality holds since $0 \leq \valStart{g_j}{\piStar}$ by feasibility and $\regLambdaStar\geq \bm{0}$. Rearranging this shows
    \begin{align*}
        - \valStart{g_i}{\regPiStar} \leq& \frac{\tau}{2} \lambdaMax + \frac{1}{\lambdaMax} \sum_{j} \lambdaStarI{j} \round{ \valStart{g_j}{\piStar} - \valStart{g_j}{\regPiStar} } + \frac{1}{\lambdaMax} \tau \entropy(\regPiStar)\\
        =& \frac{\tau}{2} \lambdaMax + \frac{1}{\lambdaMax} \sum_{j} \lambdaStarI{j} \round{ \valStart{u_j}{\piStar} - \valStart{u_j}{\regPiStar} } + \frac{1}{\lambdaMax} \tau \entropy(\regPiStar) \tag{$\bm{g}=\bm{u}-\frac{1}{H}\bm{c}$}\\
        \leq& \frac{\tau}{2} \lambdaMax + \frac{1}{\lambdaMax} \norm{\lambdaStar}_1 H + \frac{1}{\lambdaMax} \tau \entropy(\regPiStar) \tag{Hölder's}\\
        \leq& \frac{\tau}{2} \lambdaMax + \frac{1}{\lambdaMax} \round{\frac{H^2}{\Xi} + \tau \entropy(\regPiStar)} \tag{\cref{lem:dual-bound-finite}}\\
        \leq& \frac{\tau}{2} \lambdaMax + \frac{1}{\lambdaMax}\round{\frac{H^2}{\Xi} + \tau H\log(A)}.
    \end{align*}
\end{proof}

Finally, we are ready to prove last-iterate convergence by combining the previous two lemmas.

\lastIterate*

\begin{proof}
    The bound follows from \cref{lem:potential-finite} and \cref{lem:pot-to-regret-finite}. We choose $\tau = \epsilon^2$, $\eta = (H^2 I \log(A))^{-1}\Xi \epsilon^6$, $\lambdaMax = \frac{H}{\Xi} \epsilon^{-1} \geq \frac{H}{\Xi}$. Set $\Delta_r(k) := \rectangular{\valStart{r}{\piStar} - \valStart{r}{\pi_k}}_+$ and $\Delta_{g_i}(k) := \rectangular{- \valStart{g_i}{\pi_k}}_+$.

    We first consider the suboptimality for the reward. Plugging \cref{lem:potential-finite} into \cref{lem:pot-to-regret-finite} we find, using $\sqrt{a+b} \leq \sqrt{a} + \sqrt{b}$ and $1+x\leq \exp(x)$,
    \begin{align*}
        \Delta_r(k) \leq& H^{3/2}\Phi_1^{1/2} \exp\round{-\eta\tau k/2} \tag{a}\\
        &+ H^{3/2} \round{\frac{\eta}{\tau}}^{1/2} \tilde{O}(\cTauLambda^{1/2}) \tag{b}\\
        &+ \tau H\log(A). \tag{c}
    \end{align*}
    For (b), note that, using the definitions of $\eta$, $\tau$, $\lambdaMax$ (and taking $\sqrt{\cdot}$, and $\tau < 1$)
    \begin{align*}
        \cTauLambda^{1/2} \leq& \lambdaMax H^{3/2}A^{1/4}I \exp\round{\eta H \round{1 + \lambdaMax I + \log(A)}/2} + I^{1/2}(H + \tau \lambdaMax)\\
        \leq& \lambdaMax H^{3/2}A^{1/4}I \exp\round{2} + I^{1/2}(H + \tau \lambdaMax)\\
        =& \epsilon^{-1} \cdot H^{5/2}A^{1/4}I \Xi^{-1} \exp\round{2} + I^{1/2}H + I^{1/2} \epsilon^2 H\Xi^{-1}\epsilon^{-1})\\
        \lesssim& \epsilon^{-1} \cdot H^{5/2}A^{1/4}I \Xi^{-1}.
    \end{align*}
    Since 
    \begin{align*}
        \round{\frac{\eta}{\tau}}^{1/2} = (H^2I \log(A))^{-1/2}\Xi^{1/2} \epsilon^{(6 - 2)/2} = (H^2I\log(A))^{-1/2}\Xi^{1/2}\epsilon^{2},
    \end{align*}
    we thus have
    \begin{align*}
        (b) = H^{3/2} \round{\frac{\eta}{\tau}}^{1/2} \cTauLambda \lesssim H^{3}I^{1/2}A^{1/4} \Xi^{-1/2} \epsilon = \text{poly}(A,H,I, \Xi^{-1}) \cdot \epsilon.
    \end{align*}
    Similarly,
    \begin{align*}
        (c) = \tau H\log(A) =  H \log(A) \epsilon^2.
    \end{align*}
    For (a), using the standard inequality $e^{-x} \leq 1-x/2$ (if $0 \leq x \leq 1$) with $x := \eta\tau/2$, we first find 
    \begin{align*}
        \exp(-\eta\tau l/2) \leq (1 - \eta\tau /4)^l
    \end{align*}
    and hence,
    \begin{align*}
        (a) = H^{3/2}\Phi_1^{1/2} \exp\round{-\eta\tau k/2} \leq& H^{3/2} \Phi_1^{1/2} \cdot \frac{1}{k} \sum_{l=1}^k \exp\round{-\eta\tau l/2} \\
        \leq& H^{3/2}\Phi_1^{1/2} \cdot \frac{1}{k} \sum_{l=1}^k (1 - \eta\tau /4)^l\\
        \leq& H^{3/2}\Phi_1^{1/2} \cdot \frac{1}{k} \sum_{l=1}^{\infty} (1 - \eta\tau /4)^l\\
        =& H^{3/2}\Phi_1^{1/2} \cdot \frac{1}{k} \frac{4}{\eta\tau}\\
        =& H^{3/2}\Phi_1^{1/2} \frac{1}{k} \frac{4}{(H^2I \log(A))^{-1}\Xi \epsilon^6 \epsilon^2}\\
        =& 4H^{7/2} I \Xi^{-1} \log(A) \frac{1}{k} \Phi_1^{1/2} \epsilon^{-8}.
    \end{align*}
    Furthermore, since $\pi_1$ plays actions uniformly at random and $\bm{\lambda_1}=\bm{0}$, we have $\Phi_1^{1/2} \leq (H\log(A) + \frac{1}{2}I\lambdaMax^2)^{1/2} \leq H^{1/2}\log(A)^{1/2} + I^{1/2}\lambdaMax = H^{1/2}\log(A)^{1/2} + I^{1/2}H\Xi^{-1} \epsilon^{-1}$. Hence, the calculation above shows
    \begin{align*}
        (a) \leq& 4H^{7/2} I \Xi^{-1} \log(A) \frac{1}{k} (H^{1/2}\log(A)^{1/2} + I^{1/2}H\Xi^{-1} \epsilon^{-1}) \epsilon^{-8} \leq \text{poly}(A,H,I,\Xi^{-1}) \epsilon
    \end{align*}
    for $k = \Omega(\epsilon^{-10})$. Hence, summing up terms (a) to (c) and choosing $k = \Omega(\text{poly}(A,H,I,\Xi^{-1}) \epsilon^{-10})$ yields the bound for the objective.

    Next, we consider the regret for the constraints. Plugging \cref{lem:potential-finite} into \cref{lem:pot-to-regret-finite} we find, using $\sqrt{a+b} \leq \sqrt{a} + \sqrt{b}$,
    \begin{align*}
        \Delta_{u_i}(k) \leq& H^{3/2}\Phi_1^{1/2} \exp\round{-\eta\tau k/2} \tag{a'}\\
        &+ H^{3/2} \round{\frac{\eta}{\tau}}^{1/2} \tilde{O}( \cTauLambda^{1/2}) \tag{b'}\\
        &+ \tau \lambdaMax + \frac{1}{\lambdaMax}H^2 \Xi^{-1} \tag{c'}\\
        &+ \frac{1}{\lambdaMax} \tau H\log(A). \tag{d'}
    \end{align*}
    Note that terms (a'), (b') are identical to (a), (b). Moreover, for (d') we have
    \begin{align*}
        (d') = \frac{1}{\lambdaMax} \tau\log(A) =& \Xi H^{-1} \log(A) \epsilon^3.
    \end{align*}
    Finally, for (c'), we have 
    \begin{align*}
        (c') =& \tau \lambdaMax + \frac{1}{\lambdaMax}H^2 \Xi^{-1}\\
        =& \epsilon^2 \cdot H\Xi^{-1} \epsilon^{-1} + H^{-1}\Xi \epsilon \cdot H^2 \Xi^{-1}\\
        =& H(1 + \Xi^{-1}) \epsilon.
    \end{align*}
    Thus, summing up (a') to (d') and choosing $k = \Omega(\text{poly}(A,H,I,\Xi^{-1})\epsilon^{-10})$ yields the bound for the constraints.
\end{proof}

\section{Properties of the Optimistic Model} \label{app:model}

In this section, we establish important properties of the model \cref{algo:rpg-pd-finite-learn} builds.

\subsection{Building the Model} \label{sec:building-model}

First, we describe the exact model and how we perform policy evaluation.

We follow \citet{shani2020optimistic} for the optimistic exploration, but we also take the $I$ constraint functions $\bm{u}$ into account rather than just the reward function $r$. We also need to pay special attention to the auxiliary term $\psi_k$.

For all $s,a,h$ and $k\in[K]$, let $\counter := \sum_{l=1}^{k-1} \indic_{\{ s_h^l=s,~ a_h^l = a \}}$ count the number of times that the state-action pair $(s,a)$ has been visited at step $h$ before episode $k$. Here, ($s_h^l$, $a_h^l$) denotes the state-action pair visited at step $h$ in episode $l$. First, we compute the empirical averages of the reward and transition probabilities as follows:
\begin{align*}
    \rBar_{k-1,h}(s,a) :=& \frac{\sum_{l=1}^{k-1} R_{h}^{l}(s,a) \indic_{\{ s_h^l=s,~ a_h^l = a \}}}{\counterMax}, \\
    \uBar_{k-1,i,h}(s,a) :=& \frac{\sum_{l=1}^{k-1} U_{i,h}^{l}(s,a) \indic_{\{ s_h^l=s,~ a_h^l = a \}}}{\counterMax} \quad (\forall i\in [I]), \\
    \pBar_{k-1,h}(s'|s,a) :=& \frac{\sum_{l=1}^{k-1} \indic_{\{ s_h^l=s,~ a_h^l = a,~ s_{h+1}^l=s' \}}}{\counterMax},
\end{align*}
where $a \vee b := \max\{a,b\}$. We consider optimistic estimates $\rHat_k$, $\bm{\uHat}_k$, $\pHat_k$:
\begin{align*}
    \rHat_{k,h}(s,a) &:= \rBar_{k-1,h}(s,a) + b_{k-1,h}(s,a), \nonumber\\
    \uHat_{k,i,h}(s,a) &:= \uBar_{k-1,i,h}(s,a) + b_{k-1,h}(s,a) \quad (\forall i\in [I]),\\
    \pHat_{k,h}(s' | s,a) &:= \pBar_{k-1,h}(s' | s,a),
\end{align*}
with the bonuses $b_{k-1,h}(s,a) = b^r_{k-1,h}(s,a) + b^p_{k-1,h}(s,a)$ specified below. For $\psi_k$, we take\footnote{In other words, there is no bonus for the function, only for the transitions. This is because $\psi_k$ is known in episode $k$, and so the only uncertainty in the corresponding value function is due to estimating the transitions $p$. Note that the extra $\log(A)$ factor corrects for the fact that $\psi_k$ is not a function to $[0,1]$.}
\begin{align*}
    \psiHat_{k,h}(s,a) := \psi_{k,h}(s,a) + b^p_{k-1,h}(s,a) \log(A).
\end{align*}
For notational convenience, we write 
\begin{align*}
    z_k := r + \bm{\lambda}_k^T \bm{u} + \tau \psi_k, \quad \quad \quad \zHat_k := \rHat_{k} + \bm{\lambda}_k^T\bm{\uHat}_k + \tau \psiHat_k
\end{align*}
for the reward function mimicking the $\pi$-dependency of the regularized Lagrangian at $(\pi_k, \bm{\lambda}_k)$.

For any $\delta \in (0,1)$, we specify the correct bonuses to obtain our regret guarantees with probability at least $1- \delta$:
\begin{align*}
    b_{k-1,h}(s,a) :=& b^r_{k-1,h}(s,a) + b^p_{k-1,h}(s,a),
\end{align*}
where
\begin{align*} 
    b^r_{k-1,h}(s,a) := \sqrt{\frac{\frac{1}{2}\log\round{ \frac{2SAH(I+1)K}{\delta'} }}{\counterMax}}, \quad
    b^p_{k-1,h}(s,a) := H \sqrt{\frac{2S + 2\log\round{ \frac{SAHK}{\delta'} }}{\counterMax}}. 
\end{align*} 
For $\psi_k$, recall
\begin{align*}
    \psiHat_{k,h}(s,a) := \psi_{k,h}(s,a) + b^p_{k-1,h}(s,a) \log(A).
\end{align*}

We define the \emph{truncated value functions}\footnote{Importantly, note that this is the definition of $\qvalHat{\zHat_k,h}{k}{s,a}$, rather than running truncated policy evaluation on $z_k$.}
\begin{align}
    \qvalHat{\zHat_k,h}{k}{s,a} :=& \qvalHat{\rHat_k,h}{k}{s,a} + \sum_{i} \lambda_{k,i} \qvalHat{\uHat_{k,i},h}{k}{s,a} + \tau \qvalHat{\psiHat_k,h}{k}{s,a} \label{eq:def-qk},\\
    \valHat{\zHat_k,h}{k}{s} :=& \dotProd{\pi_{k,h}(\cdot|s)}{\qvalHat{\zHat_k,h}{k}{s,\cdot}}, \label{eq:def-vk}
\end{align}
where we compute $\qvalHat{\rHat_k,h}{k}{s,a}$, $\qvalHat{\uHat_{k,i},h}{k}{s,a}$, $\qvalHat{\psiHat_k,h}{k}{s,a}$ via truncated policy evaluation with respect to the estimated model, see \cref{algo:trun-eval}.
\begin{algorithm}[H]
    \begin{algorithmic}
        \small
        \STATE{Initialize $\valHat{\rHat_k,H+1}{k}{s}=\valHat{\uHat_{k,i},H+1}{k}{s}=\valHat{\psiHat_{k},H+1}{k}{s} =0$ (for $s\in\calS$)}
        \FOR{$h=H, \dots, 1$}
            \FOR{$(s,a) \in \calS \times \calA$}
                    \STATE{Truncated DP step:}
                    \begin{align*}
                        &\qvalHat{\rHat_k,h}{k}{s,a} := \min\curly{H-h+1,~ \rHat_{k,h}(s,a) + \dotProd{\pHat_{k,h}(\cdot|s,a)} {\valHat{\rHat_k,h+1}{k}{\cdot}}} \\
                        &\qvalHat{\uHat_{k,i},h}{k}{s,a} := \min\curly{H-h+1,~ \uHat_{k,i,h}(s,a) + \dotProd{\pHat_{k,h}(\cdot|s,a)} {\valHat{\uHat_{k,i},h+1}{k}{\cdot}}} \quad (\forall i \in [I])\\
                        &\qvalHat{\psiHat_k,h}{k}{s,a} := \min\curly{\psi_{k,h}(s,a) + (H-h+1)\log(A),~ \psiHat_{k,h}(s,a) + \dotProd{\pHat_{k,h}(\cdot|s,a)} {\valHat{\psiHat_k,h+1}{k}{\cdot}}}
                    \end{align*}
                    \STATE{Retrieve $V$-function:}
                    \begin{align*}
                        \valHat{\rHat_k,h}{k}{s} :=& \dotProd{\pi_{k,h}(\cdot|s)}{\qvalHat{\rHat_k,h}{k}{s,\cdot}} \\
                        \valHat{\uHat_{k,i},h}{k}{s} :=& \dotProd{\pi_{k,h}(\cdot|s)}{\qvalHat{\uHat_{k,i},h}{k}{s,\cdot}} \quad (\forall i \in [I]) \\
                        \valHat{\psiHat_k,h}{k}{s} :=& \dotProd{\pi_{k,h}(\cdot|s)}{\qvalHat{\psiHat_k,h}{k}{s,\cdot}}
                    \end{align*}
            \ENDFOR
        \ENDFOR
    \OUTPUT{$\qvalHat{\zHat_k}{k}{\cdot} := \qvalHat{\rHat_k}{k}{\cdot} +\sum_i \lambda_{k,i}\qvalHat{\uHat_{k,i}}{k}{\cdot} + \tau \qvalHat{\psiHat_k}{k}{\cdot}$, and $(\valHat{\uHat_{k,i}}{k}{s_1,1})_i$}
    \end{algorithmic}
    \caption{\textsc{Eval} (Truncated Policy Evaluation)} \label{algo:trun-eval}
\end{algorithm}
The main reason for truncating during the otherwise standard policy evaluation algorithm is the need for a bonus-independent upper bound on the surrogate value functions so that \cref{lem:optimism-val-z} holds.\footnote{In fact, truncation is only required for the update of $\pi$, and for the update of $\bm{\lambda}$, we can use either truncated or exact values.} Clearly, the truncated value functions are all lower bounded by zero and upper bounded by the actual value functions under the estimated model. Finally, note that these truncated value functions need not correspond to the true value function of a policy in some MDP. 

Recall the truncated value functions from \cref{algo:trun-eval} in \cref{sec:building-model}. Note that due to the separate definition of $\qvalHat{\zHat_k,h}{k}{s,a}$ and the updates in \cref{algo:trun-eval}, for all $s\in\calS$, $h\in[H]$,
\begin{align}
    \valHat{\zHat_k,h}{k}{s} =& \dotProd{\pi_{k,h}(\cdot|s)}{\qvalHat{\zHat_k,h}{k}{s,\cdot}} \tag{by def.} \nonumber\\
    =& \dotProd{\pi_{k,h}(\cdot|s)}{\qvalHat{\rHat_k,h}{k}{s,a} + \sum_{i} \lambda_{k,i} \qvalHat{\uHat_{k,i},h}{k}{s,a} + \tau \qvalHat{\psiHat_k,h}{k}{s,a}} \tag{by def.}\nonumber\\
    =& \dotProd{\pi_{k,h}(\cdot|s)}{\qvalHat{\rHat_k,h}{k}{s,a}} + \sum_{i} \lambda_{k,i} \dotProd{\pi_{k,h}(\cdot|s)}{\qvalHat{\uHat_{k,i},h}{k}{s,a}} \nonumber\\
    &+ \tau\dotProd{\pi_{k,h}(\cdot|s)}{\qvalHat{\psiHat_k,h}{k}{s,a}}\nonumber \\
    =& \valHat{\rHat_k,h}{k}{s} + \sum_{i} \lambda_{k,i} \valHat{\uHat_{k,i},h}{k}{s} + \tau\valHat{\psiHat_k,h}{k}{s}. \label{eq:trun-v-lin}
\end{align}
Similarly, by linearity of expectation
\begin{align}
    \val{z_k,h}{\pi_k}{s} = \val{r_k,h}{\pi_k}{s} + \sum_{i} \lambda_{k,i} \val{u_{k,i},h}{\pi_k}{s} + \tau\val{\psi_k,h}{\pi_k}{s} \label{eq:v-lin}
\end{align}
for the true value functions.

\subsection{Properties of the Model} \label{sec:model-props}

We are now ready to establish the properties of the model. In particular, we will show that it is optimistic with respect to the value function and prove bounds on the estimation error during the learning procedure.

\paragraph{Success Event}

We will condition our regret analysis on a \textit{success event} $G$, which we formally define below. Fix $\delta> 0$, and construct the estimated model as in \cref{sec:building-model}. $G$ ensures that (a) the optimistic reward estimates are in fact optimistic and (b) the true transitions are close to the estimated ones, i.e.:
\begin{align*}
     r \leq& \rHat_k, \\
     u_i \leq& \uHat_{k,i} \quad (\forall i\in [I]),\\
    \norm{p_h(\cdot | s,a) - \pBar_{k-1,h}(\cdot | s,a)}_1 H \leq& b^p_{k-1,h}(s,a) \quad (\forall s,a,h),
\end{align*}
for every episode $k\in [K]$. Formally, with $\delta':= \delta/3$, define the \emph{failure events} 
\begin{align*}
    F_k^r :=& \curly{\exists s,a,h \colon \abs{r_h(s,a) - \rBar_{k-1,h}(s,a)} \geq b^r_{k-1,h}(s,a)}, \\
    F_k^{u} :=& \curly{\exists i,s,a,h \colon \abs{u_{i,h}(s,a) - \uBar_{k-1,i,h}(s,a)} \geq b^r_{k-1,h}(s,a)}, \\
    F_k^{p} :=& \curly{\exists s,a,h \colon \norm{p_h(\cdot | s,a) - \pBar_{k-1,h}(\cdot | s,a)}_1 H \geq b^p_{k-1,h}(s,a)}, \\
    F_k^{n} :=& \curly{\exists s,a,h \colon \counter \leq \frac{1}{2} \sum_{j<k} \occ{h}{\pi_j}{s,a} -H\log\round{\frac{SAH}{\delta'}}},
\end{align*}
where $\occ{h}{\pi_j}{s,a}$ refers to the occupancy measure (\cref{app:lagrangian}), and 
\begin{align*}
    F^r :=& \round{\bigcup_{k=1}^K F_k^r} \bigcup \round{\bigcup_{k=1}^K F_k^u}, \\
    F^p :=& \bigcup_{k=1}^K F_k^p, \\
    F^n :=& \bigcup_{k=1}^K F_k^n.
\end{align*}
We define the \emph{success event} $G$ as the complement of $F^r \cup F^p \cup F^n$, i.e.
\begin{align*}
    G := \overline{F^r \cup F^p \cup F^n}.
\end{align*}

We now show that this event holds with high probability. The proof of this theorem relies on standard concentration bounds (specifically, Hoeffding for the rewards and $L1$-concentration for the transitions) and a union bound over all involved indices.
\begin{restatable}[Success event]{lemma}{succEvent} \label{lem:succ-event}
    Let $\delta > 0$ and define the bonuses accordingly. Suppose that for all $k\in [K]$, in episode $k$, policy $\pi_k$ is played. Then $P[G] \geq 1- \delta$.
\end{restatable} 

\begin{proof}
    By Hoeffding's for any possible realization of $\counter$ (and total probability), we have $P[F^r] \leq \delta'$ after union bound over all indices $s$, $a$, $h$ and all episodes $k$. For $\counter=0$ the bound holds trivially.
    
    By the $L1$ concentration bound of \citet[Theorem 2.1]{weissman2003inequalities}, for any possible realization of $\counter$ (and total probability), we have $P[F^p] \leq \delta'$ after union bound over all indices $s$, $a$, $h$ and all episodes $k$. For $\counter=0$ the bound holds trivially.

    By \citet[Corollary E.4]{dann2017unifying}, we also have $P[F^n] \leq \delta'$.

    We conclude by union bound over the three events.
\end{proof}

\paragraph{Decomposition via Extended Value Difference Lemma} The following lemma allows us to decompose the instantaneous regret into three terms that we will bound separately.

\begin{restatable}[Decomposition via simulation lemma]{lemma}{regretDecomp} \label{lem:regret-decomp}
    We have the following decomposition:
    \begin{align*}
        &\valStart{z_k}{\regPiStar} - \valStart{z_k}{\pi_k}\\
        =& \valStartHat{\zHat_k}{k} - \valStart{z_k}{\pi_k} \tag{a}\\
        +& \sum_h \exptn \rectangular{ \left\langle \qvalHat{\zHat_k,h}{k}{s_h,\cdot},  \regPiStarH{h}(\cdot|s_h) - \pi_{k,h}(\cdot|s_h) \right\rangle \biggVert s_1, \regPiStar, p} \tag{b}\\ 
        +& \sum_h \exptn \rectangular{ - \qvalHat{\zHat_k,h}{k}{s_h,a_h} + z_{k,h}(s_h,a_h) + \left\langle p_h(\cdot | s_h,a_h), \valHat{z_k, h+1}{k}{\cdot} \right\rangle \biggVert s_1, \regPiStar, p} . \tag{c}
    \end{align*}
\end{restatable}

\begin{proof}
    First, expand
    \begin{align*}
        \valStart{z_k}{\regPiStar} - \valStart{z_k}{\pi_k} = \round{\valStartHat{\zHat_k}{k} - \valStart{z_k}{\pi_k}} + \round{\valStart{z_k}{\regPiStar} - \valStartHat{\zHat_k}{k}}.
    \end{align*}
    Then apply \cref{lem:ext-val-diff} to $\pi = \pi_k$, $\pi' = \regPiStar$ and $M=(\calS,\calA,\pHat_k,\zHat_k)$, $M'=(\calS,\calA,p,z_k)$ to the second term (after multiplying both sides by $-1$).
\end{proof}

\paragraph{General On-Policy Bounds}

The following two results are standard and will allow us to bound the estimation errors during learning. Consider the setup in which policy $\pi_k$ is derived based on the previous episodes $1,\dots,k-1$, and then played in episode $k$. Recall that for all $s,a,h$ and $k\in[K]$, $\counter := \sum_{l=1}^{k-1} \indic_{\{ s_h^l=s,~ a_h^l = a \}}$ counts the visits of state-action pair $(s,a)$ at step $h$ before episode $k$. We write $\lesssim$ for asymptotic inequality up to polylogarithmic terms.

Note that in the following two lemmas, the exponent of $H$ is different from the one in the referenced proofs. This is because the referenced works consider the case of stationary transition probabilities, whereas we consider non-stationary dynamics. See \citet[Lemmas 18, 19]{shani2020optimistic}.

\begin{lemma}[Lemma 36, \citet{efroni2020exploration}] \label{aug-lem:l36-efroni} 
    Suppose for all $s$, $a$, $h$, $k \in [K]$, we have 
    \begin{align*}
        \counter &> \frac{1}{2} \sum_{j<k} \occ{h}{\pi_j}{s,a} - H \log\left( \frac{SAH}{\delta'}\right).
    \end{align*}
    Then for all $K' \leq K$
    \begin{align*}
       \sum_{k'=1}^{K'} \sum_{h=1}^{H} \E \left[ \frac{1}{\sqrt{n_{k'-1,h}(s_h^{k'},a_h^{k'})}} \mid \mathcal{F}_{k'-1} \right] \leq \tilde{O}( \sqrt{SAH^2K'} + SAH ),
    \end{align*}
    where $\mathcal{F}_{k'-1}$ is the $\sigma$-algebra induced by all random variables up to and including episode $k'-1$.
\end{lemma}

\begin{proof}
    We refer to \citet[Lemma 38]{efroni2019tight} for a proof of the statement.
\end{proof}

\begin{lemma}[Lemma 37, \citet{efroni2020exploration}] \label{aug-lem:37efr} 
    Suppose for all $s$, $a$, $h$, $k \in [K]$, we have 
    \begin{align*}
        \counter &> \frac{1}{2} \sum_{j<k} \occ{h}{\pi_j}{s,a} - H \log\left( \frac{SAH}{\delta'}\right).
    \end{align*}
    Then for all $K' \leq K$
    \begin{align*}
       \sum_{k'=1}^{K'} \sum_{h=1}^{H} \E \left[ \frac{1}{n_{k'-1,h}(s_h^{k'},a_h^{k'})} \mid \mathcal{F}_{k'-1} \right] \leq \tilde{O}( SA H^2),
    \end{align*}
    where $\mathcal{F}_{k'-1}$ is the $\sigma$-algebra induced by all random variables up to and including episode $k'-1$.
\end{lemma}

\begin{proof}
    We refer to \citet[Lemma 13]{zanette2019tighter} for a proof of the statement. 
\end{proof}

\paragraph{Estimation Error (On-Policy Error Bounds)}

We next prove bounds on the estimation error obtained while learning the model.

\begin{lemma}[Estimation error $\rHat$, $\bm{\uHat}$] \label{lem:est-error-ru}
    Conditioned on $G$, for every $K' \in [K]$, we have
    \begin{align*}
        \sum_{k=1}^{K'} \round{\valStartHat{\rHat_k}{k} - \valStart{r}{\pi_k}} \leq& \oTilde{ \sqrt{S^2AH^4 K'} + S^{3/2}AH^2 },\\
        \sum_{k=1}^{K'} \round{\valStartHat{\uHat_{k,i}}{k} - \valStart{u_i}{\pi_k}}\leq& \oTilde{ \sqrt{S^2AH^4 K'} + S^{3/2}AH^2 } \quad (\forall i\in[I]).
    \end{align*}
\end{lemma} 

\begin{proof}
    By \cref{lem:ext-val-diff} (with $\pi=\pi'=\pi_k$ and $M=(\calS,\calA,\pHat_k,\zHat_k)$, $M'=(\calS,\calA,p,z_k)$), we have according to the truncated policy evaluation (\cref{algo:trun-eval}),
    \begin{align*}
         &\valStartHat{\rHat_k}{k} - \valStart{r}{\pi_k} \\
         =& \sum_{h=1}^H \exptn\rectangular{ \qvalHat{\rHat_k,h}{k}{s_h,a_h} - r_h(s_h,a_h) - \dotProd{p_h(\cdot|s_h,a_h)}{\valHat{\rHat_k,h+1}{k}{\cdot}} \biggVert s_1, \pi_k, p}\\
         \leq& \sum_{h=1}^H \exptn\rectangular{ \rHat_{k,h}(s_h,a_h) - r_h(s_h,a_h) \biggVert s_1, \pi_k, p}\\
         &+\sum_{h=1}^H  \exptn\rectangular{\dotProd{\pHat_{k,h}(\cdot|s_h,a_h) - p_h(\cdot|s_h,a_h)}{\valHat{\rHat_k,h+1}{k}{\cdot}} \biggVert s_1, \pi_k, p}\\
         =& \sum_{h=1}^H \exptn\rectangular{ \rBar_{k-1,h}(s_h,a_h) + b_{k-1,h}^r(s_h,a_h) - r_h(s_h,a_h) \biggVert s_1, \pi_k, p}\\
         &+\sum_{h=1}^H  \exptn\rectangular{ b_{k-1,h}^p(s_h,a_h) + \dotProd{\pHat_{k,h}(\cdot|s_h,a_h) - p_h(\cdot|s_h,a_h)}{\valHat{\rHat_k,h+1}{k}{\cdot}} \biggVert s_1, \pi_k, p}.
    \end{align*}
    Since $G$ occurs, we have $\rBar_{k-1,h}(s_h,a_h) - r_h(s_h,a_h) \leq b_{k-1,h}^r(s_h,a_h)$. Moreover, 
    \begin{align*}
        &\dotProd{\pHat_{k,h}(\cdot|s_h,a_h) - p_h(\cdot|s_h,a_h)}{\valHat{\rHat_k,h+1}{k}{\cdot}} \\
        \leq& \norm{\pHat_{k,h}(\cdot|s_h,a_h) - p_h(\cdot|s_h,a_h)}_1 \norm{\valHat{\rHat_k,h+1}{k}{\cdot}}_{\infty}\\
        \leq& \norm{\pHat_{k,h}(\cdot|s_h,a_h) - p_h(\cdot|s_h,a_h)}_1 H\\
        \leq& b^p_{k-1,h}(s_h,a_h)
    \end{align*}
    by Hölder's, the truncation in the policy evaluation, and since $G$ occurs. Plugging this into the inequality above,
    \begin{align*}
         \valStartHat{\rHat_k}{k} - \valStart{r}{\pi_k}
         \leq& \sum_{h=1}^H \exptn\rectangular{ 2 b_{k-1,h}^r(s_h,a_h) \biggVert s_1, \pi_k, p}\\
         &+\sum_{h=1}^H  \exptn\rectangular{ 2 b_{k-1,h}^p(s_h,a_h) \biggVert s_1, \pi_k, p}.
    \end{align*}
    Recalling the definition of $b_{k-1,h}^r(s,a) \lesssim \frac{1}{\sqrt{\counterMax}}$, $b_{k-1,h}^p(s,a) \lesssim \frac{H\sqrt{S}}{\sqrt{\counterMax}}$ and summing up, we thus find
    \begin{align*}
        \sum_{k=1}^{K'} \round{\valStartHat{\rHat_k}{k} - \valStart{r}{\pi_k}} \lesssim& H\sqrt{S} \sum_{k=1}^{K'} \sum_{h=1}^H \exptn\rectangular{ \frac{1}{\sqrt{\counterMaxH}} \biggVert s_1, \pi_k, p}\\
        =& H\sqrt{S} \sum_{k=1}^{K'} \sum_{h=1}^H \exptn\rectangular{ \frac{1}{\sqrt{\counterMaxHK}} \biggVert \sigmaAlg{k-1}} \tag{play $\pi_k$}\\
        \lesssim& H\sqrt{S} \round{\sqrt{SAH^2 K'} + SAH} \tag{\cref{aug-lem:l36-efroni}}\\
        =& \sqrt{S^2AH^4 K'} + S^{3/2}AH^2
    \end{align*}
    where we used that we play $\pi_k$ in episode $k$ and \cref{aug-lem:l36-efroni}, which applies since $G$ occurs.

    The proof for $u_{i}$ ($i\in[I]$) is identical.
\end{proof}

\begin{lemma}[Estimation error $\psiHat$] \label{lem:est-error-psi}
    Conditioned on $G$, for every $K' \in [K]$, we have
    \begin{align*}
        \sum_{k=1}^{K'} \round{\valStartHat{\psiHat_k}{k} - \valStart{\psi_k}{\pi_k}} \leq& \oTilde{ \sqrt{S^2AH^4 K'} + S^{3/2}AH^2 }.
    \end{align*}
\end{lemma} 

\begin{proof}
    By \cref{lem:ext-val-diff} (with $\pi=\pi'=\pi_k$ and $M=(\calS,\calA,\pHat_k,\zHat_k)$, $M'=(\calS,\calA,p,z_k)$), we have according to the truncated policy evaluation,
    \begin{align*}
         &\valStartHat{\rHat_k}{k} - \valStart{r}{\pi_k}\\
         =& \sum_{h=1}^H \exptn\rectangular{ \qvalHat{\psiHat_k,h}{k}{s_h,a_h} - \psi_{k,h}(s_h,a_h) - \dotProd{p_h(\cdot|s_h,a_h)}{\valHat{\psiHat_k,h+1}{k}{\cdot}} \biggVert s_1, \pi_k, p}\\
         \leq& \sum_{h=1}^H \exptn\rectangular{ \psiHat_{k,h}(s_h,a_h) - \psi_{k,h}(s_h,a_h) \biggVert s_1, \pi_k, p}\\
         &+\sum_{h=1}^H  \exptn\rectangular{\dotProd{\pHat_{k,h}(\cdot|s_h,a_h)-p_h(\cdot|s_h,a_h)}{\valHat{\psiHat_k,h+1}{k}{\cdot}} \biggVert s_1, \pi_k, p}\\
         =& \sum_{h=1}^H \exptn\rectangular{ \psi_{k,h}(s_h,a_h) - \psi_{k,h}(s_h,a_h) \biggVert s_1, \pi_k, p}\\
         +&\sum_{h=1}^H  \exptn\rectangular{ b_{k-1,h}^p(s_h,a_h) \log(A) + \dotProd{\pHat_{k,h}(\cdot|s_h,a_h)-p_h(\cdot|s_h,a_h)}{\valHat{\psiHat_k,h+1}{k}{\cdot}} \biggVert s_1, \pi_k, p}.
    \end{align*}
    Since $G$ occurs, we have
    \begin{align*}
        &\dotProd{\pHat_{k,h}(\cdot|s_h,a_h) - p_h(\cdot|s_h,a_h)}{\valHat{\psiHat_k,h+1}{k}{\cdot}} \\
        \leq& \norm{\pHat_{k,h}(\cdot|s_h,a_h) - p_h(\cdot|s_h,a_h)}_1 \norm{\valHat{\psiHat_k,h+1}{k}{\cdot}}_{\infty}\\
        \leq& \norm{\pHat_{k,h}(\cdot|s_h,a_h) - p_h(\cdot|s_h,a_h)}_1 H\log(A)\\
        \leq& b^p_{k-1,h}(s_h,a_h) \log(A)
    \end{align*}
    by Hölder's, the truncation in the policy evaluation, and since $G$ occurs. Plugging this into the inequality above,
    \begin{align*}
         \valStartHat{\psiHat_k}{k} - \valStart{\psi_k}{\pi_k}
         \leq& \sum_{h=1}^H  \exptn\rectangular{ 2 b_{k-1,h}^p(s_h,a_h)\log(A) \biggVert s_1, \pi_k, p}
    \end{align*}
    and the rest of the proof follows exactly as in the proof of \cref{lem:est-error-ru}, with an extra $\log(A)$ factor.
\end{proof}

The following lemma allows us to control the total estimation error (a) from the optimistic model. Roughly speaking, it guarantees that the estimates are not too optimistic.
\begin{restatable}[Estimation error with regularization]{lemma}{estErrorReg}\label{lem:est-error-reg}
    Conditioned on $G$, for every $K' \in [K]$, we have (if $\tau \leq 1$)
    \begin{align*}
        \sum_{k=1}^{K'} \round{\valStartHat{\zHat_{k}}{k} - \valStart{z_k}{\pi_{k}}} \lesssim (2 + I \lambdaMax)\round{\sqrt{S^2AH^4 K'} + S^{3/2}AH^2}.
    \end{align*}
\end{restatable} 

\begin{proof}
    By \cref{eq:def-qk,eq:trun-v-lin,eq:v-lin}, conditioned on $G$,
    \begin{align*}
        &\sum_{k=1}^{K'} \round{\valStartHat{\zHat_{k}}{k} - \valStart{z_k}{\pi_{k}}} \\
        =& \sum_{k=1}^{K'} \round{\valStartHat{\rHat_k}{k} - \valStart{r}{\pi_k} }
        + \sum_{i=1}^I \lambda_k(i) \sum_{k=1}^{K'} \round{\valStartHat{\uHat_{k,i}}{k} - \valStart{u_i}{\pi_k} }
        + \tau  \sum_{k=1}^{K'} \round{\valStartHat{\psiHat_k}{k} - \valStart{\psi_k}{\pi_k} }\\
        \lesssim& (1 + I \lambdaMax + \tau)\round{\sqrt{S^2AH^4 K'} + S^{3/2}AH^2},
    \end{align*}
    and $\tau \leq 1$. The last inequality holds due to \cref{lem:est-error-ru,lem:est-error-psi}. 
\end{proof}

\paragraph{Per-State Optimism}

In the following, we show per-state optimism for the optimistic model.

\begin{lemma}[State optimism $\rHat$, $\bm{\uHat}$] \label{lem:optimism-val-ru}
    Conditioned on the success event $G$, for all $s,a,h$, and $k\in[K]$,
    \begin{align*}
        - \qvalHat{\rHat_k,h}{k}{s,a} + r_h(s,a) + \dotProd{p_h(\cdot|s,a)}{\valHat{\rHat_k,h+1}{k}{\cdot}} \leq& 0,\\
        - \qvalHat{\uHat_{k,i},h}{k}{s,a} + u_{k,i,h}(s,a) + \dotProd{p_h(\cdot|s,a)}{\valHat{\uHat_{k,i},h+1}{k}{\cdot}} \leq& 0 \quad (i\in[I])
    \end{align*}
\end{lemma}

\begin{proof}
    For $\rHat_k$, by \cref{algo:trun-eval} we have 
    \begin{align*}
        &\qvalHat{\rHat_k,h}{k}{s,a} \\
        =& \min\curly{H-h+1,~ \rHat_{k,h}(s,a) + \dotProd{\pHat_{k,h}(\cdot|s,a)} {\valHat{\rHat_k,h+1}{k}{\cdot}}}\\
        =& \min\big\{H-h+1,\\
        & \rBar_{k-1,h}(s,a) + b_{k-1,h}^r(s,a) + \dotProd{\pHat_{k,h}(\cdot|s,a)} {\valHat{\rHat_k,h+1}{k}{\cdot}} + b_{k-1,h}^p(s,a)\big\}\\
        \geq& \min\curly{1,~ \rBar_{k-1,h}(s,a) + b_{k-1,h}^r(s,a)} \\
        &+ \min\curly{H-h,\dotProd{\pHat_{k,h}(\cdot|s,a)} {\valHat{\rHat_k,h+1}{k}{\cdot}} + b_{k-1,h}^p(s,a)},
    \end{align*}
    where we used $\min\curly{a+b,c+d} \geq \min\curly{a,c} + \min\curly{b,d}$.
    Thus
    \begin{align*}
        &- \qvalHat{\rHat_k,h}{k}{s,a} + r_h(s,a) + \dotProd{p_h(\cdot|s,a)}{\valHat{\rHat_k,h+1}{k}{\cdot}}\\  
        \leq& -\min\curly{1,~ \rBar_{k-1,h}(s,a) + b_{k-1,h}^r(s,a)} + r_h(s,a) \\
        & - \min\curly{H-h,\dotProd{\pHat_{k,h}(\cdot|s,a)} {\valHat{\rHat_k,h+1}{k}{\cdot}} + b_{k-1,h}^p(s,a)} + \dotProd{p_h(\cdot|s,a)}{\valHat{\rHat_k,h+1}{k}{\cdot}}\\
        =&-\min\curly{1 - r_h(s,a),~ \rBar_{k-1,h}(s,a) - r_h(s,a) + b_{k-1,h}^r(s,a)} \\
        & - \min\bigg\{H-h -\dotProd{p_h(\cdot|s,a)}{\valHat{\rHat_k,h+1}{k}{\cdot}},\\
        &\dotProd{\pHat_{k,h}(\cdot|s,a) - p_h(\cdot|s,a)}{\valHat{\rHat_k,h+1}{k}{\cdot}} + b_{k-1,h}^p(s,a)\bigg\}\\
        =& \max\bigg\{\underbrace{r_h(s,a) - 1}_{(a)},~ \underbrace{-\rBar_{k-1,h}(s,a) + r_h(s,a) - b_{k-1,h}^r(s,a)}_{(b)}\bigg\} \\
        & +\max\bigg\{\underbrace{-(H-h) +\dotProd{p_h(\cdot|s,a)}{\valHat{\rHat_k,h+1}{k}{\cdot}}}_{(c)}, \\
        &\underbrace{-\dotProd{\pHat_{k,h}(\cdot|s,a) - p_h(\cdot|s,a)}{\valHat{\rHat_k,h+1}{k}{\cdot}} - b_{k-1,h}^p(s,a)}_{(d)}\bigg\}.
    \end{align*}
    Now, for each of the four terms appearing in the maxima, conditioned on $G$, we have 
    \begin{align*}
        (a)=r_h(s,a) - 1 \leq& 0
    \end{align*}
    by boundedness of the reward functions,
    \begin{align*}
        (b)= -\rBar_{k,h}(s,a) + r_h(s,a) - b_{k-1,h}^r(s,a) \leq& 0
    \end{align*}
    since $G$ occurs,
    \begin{align*}
        (c)=-(H-h) +\dotProd{p_h(\cdot|s,a)}{\valHat{\rHat_k,h+1}{k}{\cdot}} \leq -(H-h) + 1 \cdot (H-(h+1)+1) = 0
    \end{align*}
    by Hölder's and the truncation of our evaluation procedure, and finally
    \begin{align*}
        (d)=&-\dotProd{\pHat_{k,h}(\cdot|s,a) - p_h(\cdot|s,a)}{\valHat{\rHat_k,h+1}{k}{\cdot}} - b_{k-1,h}^p(s,a) \\
        \leq& \norm{\pHat_{k,h}(\cdot|s,a) - p_h(\cdot|s,a)}_1 \cdot (H - (h+1) +1 ) - b_{k-1,h}^p(s,a)\\
        \leq& \norm{\pHat_{k,h}(\cdot|s,a) - p_h(\cdot|s,a)}_1 \cdot H - b_{k-1,h}^p(s,a) \leq 0
    \end{align*}
    by Hölder's, the truncation of our evaluation procedure, since $\pHat_k=\pBar_{k-1}$, and since $G$ occurs. We are thus taking the minimum of non-positive terms, which shows
    \begin{align*}
        - \qvalHat{\rHat_k,h}{k}{s,a} + r_h(s,a) + \dotProd{p_h(\cdot|s,a)}{\valHat{\rHat_k,h+1}{k}{\cdot}} \leq 0.
    \end{align*}
    The proof for $\uHat_{k,i}$ $(i\in [I])$ is identical.
\end{proof}

\begin{lemma}[State optimism $\psiHat$] \label{lem:optimism-val-psi}
    Conditioned on the success event $G$, for all $s,a,h$, and $k\in[K]$,
    \begin{align*}
        - \qvalHat{\psiHat_k,h}{k}{s,a} + \psi_{k,h}(s,a) + \dotProd{p_h(\cdot|s,a)}{\valHat{\psiHat_k,h+1}{k}{\cdot}} \leq 0.
    \end{align*}
\end{lemma}

\begin{proof} 
    For $\psiHat_{k,h}(s,a) = \psi_{k,h}(s,a) + b_{k-1,h}^p(s,a)\log(A)$, from \cref{algo:trun-eval} we have 
    \begin{align*}
        &\qvalHat{\psiHat_k,h}{k}{s,a} \\
        =& \min\curly{\psi_{k,h}(s,a) + (H-h+1)\log(A),~ \psiHat_{k,h}(s,a) + \dotProd{\pHat_{k,h}(\cdot|s,a)} {\valHat{\psiHat_k,h+1}{k}{\cdot}}} \\
        =& \psi_{k,h}(s,a) \\
        &+\min\curly{(H-h+1)\log(A),~ b_{k-1,h}^p(s,a)\log(A) + \dotProd{\pHat_{k,h}(\cdot|s,a)} {\valHat{\psi_k,h+1}{k}{\cdot}}}.
    \end{align*}
    Thus
    \begin{align*} 
        &- \qvalHat{\psiHat_k,h}{k}{s,a} + \psi_{k,h}(s,a) + \dotProd{p_h(\cdot|s,a)}{\valHat{\psiHat_k,h+1}{k}{\cdot}}\\   
        =& -\min\curly{(H-h+1)\log(A),~ b_{k-1,h}^p(s,a)\log(A) + \dotProd{\pHat_{k,h}(\cdot|s,a)} {\valHat{\psiHat_k,h+1}{k}{\cdot}}}\\
        &+ \dotProd{p_h(\cdot|s,a)}{\valHat{\psiHat_k,h+1}{k}{\cdot}}\\
        =& -\min\bigg\{(H-h+1)\log(A) - \dotProd{p_h(\cdot|s,a)}{\valHat{\psiHat_k,h+1}{k}{\cdot}},\\
        & b_{k-1,h}^p(s,a)\log(A) + \dotProd{\pHat_{k,h}(\cdot|s,a) - p_h(\cdot|s,a)}{\valHat{\psiHat_k,h+1}{k}{\cdot}}\bigg\}\\
        =&+\max\bigg\{\underbrace{-(H-h+1)\log(A) + \dotProd{p_h(\cdot|s,a)}{\valHat{\psiHat_k,h+1}{k}{\cdot}}}_{=:(a)},\\
        &\underbrace{-b_{k-1,h}^p(s,a)\log(A) -\dotProd{\pHat_{k,h}(\cdot|s,a) - p_h(\cdot|s,a)}{\valHat{\psiHat_k,h+1}{k}{\cdot}}}_{=:(b)}\bigg\}.
    \end{align*}
    For (a), first note that by the truncation in the policy evaluation (\cref{algo:trun-eval})
    \begin{align}
        \valHat{\psiHat_k,h+1}{k}{s} =& \dotProd{\pi_{k,h+1}(\cdot|s)}{\qvalHat{\psiHat_k,h+1}{k}{s,\cdot}} \tag{by definition}\nonumber\\
        \leq& \sum_a \pi_{k,h+1}(a|s) \cdot \psi_{k,h+1}(s,a) \nonumber\\
        &+ \sum_a \pi_{k,h+1}(a|s) \cdot (H-(h+1)+1)\log(A) \tag{truncated update}\nonumber\\
        \leq& \log(A) + \sum_a \pi_{k,h+1}(a|s) \cdot (H-h)\log(A)\nonumber\\
        \leq& (H-h+1)\log(A). \label{eq:trunc-psi}
    \end{align}
    Hence
    \begin{align*}
        (a) =&-(H-h+1)\log(A) + \dotProd{p_h(\cdot|s,a)}{\valHat{\psiHat_k,h+1}{k}{\cdot}}\\
        \leq& -(H-h+1)\log(A) + (H-h+1)\log(A)\\
        =& 0
    \end{align*}
    by Hölder's and \cref{eq:trunc-psi}. 
    For (b),
    \begin{align*}
        (b)=&-b_{k-1,h}^p(s,a)\log(A) -\dotProd{\pHat_{k,h}(\cdot|s,a) - p_h(\cdot|s,a)}{\valHat{\psiHat_k,h+1}{k}{\cdot}}\\
        \leq& -b_{k-1,h}^p(s,a)\log(A) + \norm{\pHat_{k,h}(\cdot|s,a) - p_h(\cdot|s,a)}_1 \cdot (H-h+1)\log(A)\\
        \leq& -b_{k-1,h}^p(s,a)\log(A) + \norm{\pHat_{k,h}(\cdot|s,a) - p_h(\cdot|s,a)}_1 \cdot H\log(A)\\
        \leq& 0
    \end{align*}
    by Hölder's, and since $\pHat_k=\pBar_{k-1}$ and $G$ occurs.
\end{proof}

Finally, the following lemma shows that we can discard term (c) in \cref{lem:regret-decomp}. It guarantees that for every state, an optimistic \textit{Bellman-type} inequality holds.
\begin{restatable}[State optimism $\zHat$]{lemma}{optimismValZ} \label{lem:optimism-val-z}
    Conditioned on the success event $G$,
    \begin{align*}
        - \qvalHat{\zHat_k,h}{k}{s,a} + z_{k,h}(s,a) + \dotProd{p_h(\cdot|s,a)}{\valHat{\zHat_k,h+1}{k}{\cdot}} \leq 0.
    \end{align*}
\end{restatable}

\begin{proof}
    By \cref{eq:def-qk,eq:trun-v-lin}, conditioned on $G$,
    \begin{align*}
        &- \qvalHat{\zHat_k,h}{k}{s,a} + z_{k,h}(s,a) + \dotProd{p_h(\cdot|s,a)}{\valHat{\zHat_k,h+1}{k}{\cdot}}\\
        =& - \qvalHat{\rHat_k,h}{k}{s,a} + r_{k,h}(s,a) + \dotProd{p_h(\cdot|s,a)}{\valHat{\rHat_k,h+1}{k}{\cdot}}\\
        & \sum_{i} \lambda_{k,i} \round{- \qvalHat{\uHat_{k,i},h}{k}{s,a} + u_{k,i,h}(s,a) + \dotProd{p_h(\cdot|s,a)}{\valHat{\uHat_{k,i},h+1}{k}{\cdot}}}\\
        & \tau \round{- \qvalHat{\psiHat_k,h}{k}{s,a} + \psi_{k,h}(s,a) + \dotProd{p_h(\cdot|s,a)}{\valHat{\psiHat_k,h+1}{k}{\cdot}}}\\
        \leq& 0,
    \end{align*}
    where the last inequality holds due to \cref{lem:optimism-val-ru,lem:optimism-val-psi} and since all $\lambda_{k,i} \geq 0$.
\end{proof}

\paragraph{Value Function Bounds}

In the following, we bound the regularized value functions, which allows us to apply descent properties of the regularized primal-dual algorithm. Recall that $\bm{g} = \bm{u} - \frac{1}{H}\bm{c}$.

\begin{lemma}[Value function bounds]\label{lem:val-bounds-finite}
    For any $s,a,h$, we have
    \begin{align}
        0 \leq \qval{r + \bm{\lambda}_k^T \bm{u} + \tau \psi_k, h}{\pi_k}{s,a} \leq & - \tau \log(\pi_{k,h}(a|s)) + H(1 + I \lambdaMax + \tau \log(A)) \label{eq:q-bound-finite}
    \end{align}
    Moreover, for any $s,a,h$,
    \begin{align*}
        &\sum_{a} \pi_{k,h}(a|s) \exp\round{\eta \qval{r + \bm{\lambda}_k^T \bm{u} + \tau \psi_k, h}{\pi_k}{s,a}} \qval{r + \bm{\lambda}_k^T \bm{u} + \tau \psi_k, h}{\pi_k}{s,a}^2\\ 
        \leq& \sqrt{A}\exp\round{\eta H \round{1 + \lambdaMax I + \tau\log(A)}} \round{ 2H^2 \round{1 + I \lambdaMax + \tau \log(A)}^2 + 2\tau^2 (64/e^2) }
    \end{align*}
    and
    \begin{align*}
         \norm{\valStartVec{\bm{g}}{\pi_k} + \tau \bm{\lambda}_k } \leq& \sqrt{I} \round{ H + \tau \lambdaMax } 
    \end{align*}
\end{lemma}

\begin{proof}
    We prove the first inequality. Non-negativity is immediate. Moreover, for all $s$, $a$,
    \begin{align*}
        \qval{r + \bm{\lambda}_k^T \bm{u} + \tau \psi_k, h}{\pi_k}{s,a} =& \abs{\qval{r, h}{\pi_k}{s,a} + \sum_{i} \lambda_{k,i} \qval{u_i, h}{\pi_k}{s,a} + \tau \qval{\psi_k,h}{\pi_k}{s,a}}\\
        \leq& \abs{ \qval{r,h}{\pi_k}{s,a}} + \sum_{i} \lambda_{k,i} \abs{\qval{u_i,h}{\pi_k}{s,a}} + \tau \abs{\qval{\psi_k,h}{\pi_k}{s,a}}\\
        \leq& H + I \lambdaMax H + \tau \exptn_{\pi_k}\rectangular{ \sum_{h'=h}^{H} -\log(\pi_{k,h'}(a_{h'}|s_{h'}))\mid s_h=s,a_h=a}.
    \end{align*}
    We finish by bounding 
    \begin{align*}
        &\exptn_{\pi_k}\rectangular{ \sum_{h'=h}^{H} -\log(\pi_{k,h'}(a_{h'}|s_{h'})) \biggVert s_h=s,a_h=a}\\
        =& -\log(\pi_{k,h}(a|s)) + \exptn_{\pi_k}\rectangular{ \sum_{h'=h+1}^{H} -\log(\pi_{k,h'}(a_{h'}|s_{h'})) \biggVert a_h=a, s_h=s}\\
        =& -\log(\pi_{k,h}(a|s)) + \exptn_{\pi_k}\rectangular{ \sum_{h'=h+1}^{H} -\log(\pi_{k,h'}(a_{h'}|s_{h'})) \biggVert s_{h+1} \sim p_h(\cdot | s, a)}\\
        =& -\log(\pi_{k,h}(a|s)) + \sum_{h'=h+1} \sum_{s'} \occ{s_{h+1} \sim p_h(\cdot | s, a), h'}{\pi_k}{s'} \sum_{a'} - \pi_{k,h'}(a'|s') \log(\pi_{k,h'}(a'|s'))\\
        \leq& -\log(\pi_{k,h}(a|s)) + \sum_{h'=h+1} \sum_{s'} \occ{s_{h+1} \sim p_h(\cdot | s, a), h'}{\pi_k}{s'} \log(A)\\
        =& -\log(\pi_{k,h}(a|s)) + \sum_{h'=h+1} \log(A)\\
        \leq& -\log(\pi_{k,h}(a|s)) + H\log(A),
    \end{align*}
    where we used the standard bound on the entropy in the third to last step, and we are considering unnormalized occupancy measures throughout.

    For the second inequality, first note that (using $(a+b)^2 \leq 2a^2 + 2b^2$)
    \begin{align*}
        \qval{r + \bm{\lambda}_k^T \bm{u} + \tau \psi_k, h}{\pi_k}{s,a}^2 \leq& \underbrace{2H^2 \round{1 + I \lambdaMax + \tau \log(A)}^2}_{=: C_1} + 2\tau^2\log^2\round{\frac{1}{\pi_{k,h}(a|s)}}.
    \end{align*} 
    Moreover, using \cref{eq:q-bound-finite} we have 
    \begin{align*}
        &\pi_{k,h}(a|s) \exp\round{\eta \qval{r + \bm{\lambda}_k^T \bm{u} + \tau \psi_k,h}{\pi_k}{s,a}}\\
        \leq& \pi_{k,h}(a|s)^{1-\eta\tau} \underbrace{\exp\round{\eta H \round{1 + \lambdaMax I + \tau\log(A)}}}_{=:C_2}
    \end{align*}
    We thus find, if $\eta\tau \leq 1/4 < 1/2$ (which is easily satisfied since we choose $\eta$, $\tau$ small),
    \begin{align*}
        &\sum_a \pi_{k,h}(a|s) \exp\round{\eta \qval{r + \bm{\lambda}_k^T \bm{u} + \tau \psi_k, h}{\pi_k}{s,a}} \cdot C_1\\
        \leq& \sum_a \pi_{k,h}(a|s)^{1-\eta\tau} \cdot C_2 \cdot C_1\\
        \leq& \sum_a \pi_{k,h}(a|s)^{1/2} \cdot C_2 \cdot C_1\\
        \leq& \round{\sum_a \pi_{k,h}(a|s)}^{1/2} \round{A C_2^2 \cdot C_1^2}^{1/2}\\
        =& \sqrt{A} C_1 C_2,
    \end{align*}
    where the last inequality is Cauchy-Schwarz, and 
    \begin{align*}
        &\sum_a \pi_{k,h}(a|s) \exp\round{\eta \qval{r + \bm{\lambda}_k^T \bm{u} + \tau \psi_k,h}{\pi_k}{s,a}} \cdot 2\tau^2\log^2\round{\frac{1}{\pi_{k,h}(a|s)}}\\
        \leq& \sum_a \pi_{k,h}(a|s)^{1-\eta\tau} \cdot C_2 \cdot 2\tau^2\log^2\round{\frac{1}{\pi_{k,h}(a|s)}}\\
        \leq& \sum_a \pi_{k,h}(a|s)^{1/2} \cdot C_2 \cdot 2\tau^2\pi_k(a|s,h)^{1/4}\log^2\round{\frac{1}{\pi_{k,h}(a|s)}}\\
        \leq& \sum_a \pi_{k,h}(a|s)^{1/2} \cdot C_2 \cdot 2\tau^2 (64/e^2)\\
        \leq& \sqrt{A} C_2 \cdot 2\tau^2 (64/e^2),
    \end{align*}
    where we used the fact that $q^{1/4}\log^2(1/q) \leq 64/e^2$ for $q\in(0,1)$, and Cauchy-Schwarz in the same manner as before. Adding up the previous two terms and plugging in the definitions of the constants yields the second inequality.
    
    For the third inequality, we find (recalling that $g_h(s,a) \in [-1,1]^I$)
    \begin{align*}
        \norm{\valStartVec{\bm{g}}{\pi_k} + \tau \bm{\lambda}_k} \leq& \norm{\valStartVec{\bm{g}}{\pi_k}} + \tau \norm{\bm{\lambda}_k}\\
        \leq& \sqrt{I} H + \tau \sqrt{I} \lambdaMax,
    \end{align*}
    concluding the proof.
\end{proof}

\begin{lemma}[Truncated value function bounds] \label{lem:val-bounds-finite-surrogate}
    For all $s,a,h$, we have 
    \begin{align}
        0 \leq \qvalHat{\zHat_k,h}{k}{s,a} \leq - \tau \log(\pi_{k,h}(a|s)) + H(1 + I \lambdaMax + \tau \log(A)) \label{eq:q-bound-finite-surrogate}
    \end{align}
    and 
    \begin{align*}
        &\sum_{a} \pi_{k,h}(a|s) \exp\round{ \eta \qvalHat{\zHat_k,h}{k}{s,a} } \qvalHat{\zHat_k,h}{k}{s,a}^2\\
        \leq& \sqrt{A}\exp\round{\eta H \round{1 + \lambdaMax I + \tau\log(A)}} \round{ 2H^2 \round{1 + I \lambdaMax + \tau \log(A)}^2 + 2\tau^2 (64/e^2) }
    \end{align*}
    and 
    \begin{align*}
        \norm{\valStartHatVec{\bm{\uHat}_{k}}{k} - \bm{c} + \tau \bm{\lambda}_k} \leq \sqrt{I} \round{ H + \tau \lambdaMax } 
    \end{align*}
\end{lemma}

\begin{proof}
    Since the truncated value functions are bounded between $0$ and the true value functions, the statement follows from \cref{lem:val-bounds-finite}.
\end{proof}

\section{Regret Analysis} \label{app:pd}

In this section, we provide all proofs of the result from \cref{sec:pd}, leading to a regret bound of the regularized primal-dual algorithm (\cref{algo:rpg-dual-finite-learn}).

We write $\lesssim$ for asymptotic inequality up to polylogarithmic terms. First, we note that the primal-dual updates indeed correspond to mirror descent updates.

\begin{observation} \label{lem:update-exp-kl-finite-learn}
    The closed-form expressions in \cref{algo:rpg-pd-finite-learn} are solutions to
    \begin{align*}
        \max_{\pi_h(\cdot | s) \in \simplex{\calA}}& \round{ \sum_{a \in \calA} \pi_h (a | s) \qvalHat{\zHat_k,h}{k}{s,a} - \frac{1}{\eta} \text{KL}(\pi_h(\cdot|s), \pi_{k,h}(\cdot|s))},\\
        \min_{\bm{\lambda} \in \Lambda}& \round{ \bm{\lambda}^T \round{ \valStartHatVec{\bm{\uHat}_k}{k} - \bm{c} + \tau \bm{\lambda}_k} + \frac{1}{2\eta} \norm{\bm{\lambda} - \bm{\lambda}_k}^2 },
    \end{align*}
    respectively.
\end{observation}
This means that both the primal and dual variables are updated via mirror descent with different regularizers. Hence, we can make use of the descent lemmas for online mirror descent (we refer to \cref{lem:md-descent-line,lem:md-descent-proj} in \cref{sec:omd}). However, the value functions that serve as (surrogate) gradients are only estimates. Thus, the convergence proof for the potential function $\Phi_k$ that measures the distance from $(\pi_k, \bm{\lambda}_k)$ to $(\regPiStar, \regLambdaStar)$ needs to take the estimation errors into account.

\potentialFiniteLearn*

\begin{proof}
    Condition on $G$, which occurs with probability at least $1-\delta$ by \cref{lem:succ-event}. We first decompose the $k$-th \emph{primal-dual gap} as follows:
    \begin{align}
        \regLagrangian(\regPiStar, \bm{\lambda}_k) - \regLagrangian(\pi_k, \regLambdaStar) = \underbrace{\regLagrangian(\regPiStar, \bm{\lambda}_k) - \regLagrangian(\pi_k, \bm{\lambda}_k)}_{(i)} + \underbrace{\regLagrangian(\pi_k, \bm{\lambda}_k) - \regLagrangian(\pi_k, \regLambdaStar)}_{(ii)}. \label{eq:pd-decomp-finite-learn}
    \end{align}
    We first bound term (i):
    \begin{align}
        (i) =& \regLagrangian(\regPiStar, \bm{\lambda}_k) - \regLagrangian(\pi_k, \bm{\lambda}_k)\nonumber\\
        =& \valStart{r + \bm{\lambda}_k^T\bm{g}}{\regPiStar} - \valStart{r + \bm{\lambda}_k^T\bm{g}}{\pi_k} \tag{as $\occ{h}{\pi}{s,a}=\occ{h}{\pi}{s}\pi_h(a|s)$, and cancel $\norm{\bm{\lambda}_k}^2$}\nonumber\\
        &- \tau \sum_{s,a,h} \occ{h}{\regPiStar}{s} \regPiStarH{h}(a|s) \log(\regPiStarH{h}(a|s)) + \tau \sum_{s,a,h} \occ{h}{\pi_k}{s} \pi_{k,h}(a|s) \log(\pi_{k,h}(a|s))\nonumber\\
        =& \valStart{r + \bm{\lambda}_k^T\bm{g} + \tau \psi_k}{\regPiStar} - \valStart{r + \bm{\lambda}_k^T\bm{g} + \tau \psi_k}{\pi_k} \quad \tag{since $\psi_{k,h}(s,a) = -\log(\pi_{k,h}(a|s))$} \nonumber\\
        &+ \tau \sum_{s,a,h} \occ{h}{\regPiStar}{s} \regPiStarH{h}(a|s) \log(\pi_{k,h}(a|s)) - \tau \sum_{s,a,h} \occ{h}{\pi_k}{s} \pi_{k,h}(a|s) \log(\pi_{k,h}(a|s))\nonumber\\
        &- \tau \sum_{s,a,h} \occ{h}{\regPiStar}{s} \regPiStarH{h}(a|s) \log(\regPiStarH{h}(a|s)) + \tau \sum_{s,a,h} \occ{h}{\pi_k}{s} \pi_{k,h}(a|s) \log(\pi_{k,h}(a|s))\nonumber\\
        =& \valStart{r + \bm{\lambda}_k^T\bm{g} + \tau \psi_k}{\regPiStar} - \valStart{r + \bm{\lambda}_k^T\bm{g} + \tau \psi_k}{\pi_k} \nonumber\\
        &+ \tau \sum_{s,a,h} \occ{h}{\regPiStar}{s} \regPiStarH{h}(a|s) \log(\pi_{k,h}(a|s)) \nonumber\\
        &- \tau \sum_{s,a,h} \occ{h}{\regPiStar}{s} \regPiStarH{h}(a|s) \log(\regPiStarH{h}(a|s)) \nonumber\\
        =& \valStart{r + \bm{\lambda}_k^T\bm{g} + \tau \psi_k}{\regPiStar} - \valStart{r + \bm{\lambda}_k^T\bm{g} + \tau \psi_k}{\pi_k} \nonumber\\
        &- \tau \sum_{s,h} \occ{h}{\regPiStar}{s} \sum_{a} \regPiStarH{h}(a|s) \log\round{\frac{\regPiStarH{h}(a|s)}{\pi_{k,h}(a|s)}}\nonumber\\
        =& \valStart{r + \bm{\lambda}_k^T\bm{g} + \tau \psi_k}{\regPiStar} - \valStart{r + \bm{\lambda}_k^T\bm{g} + \tau \psi_k}{\pi_k} - \tau\sum_{s,h} \occ{h}{\regPiStar}{s} \kl_{k,h}(s) \nonumber\\
        =& \valStart{r + \bm{\lambda}_k^T\bm{g} + \tau \psi_k}{\regPiStar} - \valStart{r + \bm{\lambda}_k^T\bm{g} + \tau \psi_k}{\pi_k} - \tau \kl_k \nonumber\\
        =& \valStart{r + \bm{\lambda}_k^T\bm{u} + \tau \psi_k}{\regPiStar} - \valStart{r + \bm{\lambda}_k^T\bm{u} + \tau \psi_k}{\pi_k} - \tau \kl_k \tag{as $\bm{g} = \bm{u} - \frac{1}{H}\bm{c}$}\nonumber\\
        =&\valStart{z_k}{\regPiStar} - \valStart{z_k}{\pi_k} - \tau \kl_k \nonumber
    \end{align}

    Now by \cref{lem:regret-decomp}, we have 
    \begin{align*}
        &\valStart{z_k}{\regPiStar} - \valStart{z_k}{\pi_k} \\
        =& (\valStartHat{\zHat_k}{k} - \valStart{z_k}{\pi_k}) \tag{a}\\
        & + \sum_h \exptn \rectangular{ \left\langle \qvalHat{\zHat_k,h}{k}{s_h,\cdot},  \regPiStarH{h}(\cdot|s_h) - \pi_{k,h}(\cdot|s_h) \right\rangle \biggVert s_1, \regPiStar, p}\tag{b}\\
        & + \sum_h \exptn \rectangular{ - \qvalHat{\zHat_k,h}{k}{s_h,a_h} + z_{k,h}(s_h,a_h) + \left\langle p_h(\cdot | s_h,a_h), \valHat{\zHat_k,h+1}{\pi_k}{\cdot} \right\rangle \biggVert s_1, \regPiStar, p} \tag{c}
    \end{align*}
    We leave term (a) as is and will sum over $k$ later. For term (b), note that for all $s,h$,
    \begin{align*}
        &\left\langle \qvalHat{\zHat_k,h}{k}{s,\cdot},  \regPiStarH{h}(\cdot|s) - \pi_{k,h}(\cdot|s) \right\rangle\\
        \leq& \frac{\kl_{k,h}(s) - \kl_{k+1,h}(s)}{\eta} + \frac{\eta}{2}\sum_{a} \pi_{k,h}(a|s)\exp\round{\qvalHat{\zHat_k,h}{k}{s,a}} \qvalHat{\zHat_k,h}{k}{s,a}^2 \tag{\cref{lem:md-descent-kl}}\\
        \leq&\frac{\kl_{k,h}(s) - \kl_{k+1,h}(s)}{\eta} + \frac{\eta}{2} \frac{1}{H} \dTauLambda, \tag{\cref{lem:val-bounds-finite-surrogate}}
    \end{align*}
    with 
    \begin{align*}
        \dTauLambda = HA^{1/2}\exp\round{\eta H \round{1 + \lambdaMax I + \tau\log(A)}} \round{ 2H^2 \round{1 + I \lambdaMax + \tau \log(A)}^2 + 2\tau^2 (64/e^2) }
    \end{align*}
    and where we were able to apply \cref{lem:md-descent-kl} by \cref{lem:update-exp-kl-finite-learn} and since $\qvalHat{\zHat_k,h}{k}{s,a}\geq0$. Hence, 
    \begin{align*}
        \sum_h \exptn \rectangular{ \left\langle \qvalHat{\zHat_k,h}{k}{s_h,\cdot},  \regPiStarH{h}(\cdot|s_h) - \pi_{k,h}(\cdot|s_h) \right\rangle \biggVert s_1, \regPiStar, p}
        =& \sum_{s,h} \occ{h}{\regPiStar}{s} \left\langle \qvalHat{\zHat_k,h}{k}{s,\cdot},  \regPiStarH{h}(\cdot|s) - \pi_{k,h}(\cdot|s) \right\rangle \\
        \leq& \sum_{s,h} \occ{h}{\regPiStar}{s} \round{\frac{\kl_{k,h}(s) - \kl_{k+1,h}(s)}{\eta} + \frac{\eta}{2} \frac{1}{H}\dTauLambda}\\
        =& \frac{\kl_k - \kl_{k+1}}{\eta} + \frac{\eta}{2} \dTauLambda.
    \end{align*}
    Term (c) is $\leq 0$ by \cref{lem:optimism-val-z}, which applies since $G$ occurs. Plugging in, we thus find 
    \begin{align}
        (i) =& \valStart{z_k}{\regPiStar} - \valStart{z_k}{\pi_k} - \tau \kl_k \nonumber\\
        \leq& (\valStartHat{\zHat_k}{k} - \valStart{z_k}{\pi_k}) 
         + \frac{\kl_k - \kl_{k+1}}{\eta} + \frac{\eta}{2} \dTauLambda 
         + 0
         - \tau \kl_k \nonumber\\
         =& (\valStartHat{\zHat_k}{k} - \valStart{z_k}{\pi_k})
         + \frac{(1-\eta\tau)\kl_k - \kl_{k+1}}{\eta} + \frac{\eta}{2} \dTauLambda. \label{eq:rpg-i-bound-finite-learn}
    \end{align}
    We now bound term (ii):
    \begin{align}
        (ii) =& \regLagrangian(\pi_k, \bm{\lambda}_k) - \regLagrangian(\pi_k, \regLambdaStar)\nonumber\\
        =& \valStart{r+\bm{\lambda}_k^T \bm{g}}{\pi_k} - \valStart{r+(\regLambdaStar)^T \bm{g}}{\pi_k} + \frac{\tau}{2} \norm{\bm{\lambda}_k}^2 - \frac{\tau}{2} \norm{\regLambdaStar}^2 \tag{cancel $\entropy(\pi_k)$}\nonumber\\
        =& \sum_{i} (\lambda_{k,i} - \regLambdaStarI{i}) \valStart{g_i}{\pi_k} + \frac{\tau}{2} \norm{\bm{\lambda}_k}^2 - \frac{\tau}{2} \norm{\regLambdaStar}^2\nonumber\\
        =& \sum_{i} (\lambda_{k,i} - \regLambdaStarI{i}) (\valStart{u_i}{\pi_k} - c_i + \tau \lambda_{k,i}) - \frac{\tau}{2} \norm{\bm{\lambda}_k - \regLambdaStar}^2 \nonumber\\
        =& \sum_{i} (\lambda_{k,i} - \regLambdaStarI{i}) (\valStartHat{\uHat_{k,i}}{k} - c_i + \tau \lambda_{k,i}) \tag{a}\nonumber\\
        &\sum_{i} (\lambda_{k,i} - \regLambdaStarI{i}) (\valStart{u_i}{\pi_k} - \valStartHat{\uHat_{k,i}}{k}) \tag{b}\nonumber\\
        &- \frac{\tau}{2} \norm{\bm{\lambda}_k - \regLambdaStar}^2 \nonumber
    \end{align}
    To bound (a), we note that
    \begin{align*}
        &\sum_{i} (\lambda_{k,i} - \regLambdaStarI{i}) (\valStartHat{\uHat_{k,i}}{k} - c_i + \tau \lambda_{k,i})\\
        \leq& \frac{\norm{\regLambdaStar - \bm{\lambda}_k}^2 - \norm{\regLambdaStar - \bm{\lambda}_{k+1}}^2}{2\eta} +\frac{\eta}{2}\norm{\valStartHatVec{\uHat_{k}}{k} - \bm{c} + \tau \bm{\lambda}_k}^2 \tag{\cref{lem:md-descent-proj}}\\
        \leq& \frac{\norm{\regLambdaStar - \bm{\lambda}_k}^2 - \norm{\regLambdaStar - \bm{\lambda}_{k+1}}^2}{2\eta} + \frac{\eta}{2} \dTauLambdaPrime, \tag{\cref{lem:val-bounds-finite-surrogate}}
    \end{align*}
    with $\dTauLambdaPrime = I(H + \tau \lambdaMax)^2$ and where we were able to apply \cref{lem:md-descent-proj} by \cref{lem:update-exp-kl-finite-learn}. We can bound (b) via 
    \begin{align*}
        \sum_{i} (\lambda_{k,i} - \regLambdaStarI{i}) ( \valStart{u_i}{\pi_k} - \valStartHat{\uHat_{k,i}}{k} ) \leq& \sum_{i} \lambdaMax \abs{ \valStartHat{\uHat_{k,i}}{k} - \valStart{u_i}{\pi_k} }.
    \end{align*}
    Plugging in, we find 
    \begin{align}
        (ii) =& \sum_{i} (\lambda_{k,i} - \regLambdaStarI{i} ) (\valStart{u_i}{\pi_k} - c_i + \tau \lambda_{k,i}) - \frac{\tau}{2} \norm{\bm{\lambda}_k - \regLambdaStar}^2 \nonumber\\
        \leq& \frac{\norm{\regLambdaStar - \bm{\lambda}_k}^2 - \norm{\regLambdaStar - \bm{\lambda}_{k+1}}^2}{2\eta} + \frac{\eta}{2} \dTauLambdaPrime
        + \sum_i \lambdaMax \abs{ \valStartHat{\uHat_{k,i}}{k} - \valStart{u_i}{\pi_k} }
        - \frac{\tau}{2} \norm{\bm{\lambda}_k - \regLambdaStar}^2\nonumber\\
        =& \frac{(1-\eta\tau)\norm{\regLambdaStar - \bm{\lambda}_k}^2 - \norm{\regLambdaStar - \bm{\lambda}_{k+1}}^2}{2\eta} + \frac{\eta}{2} \dTauLambdaPrime
        + \sum_i \lambdaMax \abs{ \valStartHat{\uHat_{k,i}}{k} - \valStart{u_i}{\pi_k} }. \label{eq:rpg-ii-bound-finite-learn}
    \end{align}
    Before proceeding, note that conditioned on $G$, we have 
    \begin{align*}
        \valStartHat{\uHat_{k,i}}{k} - \valStart{u_i}{\pi_k} \geq& 0,\\
        \valStartHat{\zHat_k}{k} - \valStart{z_k}{\pi_k} \geq& 0,
    \end{align*}
    by \cref{lem:optimism-val-ru,lem:optimism-val-z}, respectively. Hence, we can treat these differences and their absolute values interchangeably in what follows.
    
    From \cref{lem:sp-cmdp-finite} (with $\pi=\pi_k$, $\bm{\lambda}=\bm{\lambda}_k$), we have $0 \leq \regLagrangian(\regPiStar, \bm{\lambda}_k) - \regLagrangian(\pi_k, \regLambdaStar)$. Moreover, recall $\Phi_k = \kl_k + \frac{1}{2} \norm{\bm{\lambda}_k - \regLambdaStar}^2$, thus by \cref{eq:rpg-i-bound-finite-learn,eq:rpg-ii-bound-finite-learn},
    \begin{align*}
        \Phi_{k+1} =& \kl_{k+1} + \frac{1}{2} \norm{\bm{\lambda}_{k+1} - \regLambdaStar}^2\\
        \leq& (1-\eta\tau) \kl_k + \frac{\eta^2}{2} \dTauLambda  + \eta(\valStartHat{\zHat_k}{k} - \valStart{z_k}{\pi_k})  - \eta(i)  \tag{\cref{eq:rpg-i-bound-finite-learn,eq:rpg-ii-bound-finite-learn}}\\
        & + (1-\eta\tau) \frac{\norm{\bm{\lambda}_k - \regLambdaStar}^2}{2} + \frac{\eta^2}{2} \dTauLambdaPrime + \eta \sum_i \lambdaMax \abs{ \valStartHat{\uHat_{k,i}}{k} - \valStart{u_i}{\pi_k} } - \eta(ii)\\
        \leq& (1-\eta\tau) \Phi_k + \eta^2(\dTauLambda + \dTauLambdaPrime)\\
        &+ \eta(\valStartHat{\zHat_k}{k} - \valStart{z_k}{\pi_k}) + \eta \sum_i \lambdaMax \abs{ \valStartHat{\uHat_{k,i}}{k} - \valStart{u_i}{\pi_k} }
        \tag{Def. $\Phi_k$}\\
        &- \eta \round{(i) + (ii)}  \\
        \leq& (1-\eta\tau) \Phi_k + \eta^2(\dTauLambda + \dTauLambdaPrime)\\
        &+ \eta(\valStartHat{\zHat_k}{k} - \valStart{z_k}{\pi_k}) + \eta \sum_{i} \lambdaMax \abs{ \valStartHat{\uHat_{k,i}}{\pi_k} - \valStart{u_i}{\pi_k} }
        \\
        &- \eta \round{\regLagrangian(\regPiStar, \bm{\lambda}_k) - \regLagrangian(\pi_k, \regLambdaStar)} \tag{\cref{eq:pd-decomp-finite-learn}}\\
        \leq& (1-\eta\tau) \Phi_k + \eta^2(\dTauLambda + \dTauLambdaPrime) \tag{as $\regLagrangian(\regPiStar, \bm{\lambda}_k) - \regLagrangian(\pi_k, \regLambdaStar) \geq 0$}\\
        &+ \eta(\valStartHat{\zHat_k}{k} - \valStart{z_k}{\pi_k}) + \eta \sum_i \lambdaMax \abs{ \valStartHat{\uHat_{k,i}}{k} - \valStart{u_i}{\pi_k} } .
    \end{align*}
    By induction and geometric series bound we find
    \begin{align*}
        \Phi_{k+1} \leq& (1-\eta\tau)^k \Phi_1 + \sum_{k'=1}^k (1-\eta\tau)^{k+1-k'} \eta^2 (\dTauLambda + \dTauLambdaPrime) \\
        &+ \sum_{k'=1}^k (1-\eta\tau)^{k+1-k'} \round{\eta(\valStartHat{z_{k'}}{k'} - \valStart{z_{k'}}{\pi_{k'}}) + \eta \sum_i \lambdaMax \abs{ \valStartHat{\uHat_{k',i}}{k'} - \valStart{u_i}{\pi_{k'}} }} \\
        \leq& (1-\eta\tau)^k \Phi_{1} + \frac{1}{\eta\tau} \eta^2 (\dTauLambda + \dTauLambdaPrime)\\
        &+ \eta\sum_{k'=1}^k \round{(\valStartHat{\zHat_{k'}}{k'} - \valStart{z_k}{\pi_{k'}}) + \sum_i \lambdaMax \abs{ \valStartHat{\uHat_{k',i}}{k'} - \valStart{u_i}{\pi_{k'}} }}\\
        \leq& (1-\eta\tau)^k \Phi_{1} + \frac{\eta}{\tau} (\dTauLambda + \dTauLambdaPrime)\\
        &+ \tilde{O} \round{\eta \round{(2+I\lambdaMax) + I \lambdaMax} \round{\sqrt{S^2AH^4 k} + S^{3/2}AH^2}},
    \end{align*}
    where we used \cref{lem:est-error-reg,lem:est-error-ru} (which apply since $G$ occurs) in the final step. Finally, noting that $\dTauLambda + \dTauLambdaPrime \leq \tilde{O}(\cTauLambda)$ (see \cref{lem:last-iterate}) and invoking $\lambdaMax\geq 1$ concludes the proof.
\end{proof}

Invoking \cref{lem:pot-to-regret-finite}, we can leverage the convergence of the potential function to show a sublinear regret bound for \cref{algo:rpg-dual-finite-learn}.

\rateFiniteLearn*

\begin{proof}
    The bound follows from \cref{lem:potential-finite-learn} and \cref{lem:pot-to-regret-finite}.

    Set $\Delta_r(k) := \rectangular{\valStart{r}{\piStar} - \valStart{r}{\pi_k}}_+$ and $\Delta_{g_i}(k) := \rectangular{- \valStart{g_i}{\pi_k}}_+$. Condition on the success event $G$, which happens with probability at least $1-\delta$ by \cref{lem:succ-event}.

    We first consider the regret for the reward. Plugging \cref{lem:potential-finite-learn} into \cref{lem:pot-to-regret-finite} we find, using $\sqrt{a+b} \leq \sqrt{a} + \sqrt{b}$ and $1+x\leq \exp(x)$,
    \begin{align*}
        \Delta_r(k) \lesssim& H^{3/2}\Phi_1^{1/2} \exp\round{-\eta\tau k/2} \tag{a}\\
        &+ H^{3/2} \round{\frac{\eta}{\tau}}^{1/2} \cTauLambda^{1/2} \tag{b}\\
        &+ H^{3/2} \round{ \eta \lambdaMax I SA^{1/2}H^2 k^{1/2} }^{1/2} \tag{c}\\
        &+ H^{3/2} \round{\eta \lambdaMax S^{3/2}AH^2}^{1/2}\tag{d}\\
        &+ \tau H\log(A). \tag{e}
    \end{align*}
    We first show that we can ignore terms (b), (d), and (e) since they are $o(K^{-13/14})$. For (b), note that, using the definitions of $\eta$, $\tau$, $\lambdaMax$ (and taking $\sqrt{\cdot}$, and $\tau < 1$)
    \begin{align*}
        \cTauLambda^{1/2} \leq& \lambdaMax H^{3/2}A^{1/4}I\cdot\exp\round{(\eta H \round{1 + \lambdaMax I +  \log(A)})/2} + I^{1/2} \round{ H + \tau \lambdaMax }\\
        \lesssim& \lambdaMax H^{3/2}A^{1/4}I\cdot\exp\round{ (HI)^{-1} \Xi K^{-5/7} \round{1 + \frac{H}{\Xi} K^{1/14}  I + \log(A)}} + I^{1/2} \round{ H + H\Xi^{-1} K^{-1/14}}\\
        \lesssim& \lambdaMax H^{3/2}A^{5/4}I\cdot \exp(2) + I^{1/2} \round{ H + H\Xi^{-1} K^{-1/14}} \\
        \lesssim& H^{5/2}A^{5/4}I\Xi^{-1}K^{1/14}.
    \end{align*}
    Since 
    \begin{align*}
        \round{\frac{\eta}{\tau}}^{1/2} = (H^2I)^{-1/2}\Xi^{1/2} K^{(-5/7 + 1/7)/2} = (H^2I)^{-1/2}\Xi^{1/2}K^{-2/7},    
    \end{align*}
    we thus have
    \begin{align*}
        (b) = H^{3/2} \round{\frac{\eta}{\tau}}^{1/2} \cTauLambda^{1/2} \leq \text{poly}(A,H,I,\Xi^{-1}) K^{-3/14}. 
    \end{align*}
    Hence, when summing (b) over $k=1,\dots,K$, it only contributes $o(K^{-13/14})$ to the regret.
    Similarly,
    \begin{align*} 
        (d) =& H^{3/2} \round{\eta \lambdaMax S^{3/2}AH^2}^{1/2}
        \leq H^{3/2} \round{S^{3/2}AH^2}^{1/2}K^{- (9/2)/14},\\
        (e) =& \tau H\log(A) = K^{-1/7} H \log(A).
    \end{align*}
    Hence, when summing (d) and (e) over $k=1,\dots,K$, they only contribute $o(K^{-13/14})$ to the regret.
    We now turn to terms (a) and (c). For (a), using the standard inequality $e^{-x} \leq 1-x/2$ (if $0\leq x \leq 1$) with $x := \eta\tau/2$, we first find 
    \begin{align*}
        \exp(-\eta\tau k/2) \leq (1 - \eta\tau /4)^k
    \end{align*}
    and hence, (after summing (a) over $k=1,\dots,K$)
    \begin{align*}
        \sum_{k=1}^K H^{3/2}\Phi_1^{1/2} \exp\round{-\eta\tau k/2} \leq& H^{3/2}\Phi_1^{1/2} \sum_{k=1}^K (1 - \eta\tau /4)^k\\
        \leq& H^{3/2}\Phi_1^{1/2}\sum_{k=1}^{\infty} (1 - \eta\tau /4)^k\\
        =& H^{3/2}\Phi_1^{1/2}\frac{4}{\eta\tau}\\
        =& H^{3/2}\Phi_1^{1/2} \frac{4}{(H^{2}I)^{-1}\Xi K^{-5/7} K^{-1/7}}\\
        =& 4H^{7/2} I \Xi^{-1} \Phi_1^{1/2} K^{12/14}.
    \end{align*}
    Furthermore, since $\pi_1$ plays actions uniformly at random and $\bm{\lambda_1}=\bm{0}$, we have $\Phi_1^{1/2} \leq (H\log(A) + \frac{1}{2}I\lambdaMax^2)^{1/2} \leq H^{1/2}\log(A)^{1/2} + I^{1/2}\lambdaMax = H^{1/2}\log(A)^{1/2} + I^{1/2}H\Xi^{-1} K^{1/14}$. Hence, the calculation above shows
    \begin{align*}
        (a) \leq& 4H^{7/2} I \Xi^{-1} \round{H^{1/2}\log(A)^{1/2} + I^{1/2}H\Xi^{-1} K^{1/14}} K^{12/14}\\
        \leq& 4H^{4} I \Xi^{-1}\log(A)^{1/2} K^{12/14} + 4H^{9/2} I \Xi^{-2} I^{1/2} K^{13/14}.
    \end{align*}
    For (c), we find (after summing over $k=1,\dots,K$)
    \begin{align*}
        \sum_{k=1}^K H^{3/2} \round{ \eta \lambdaMax I SA^{1/2}H^2 k^{1/2} }^{1/2}
        =& \sum_{k=1}^K H^{3/2} \round{ (H^2I)^{-1}\Xi K^{-5/7} H\Xi^{-1} K^{1/14}  I SA^{1/2}H^2 k^{1/2} }^{1/2}\\
        \leq& K \cdot H^{3/2} \round{ (H^2I)^{-1}\Xi K^{-5/7} H\Xi^{-1} K^{1/14}  I SA^{1/2}H^2 K^{1/2} }^{1/2}\\
        =& H^{3/2} \round{ (H^2I)^{-1}\Xi H\Xi^{-1}  I SA^{1/2}H^2 }^{1/2} K^{1 + (-5/7 + 1/14 + 1/2)/2}\\
        =& H^{3/2} \round{ SA^{1/2}H }^{1/2} K^{13/14}\\
        =& H^{2}S^{1/2}A^{1/4} K^{13/14}.
    \end{align*}
    Hence, summing up terms (a) to (e) over $k=1,\dots,K$ indeed yields the bound for the objective (using $13/14 \simeq 0.93$):
    \begin{align*}
        \Reg(K; r) \lesssim \round{H^{9/2} I \Xi^{-2} I^{1/2} + H^{2}S^{1/2}A^{1/4}}  K^{0.93}.
    \end{align*}

    Next, we consider the regret for the constraints. Plugging \cref{lem:potential-finite-learn} into \cref{lem:pot-to-regret-finite} we find, using $\sqrt{a+b} \leq \sqrt{a} + \sqrt{b}$,
    \begin{align*}
        \Delta_{u_i}(k) \lesssim& H^{3/2}\Phi_1^{1/2} \exp\round{-\eta\tau k/2} \tag{a'}\\
        &+ H^{3/2} \round{\frac{\eta}{\tau}}^{1/2} \cTauLambda^{1/2} \tag{b'}\\
        &+ H^{3/2} \round{ \eta \lambdaMax I SA^{1/2}H^2  k^{1/2} }^{1/2} \tag{c'}\\
        &+ H^{3/2} \round{\eta \lambdaMax S^{3/2}AH^2}^{1/2}\tag{d'}\\
        &+ \tau \lambdaMax + \frac{1}{\lambdaMax}H^2 \Xi^{-1} \tag{e'}\\
        &+ \frac{1}{\lambdaMax} \tau H\log(A). \tag{f'}
    \end{align*}
    Note that terms (a'), (b'), (c'), (d') are identical to (a), (b), (c), (d). Summed up, they thus correspond to the same regret as the one for the reward. Moreover, for (f') we have
    \begin{align*}
        (f') = \frac{1}{\lambdaMax} \tau\log(A) =& \Xi H^{-1} \log(A) K^{-3/14},
    \end{align*}
    and thus, we can ignore (f') when summing from $k=1,\dots,K$. Finally, for (e'), we have 
    \begin{align*}
        (e') =& \tau \lambdaMax + \frac{1}{\lambdaMax}H^2 \Xi^{-1}\\
        =& K^{-1/7} \cdot H\Xi^{-1} K^{1/14} + H^{-1}\Xi K^{-1/14} \cdot  H^2 \Xi^{-1}\\
        =& H(1 + \Xi^{-1}) K^{-1/14}.
    \end{align*}
    Thus, summing up (a') to (f') from $k=1$ to $K$ yields the regret for the constraints (using $13/14 \simeq 0.93$):
    \begin{align*}
        \Reg(K; \bm{u}) \lesssim \round{H^{9/2} I \Xi^{-2} I^{1/2} + H^{2}S^{1/2}A^{1/4} + H(1 + \Xi^{-1})}  K^{0.93}.
    \end{align*}
\end{proof}

\section{Examples and Simulation} \label{sec:further-exp}

In this section, we expand on the description provided in \cref{sec:experiments}.

\subsection{Examples}

In this section, we provide examples to highlight crucial differences between the strong regret $\mathcal{R}$ and the weak regret $\mathcal{R}_{\text{weak}}$.

\paragraph{A Minimal Example} Consider the case in which an agent repeatedly has the option between three different investment strategies $(o_1, o_2, o_3)$. Each of them yields a respective reward $(r_1, r_2, r_3)=(1/5, 1, 3/5)$. There is an initial cost $(u_1, u_2, u_3) = (9/10, 1/10, 5/10)$ associated with each option, which must not exceed the budget $c = 1/2$ of the agent. If the budget is exceeded, the agent will be in debt, and a larger debt is associated with a higher risk. Clearly, we can model this scenario as an optimization problem $\max_{x \in \simplex{[3]}} \sum_{i\in[3]} x_i r_i$, subject to $\sum_{i\in[3]} x_i r_i \leq c$, where $x$ describes the distribution of investments in the respective strategies. We can equivalently model this problem as a CMDP with one state $s_1$, horizon $H=1$ and three actions $\calA=[3]$  (i.e., a constrained bandit), in which $r_1(s_1, a) = r_a$, $\bm{u}_1(s_1,a) = 1-u_a$ ($a\in[3]$) and $\bm{c} = 1-c$. Here, $o_1$ is highly profitable but too risky, $o_2$ is less profitable but safe, and $o_3$ is a compromise between both. Note that $\regret(K,\bm{u})$ now measures the total amount of debt the algorithm accumulated during $K$ episodes. A strategy $A$ that always plays $o_3$ will have a total debt of $0$ since $o_3$ does not violate the constraint. On the other hand, a strategy $B$ that plays $o_1$ and $o_2$ in an alternating fashion will have a total debt of $K/5$, which is linear in $K$, despite having a weak regret of $0$ due to the aforementioned cancellations. Both strategies, $A$ and $B$, have the same accumulated objective. This simple example, which does not even require an unknown environment, illustrates why weak regret cannot be the right notion of safety during learning. We point to \citet{calvo2023state,moskovitz2023reload,stooke2020responsive} for other examples exhibiting similar behavior.

\paragraph{A Slight Relaxation}
For practical purposes, one may consider the strong regret only for the constraint violations and the weak one for the objective. The reasoning behind this possible relaxation is that we may tolerate superoptimal performance with respect to the reward (when the algorithm violates the constraints) while we still do not tolerate additive negative constraint violations (as discussed). For example, an agent maximizing a wealth function may be allowed to obtain a higher return than any safe method, adding a negative term to the objective regret. In this case, we may want to allow compensating suboptimal returns by superoptimal ones, while we do not want to allow compensating unsafe strategies by strictly safe ones. However, this relaxation does not improve our theoretical results in their current form. It remains open whether, under this relaxation, stronger results are possible to obtain.

We argue that despite the possible relaxation, it is sensible to require the strong regret for both the violation and the objective, as done in this paper, which is also what \citet{efroni2020exploration} referred to. Indeed, one can think of settings in which one gets paid out the return of an episode only up until the limit that would be attainable subject to the constraints. For instance, if there is an illegal set of options, the controller of the environment may decide to pay out only so much as attainable when they are not being used (but may not know whether the illegal actions have, in fact, been used). In other words, the return is $W_k = \min\{\valStart{r}{\pi_k}, \valStart{r}{\piStar}\}$. Hence, per-episode regret would be $$\valStart{r}{\piStar} - W_k =\valStart{r}{\piStar} - \min\{\valStart{r}{\pi_k}, \valStart{r}{\piStar}\} = [\valStart{r}{\piStar} - \valStart{r}{\pi_k}]_{+},$$ which is exactly the strong regret of the episode.

\subsection{Experiment Details}

In this section, we report all details of the parameters, environments, and hardware used for our simulations.

\textbf{Hyperparameters} For the vanilla algorithms, we run for $K=4000$ episodes for each step size $\eta \in \{ 0.05, 0.075, 0.1, 0.125, 0.15, 0.2 \}$, which we observed to be a reasonable range across CMDPs when fixing the number of episodes. Similarly, for the regularized algorithms, we perform the same parameter search across all pairs of step size $\eta \in \{ 0.05, 0.1, 0.2 \}$ and regularization parameter $\tau \in \{ 0.01, 0.02 \}$, totaling a number of six hyperparameter configurations as well. We always set $\lambdaMax = 6$, which did not play a role in our simulations as long as it was chosen sufficiently large. We use exploration bonuses $0.08 \cdot n_h(s,a)^{-1/2}$. For each hyperparameter configuration, we sample $n=5$ independent runs with $K$ episodes, obtain the regret curves and plot their average. For each algorithm, we then report the result for the best hyperparameter configuration in hindsight (with respect to the strong regrets).

\textbf{Environment} As described, we sample the rewards $r$ uniformly at random and the constraints as $c= (1-r) + \beta \zeta $, for a Gaussian vector $\eta \in \reals^{HSA}$ and $\beta = 0.1$.  We sample an environment with $S=A=H=5$ according to the procedure above. Throughout seeds (and other CMDP sizes), this led to CMDPs in which the oscillations of the iterates and error cancellations can be observed. As argued, sampling constraints and rewards fully independently does not provide CMDPs that are interesting test beds. Indeed, unlike in random CMDPs, in real-life situations, we often observe goals that are explicitly conflicting with safety constraints, as otherwise, there is no need to encode them via a CMDP. For instance, consider a vehicle that aims to arrive fast but not go over the speed limit or cross the sidewalk. The latter would be the fastest option, but it conflicts with the constraint. In other words, the constraint and the reward function are negatively correlated.

All simulations were performed on a MacBook Pro 2.8 GHz Quad-Core Intel Core i7. We provide the code in the supplementary material. For all experiments, we set the seed to $123$.

\subsection{Dual Approach} \label{sec:dual}

In this section, we provide the pseudo-code for the regularized dual algorithm. One approach to solving the regularized min-max CMDP problem is to perform dual descent on the regularized Lagrangian. Evaluating the dual function then corresponds to entropy-regularized dynamic programming in an MDP, and the descent step on the dual function corresponds to a mirror descent step of the Lagrange multipliers on the regularized Lagrangian. Replacing the required value functions in this scheme with the optimistically estimated ones yields \cref{algo:rpg-dual-finite-learn}. Note that here, we consider the estimated value functions without truncation. Set $\hat{\entropy}_k(\pi) := \valStart{\psi_{\pi,k}}{\pi}$, where $\psi_{\pi,k,h}(s,a) := -\log(\pi_h(a|s))+ b^p_{k-1,h}(s,a) \log(A)$.
\begin{algorithm}[H] 
	\begin{algorithmic}
        \REQUIRE{$\Lambda = [0,\lambdaMax]^I$, stepsize $\eta>0$, regularization parameter $\tau > 0$, number of episodes $K$, initial policy $\pi_{1,h}(a | s) = 1/A$ ($\forall s,a,h$), $\bm{\lambda}_1 := \bm{0} \in \reals^I$}
        \vspace{0.2cm}
        \FOR{$ k=1,\dots, K$}
        \STATE{Update primal variable via regularized dynamic programming} 
        \begin{align*}
            \pi_{k} &= \arg\min_{\pi \in \Pi} \round{ \valStart{\rHat_k}{\pHat_k, \pi}  + \bm{\lambda}_k^T \valStartVec{\bm{\uHat}_k}{\pHat_k, \pi} + \tau \hat{\entropy}_k(\pi) }
        \end{align*}\vspace{-0.4cm}
        \STATE{Evaluate $\valStartVec{\bm{\uHat}_k}{\pHat_k,\pi_k}$}
        \STATE{Update dual variables:}
            \begin{align*}
                \bm{\lambda}_{k+1} &= \proj{\Lambda}{\bm{\lambda}_k - \eta (\valStartVec{\bm{\uHat}_k}{\pHat_k, \pi_k} - \bm{c} + \tau \bm{\lambda}_k) }.
            \end{align*}
        \STATE{Play $\pi_k$ for one episode, update model estimates: $\rHat_{k+1}$, $\bm{\uHat}_{k+1}$, $\pHat_{k+1}$, $\psiHat_{k+1}$}
        \ENDFOR    
	\caption{Regularized Dual Algorithm with Optimistic Exploration} 
    \label{algo:rpg-dual-finite-learn}
	\end{algorithmic}
\end{algorithm}

\section{Difference Lemmas}

In the following, we recap the well-known performance difference and value difference lemma.

\begin{lemma}[Performance difference lemma] \label{lem:pdl-finite}
    For all $\pi, \pi' \in \Pi$ and all $r' \colon \calS \times \calA \to \reals$, we have
    \begin{align*}
        \valStart{r'}{\pi}  - \valStart{r'}{\pi'} =& \exptn_{\pi} \rectangular{ \sum_{h=1}^H \sum_a \qval{r',h}{\pi'}{s_h,a} \round{ \pi_h(a|s_h) - \pi'_h(a|s_h) }}\\
        =& \sum_{h=1}^H \sum_{s} \occ{h}{\pi}{s} \sum_a \qval{r',h}{\pi'}{s,a} \round{ \pi_h(a|s) - \pi'_h(a|s) }\\
        =& \sum_{h=1}^H \sum_{s} \occ{h}{\pi}{s} \langle \qval{r',h}{\pi'}{s,\cdot}, \pi_h(\cdot|s) - \pi'_h(\cdot|s) \rangle.
    \end{align*}
\end{lemma}

\begin{proof}
    See \citet[Lemma 3.2]{cai2020provably} for the first equality. The second equality follows since we consider unnormalized occupancy measures. The third equality holds by definition of the inner product.
\end{proof}

From \citet[Lemma 1]{shani2020optimistic}:

\begin{lemma}[Extended value difference lemma (aka simulation lemma)] \label{lem:ext-val-diff}
    Let $\pi$, $\pi'$ be policies and $M = (\calS,\calA,p,r)$, $M' = (\calS,\calA,p',r')$ be MDPs. Let $\qvalHat{h}{M}{s,a}$ be an approximation of the value function $\qval{r,h}{p,\pi}{s,a}$. Let $\valHat{h}{M}{s} := \dotProd{\qvalHat{h}{M}{s,\cdot}}{\pi_h(\cdot|s)}$. Then
    \begin{align*}
        &\valHat{ }{M}{s_1,1} - \val{r',1}{p',\pi'}{s_1} \\
        =& \sum_{h=1}^H \exptn\rectangular{ \dotProd{\qvalHat{h}{M}{s_h,\cdot}}{\pi_h(\cdot|s_h) - \pi'_h(\cdot|s_h)} \biggVert s_1; p', \pi'}\\
        &+ \sum_{h=1}^H \exptn\rectangular{\qvalHat{h}{M}{s_h,a_h} - r'_h(s,a) - \dotProd{p'_h(\cdot|s_h,a_h)}{\valHat{h+1}{M}{\cdot}}  \biggVert s_1; p', \pi'},
    \end{align*}
    where $\val{r',1}{p',\pi'}{s_1}$ is the value function of $\pi'$ in $M'$.
\end{lemma}

Note that $\hat{Q}$ need not correspond to a true value function under some model.

\section{Convex Optimization Background} \label{app:convopt}

In this section, we review fundamental results from the optimization literature. All of these results are standard, and we include them for completeness. They are not novel by themselves nor specific to the sections in which we make use of them.

\subsection{Convex Min-Max Optimization} \label{app:minmax}

While the following results from min-max optimization are commonly used, we establish them here for our setup (both for completeness and due to the lack of a unifying resource for our case).

\paragraph{Setup} 

Let $\calX \subset \reals^{d_x}$, $\calY \subset \reals^{d_y}$ be (nonempty) compact convex sets and let $f \colon \calX \times \calY \to \reals$ be a continuous and convex-concave function. Set $\fBar \colon \calX \to \reals$, $\fBar(x) := \max_{y \in \calY} f(x,y)$, and $\fUnderline \colon \calY \to \reals$, $\fUnderline(y) := \min_{x \in \calX} f(x,y)$, which both exist by continuity of $f$ on a compact domain.

\begin{lemma}[Existence minimax points] \label{lem:exist-minmax}
    We have 
    \begin{align}
        \inf_{x \in \calX} \max_{y \in \calY}~ f(x,y) = \sup_{y \in \calY} \min_{x \in \calX}~ f(x,y),
    \end{align}
    and the maximum and minimum are attained, i.e., there exist $\xStar \in \calX$, $\yStar \in \calY$ such that
    \begin{align}
        \fBar(\xStar) =& \inf_{x \in \calX} \max_{y \in \calY}~ f(x,y),\\
        \fUnderline(\yStar) =& \sup_{y \in \calY} \min_{x \in \calX}~ f(x,y).
    \end{align}
\end{lemma}

\begin{proof}
    The first equality holds due to Sion's Minimax Theorem \citep{sion1958general}. Note that the second part of the lemma is not immediate from Sion-like statements.

    We shall prove that $\fBar$ is continuous. By compactness of $\calX$, this implies the existence of $\xStar$. By symmetry, this also settles the existence of $\yStar$ (by repeating the argument for $-f$). Thus, let $x \in \calX$ and consider a sequence $(x_k)_k$ in $\calX$ such that $x_k \to x$. We aim to show that $\fBar(x_k) \to \fBar(x)$, which would conclude the proof.

    Let $y \in \arg\max_{y' \in \calY} f(x,y)$, which exists by continuity. Then, for every $k$ we have 
    \begin{align*}
        \fBar(x_k) = \max_{y' \in\calY} f(x_k, y') \geq f(x_k, y) \to f(x,y).
    \end{align*}
    Taking $\lim\inf_k$ on both sides yields 
    \begin{align}
        \lim\inf_{k} \fBar(x_k) \geq f(x,y) = \fBar(x).
    \end{align}
    Assume by contradiction that $\lim \sup_{k} \fBar(x_k) > \fBar(x)$. Then we can pick $\delta > 0$ such that $\lim\sup_{k} \fBar(x_k) \geq \fBar(x) + \delta$. Thus we can pick a subsequence $x_{n(k)}$ and $y_{n(k)} \in \arg\max_{y' \in \calY} f(x_{n(k)}, y')$ such that for all $k$, 
    \begin{align}
        \fBar(x_{n(k)}) \geq \fBar(x) + \delta /2\geq f(x,y_{n(k)}) + \delta/2. \label{eq:eps-delta1}
    \end{align}
    Since $\calY$ is compact, by further picking a subsequence if needed, we can WLOG assume that there exists $\tilde{y} \in \calY$ such that $y_{n(k)} \to \tilde{y}$. Then by \cref{eq:eps-delta1},
    \begin{align*}
        f(x_{n(k)},y_{n(k)}) =& \fBar(x_{n(k)}) \geq f(x,y_{n(k)}) + \delta/2.
    \end{align*}
    Taking $k \to \infty$ and using continuity of $f$ yields the contradiction $f(x,\tilde{y}) \geq f(x,\tilde{y}) + \delta/2$. We therefore must have $\lim \sup_{k} \fBar(x_k) \leq \fBar(x) \leq \lim\inf_{k} \fBar(x_k)$, proving $\fBar(x_k) \to \fBar(x)$. Thus, $\fBar$ is indeed continuous.
\end{proof}

\paragraph{General Setup}

The statements in this paragraph all concern the following more general setup (dropping convex-concavity and boundedness of the domain). As we showed in the previous paragraph, all assertions made here hold in the continuous, convex-concave compactly constrained setup. 

Let $\calX \subset \reals^{d_x}$, $\calY \subset \reals^{d_y}$ be (nonempty) closed sets and let $f \colon \calX \times \calY \to \reals$ be a continuous function. Consider $\fBar \colon \calX \to \reals \cup \curly{\pm\infty}$, $\fBar(x) := \max_{y \in \calY} f(x,y)$ and $\fUnderline \colon \calY \to \reals \cup \curly{\pm\infty}$, $\fUnderline(y) := \min_{x \in \calX} f(x,y)$.

\begin{lemma}[Min-max to saddle point] \label{lem:minmax-to-sp}
    Let 
    \begin{align*}
        \xStar& \in \arg\min_{x \in \calX} \max_{y \in \calY}~ f(x,y),\\
        \yStar& \in \arg\max_{y \in \calY} \min_{x \in \calX}~ f(x,y),
    \end{align*}
    and assume $\fBar(\xStar)=\fUnderline(\yStar)$. Then $(\xStar,\yStar)$ is a \emph{saddle point}, i.e., for all $x \in \calX$, $y\in\calY$, we have
    \begin{align*}
        f(\xStar,y) \leq f(\xStar,\yStar) \leq f(x,\yStar).
    \end{align*}
\end{lemma}

\begin{proof}
    We have
    \begin{align*}
        f(\xStar,y) \leq \max_{y'\in\calY} f(\xStar,y') = \fBar(\xStar) = \fUnderline(\yStar) = \min_{x'\in\calX} f(x',\yStar) \leq f(\xStar,\yStar),
    \end{align*}
    proving the first inequality, and for the second, we have
    \begin{align*}
        f(x,\yStar) \geq \min_{x'\in\calX} f(x',\yStar) = \fUnderline(\yStar) = \fBar(\xStar) = \max_{y'\in\calY} f(\xStar,y') \geq f(\xStar,\yStar).
    \end{align*}
    Note that this proof does not require convexity or compactness.
\end{proof}

\begin{lemma}[Saddle point to min-max]\label{lem:sp-to-minmax}
    Let $(\xStar,\yStar)$ be a saddle point, i.e., for all $x \in \calX$, $y\in\calY$, we have
    \begin{align*}
        f(\xStar,y) \leq f(\xStar,\yStar) \leq f(x,\yStar).
    \end{align*}
    Then 
    \begin{align*}
        \xStar& \in \arg\min_{x \in \calX} \max_{y \in \calY}~ f(x,y),\\
        \yStar& \in \arg\max_{y \in \calY} \min_{x \in \calX}~ f(x,y).
    \end{align*}
\end{lemma}

\begin{proof}
    We first note that the assertion implies $\max_{y'} f(\xStar, y') \leq \min_{x'\in\calX} f(x',\yStar)$. Hence
    \begin{align*}
        \fBar(\xStar) = \max_{y'} f(\xStar, y') \leq \min_{x'\in\calX} f(x',\yStar) \leq \min_{x'\in\calX} \max_{y'\in\calY} f(x',y'),
    \end{align*}
    showing the claim for $\xStar$. For $\yStar$, we note that similarly,
    \begin{align*}
        \fUnderline(\yStar) = \min_{x'\in\calX} f(x',\yStar) \geq \max_{y' \in\calY} f(\xStar, y') \geq \max_{y'\in\calY} \min_{x'\in\calX} f(x', y'),
    \end{align*}
    concluding the proof.

    Note that this proof does not require convexity or compactness.
\end{proof}

\begin{lemma}[Invariance of saddle points]\label{lem:sp-invar}
    Let
    \begin{align*}
        \xStar& \in \arg\min_{x \in \calX} \max_{y \in \calY}~ f(x,y),\\
        \yStar& \in \arg\max_{y \in \calY} \min_{x \in \calX}~ f(x,y),
    \end{align*}
    and assume $\fBar(\xStar)=\fUnderline(\yStar)$. Consider closed sets $\calX' \subset \calX$, $\calY' \subset \calY$. If $(\xStar,\yStar) \in \calX'\times\calY'$, then
    \begin{align*}
        \xStar& \in \arg\min_{x \in \calX'} \max_{y \in \calY'}~ f(x,y),\\
        \yStar& \in \arg\max_{y \in \calY'} \min_{x \in \calX'}~ f(x,y).
    \end{align*}
\end{lemma}

\begin{proof}
    By \cref{lem:minmax-to-sp} (which applies since $\fBar(\xStar)=\fUnderline(\yStar)$), $(\xStar,\yStar)$ is a saddle point for the minmax problem with domain $\calX\times\calY$. Thus, since $\yStar\in\calY' \subset \calY$, we have $f(\xStar,\yStar) = \max_{y\in\calY} f(\xStar,y) = \max_{y\in\calY'} f(\xStar,y)$. Moreover, since $\calX' \subset \calX$ and $\yStar\in\calY'$, we have $f(\xStar,\yStar) = \min_{x\in\calX} f(x,\yStar) \leq \min_{x\in\calX'} f(x,\yStar) \leq \min_{x\in\calX'} \max_{y\in\calY'} f(x,y)$. Hence $\max_{y\in\calY'} f(\xStar,y) \leq \min_{x\in\calX'} \max_{y\in\calY'} f(x,y)$, proving $$\xStar \in \arg\min_{x \in \calX'} \max_{y \in \calY'}~ f(x,y).$$ The proof for $\yStar$ follows by repeating the argument for $-f$.
    
    Note that this proof does not require convexity or compactness.
\end{proof}

\subsection{Constrained Convex Optimization} \label{aug-sec:conv-prelim}

We state some well-known results from constrained convex optimization that will be useful. The results are standard, and we refer, for example, to the work by \citet{beck2017first}. 

Consider the (primal) optimization problem
\begin{align}
    f^* := \min ~~~ & f(x) \nonumber\\
    \text{s.t.} ~~~ & g(x) \leq 0 \label{aug-eq:opt-P}\\
    & x \in X \nonumber
\end{align}
with the following assumptions.

\begin{assumption}[Assumption 8.41, \citet{beck2017first}] \label{aug-ass:841}
    In \cref{aug-eq:opt-P},
    \begin{itemize}
        \item[(a)] $X \subset \R^n $ is convex
        \item[(b)] $f\colon \R^n \to \R$ is convex
        \item[(c)] $g(\cdot) := (g_1(\cdot), \dots, g_m(\cdot)) ^T$ with $g_i \colon \R^n \to \R$ convex
        \item[(d)] \cref{aug-eq:opt-P} has a finite optimal value $f^*$, which is attained by exactly the elements of $X^* \neq \emptyset$ 
        \item[(e)] There exists $\xbar \in X$ such that $g(\xbar) < 0$
        \item[(f)] For all $\lambda \in \R_{\geq 0}^m$, $\min_{x\in X} (f(x) + \lambda^Tg(x))$ has an optimal solution
    \end{itemize}
\end{assumption}

In this setup, we define the \textit{dual objective} as
\begin{align*}
    q(\lambda) := \min_{x \in X} \left( f(x) + \lambda^T g(x) \right),
\end{align*}
where $\mathcal{L} \colon \R^n \times \R^m \to \R$, $\mathcal{L}(x;\lambda) := f(x) + \lambda^T g(x)$ is the \textit{Lagrangian} of the problem in \cref{aug-eq:opt-P}. The \textit{dual problem} is then defined as 
\begin{align*}
    q^* := \max ~~~ & q(\lambda)\\
    \text{s.t.} ~~~ & \lambda \geq 0.
\end{align*}
In this setup, we have the following results connecting the primal and the dual problem.

\begin{restatable}[Theorem A.1, \citet{beck2017first}]{theorem}{thmduality} \label{aug-thm:duality}
    Under Assumption \ref{aug-ass:841}, strong duality holds in the following sense: We have 
    \begin{align*}
        f^* = q^*
    \end{align*}
    and the optimal solution of the dual problem is attained, with the set of optimal solutions $\Lambda^* \neq \emptyset$.
\end{restatable}

\begin{proof}
     Proposition 6.4.4 of \citet{bertsekas2003convex} gives a proof of the more general Theorem A.1 of \citet{beck2017first}. We remark that if we assume affine constraints $g$ and $X$ being a polytope, then we can drop assumption (e) \citep[Theorem A.1]{beck2017first}. 
\end{proof}

\begin{restatable}[]{theorem}{thmconvauxineq}\label{aug-thm:conv-aux-ineq}
    Suppose Assumption \ref{aug-ass:841} holds. Let $x^*\in X^*$, $\lambda^* \in \Lambda^*$ and $x \in X$. Then 
    \begin{align*}
        f(x) - f(x^*) + (\lambda^*)^T g(x) \geq 0.
    \end{align*}
\end{restatable}

\begin{proof}
    We have 
    \begin{align*}
        f(x) =& f(x) + (\lambda^*)^T g(x) - (\lambda^*)^T g(x)&\\
        \geq& q(\lambda^*)- (\lambda^*)^T g(x) & \text{(definition of } q(\cdot)\text{)}\\
        =& f(x^*) - (\lambda^*)^T g(x)& \text{(since by \cref{aug-thm:duality}, } q^*=f^*\text{)}
    \end{align*}
    and rearranging this proves the claim. Again, we see that we can drop assumption (e) if we consider affine constraints $g$ and a polytope $X$.
\end{proof}

\begin{restatable}[]{theorem}{thmconvdual} \label{aug-thm:conv-dual}
    Under Assumption \ref{aug-ass:841}, for all $\lambda^* \in \Lambda^*$ and $\xbar$ as in (e), we have 
    \begin{align*}
        \| \lambda^* \| \leq \|\lambda^* \|_1 \leq \frac{f(\xbar)  - f^*}{\min_{i\in [m]} (-g_i(\xbar))}.
    \end{align*}
\end{restatable}

\begin{proof}
    The first relation holds since $\lambda^* \geq 0$. We show the second relation as follows (cf. \citet[Theorem 8.42]{beck2017first}). We have 
    \begin{align*}
        f(x^*) =& q(\lambda^*) &\text{(\cref{aug-thm:duality})}\\
        \leq& f(\xbar) + (\lambda^*)^T g(\xbar) &\text{(definition of } q(\cdot)\text{)}\\
        \leq& f(\xbar) + \|\lambda^*\|_1 \max_{i\in[m]} g_i(\xbar) &\text{(since } \lambda^* \geq 0\text{)}\\
        =& f(\xbar) - \|\lambda^*\|_1 \min_{i\in[m]} (-g_i(\xbar)) & 
    \end{align*}
    and rearranging this proves the claim. We remark that for this theorem, we do need assumption (e), even in the affine case. 
\end{proof}

\subsection{Online Mirror Descent} \label{sec:omd} 

\paragraph{Setup} In the following, we consider a convex set $X \subset \reals^d$ and a non-empty closed convex set $V \subset X$. Let $\psi \colon X \to \reals$ be proper, closed, and strictly convex on $V$. Let $B_{\psi} \colon X \times \interior{X} \to \reals$ be the \emph{Bregman divergence} associated with $\psi$. Define $\norm{x}_A := \sqrt{x^TAx}$. 
Assume that
\begin{align}
    &\lim_{\lambda \to 0} \round{\nabla \psi (x + \lambda(y-x))}^T (y - x) = - \infty \quad (\forall x \in \text{bdry}(X),~ y \in \interior{X} ),\text{ or} \label{eq:md-cond1}\\ 
    &V \subset \interior{X} \label{eq:md-cond2}.
\end{align}

Consider the following descent lemma using local norms.
\begin{lemma}[MD descent lemma, \citet{orabona2019modern} Lemma 6.16] \label{lem:md-descent-line}
    Suppose $\psi$ is twice differentiable, with positive definite Hessian in the interior of its domain. Assume 
    \begin{align}
        \xTilde \in \arg\min_{\xBar \in X} g^T \xBar + \frac{1}{\eta} B_{\psi}(\xBar, x), \label{eq:md-update-unnormalized}\\
        x' \in \arg\min_{\xBar \in V} g^T \xBar + \frac{1}{\eta} B_{\psi}(\xBar, x) \label{eq:md-update1}
    \end{align}
    exist. Then, for all $\xStar \in V$, there exist $z$ on the line segment\footnote{The line segment between two vectors is the convex hull of the set containing those two vectors.} between $x$ and $x'$, and $z'$ on the line segment between $x$ and $\xTilde$ such that
    \begin{align*}
         \eta g^T (x - \xStar) \leq \bregmanPsi{\xStar}{x} - \bregmanPsi{\xStar}{x'} + \frac{\eta^2}{2} \min \curly{\norm{g}_{\round{\nabla^2 \psi (z)}^{-1}}^2, \norm{g}_{\round{\nabla^2 \psi (z)}^{-1}}^2}.
    \end{align*}
\end{lemma}

We get the following descent lemma for exponentiated Q-ascent.

\begin{lemma} \label{lem:md-descent-kl}
     Let $V:= \simplex{[d]}$, and $g \in \reals_{\geq 0}^d =: X$. Then $\xTilde := \arg\max_{\xBar \in X} g^T \xBar - \frac{1}{\eta} \kl(\xBar,x)$ and $\arg\max_{\xBar \in V} g^T \xBar - \frac{1}{\eta} \kl(\xBar,x)$ exist and are unique. Moreover, if $g$ only has non-negative entries, then for all $\xStar \in V$ we have 
    \begin{align*}
        g^T (\xStar - x) \leq \frac{\kl(\xStar,x) - \kl(\xStar,x')}{\eta} + \frac{\eta}{2} \sum_{i=1}^d \xTilde_i g_i^2.
    \end{align*}
\end{lemma}

\begin{proof}
    Note that the negative entropy $\psi(x) = \sum_i x_i\log(x_i)$ is strictly convex and twice differentiable and satisfies \cref{eq:md-cond1}, as a short calculation reveals. Moreover, for $p,q \in V$, we have $\bregmanPsi{p}{q} = \kl(p,q)$ \citep[Example 6.4]{orabona2019modern}. Existence and uniqueness are discussed in \citet{orabona2019modern}. 

    Maximizing $g^T \xBar - \frac{1}{\eta} \kl(\xBar,x)$ is equivalent to minimizing $(-g)^T \xBar + \frac{1}{\eta} \kl(\xBar,x)$, allowing us to apply \cref{lem:md-descent-line}. Note that for $z \in \reals_{>0}^I$, $\nabla^2 \psi(z) = \text{diag}(1/z_1, \dots, 1/z_d)$. Thus, $\norm{-g}_{\round{\nabla^2 \psi(z)}^{-1}} = \sum_i z_i g_i^2$. Moreover, we have $\xTilde_i = x \exp\round{- \eta (-g_i)} \geq x_i$ \citep{orabona2019modern}, since $g_i \geq 0$, and thus by \cref{lem:md-descent-line}, we can pick $z' \leq \xTilde$ (componentwise) such that
    \begin{align*}
        g^T (\xStar - x) =& (-g)^T (x - \xStar) \\
        \leq& \frac{\kl(\xStar,x) - \kl(\xStar,x')}{\eta} + \frac{\eta}{2} \sum_{i=1}^d z'_i g_i^2\\
        \leq& \frac{\kl(\xStar,x) - \kl(\xStar,x')}{\eta} + \frac{\eta}{2} \sum_{i=1}^d \xTilde_i g_i^2.
    \end{align*}
\end{proof}

\begin{lemma}[MD descent lemma, cf. \citet{orabona2019modern} Lemma 6.9] \label{lem:md-descent-inf}
    Let $x \in V$, $g \in \reals^d$, and $\eta > 0$. Assume further that $\psi$ is $\mu$-strongly convex w.r.t. some norm $\norm{\cdot}_{\reals^d}$ in $V$. Then,
    \begin{align}
        x' = \arg\min_{\xBar \in V} g^T \xBar + \frac{1}{\eta} B_{\psi}(\xBar, x) \label{eq:md-update}
    \end{align}
    exists and is unique. Moreover, for all $\xStar \in V$, the following inequality holds:
    \begin{align*}
        \eta g^T (x - \xStar) \leq \bregmanPsi{\xStar}{x} - \bregmanPsi{\xStar}{x'} + \frac{\eta^2}{2 \mu} \norm{g}_{*}^2,
    \end{align*}
    where $\norm{\cdot}_{*}$ is the \emph{dual norm} associated with $\norm{\cdot}_{\reals^d}$.
\end{lemma}

We can deduce the descent lemma for projected gradient descent.

\begin{lemma}[Descent lemma PGD] \label{lem:md-descent-proj}
    Let $x \in V$, $g \in \reals^d$, and $\eta > 0$. Then $x' := \arg\min_{\xBar \in V} \xBar^T g + \frac{1}{2\eta} \norm{\xBar - x}^2$ exists and is unique. Moreover, for all $\xStar \in V$ we have 
    \begin{align*}
        g^T(x-\xStar) \leq \frac{\norm{\xStar-x}^2 - \norm{\xStar - x'}^2}{\eta} + \frac{\eta}{2} \norm{g}^2.
    \end{align*}
\end{lemma}

\begin{proof}
    Note that $\psi(x) = \frac{1}{2}\norm{x}^2$ is $1$-strongly convex w.r.t. the L2 norm. Moreover, for $a,b \in V$, we have $\bregmanPsi{a}{b} = \frac{1}{2}\norm{a-b}^2$ \citep[Example 6.4]{orabona2019modern}. Since the L2 norm is its own dual norm, applying \cref{lem:md-descent-inf} to the minimization of $g^T \xBar + \frac{1}{2\eta}\norm{\xBar - x}^2$ yields the claim.
\end{proof}

Finally, the following lemma shows that the updates of the algorithms indeed fall into the category of (online) mirror descent.
\begin{lemma}[\cite{orabona2019modern}] \label{lem:update-md-implicit}
    Consider a compact set $Y \subset \reals^{I}$ with $y\in Y$, and let $x \in \simplex{[d]}$ for some $d\in\mathbb{Z}_{\geq 1}$. Then, the closed-form expressions 
    \begin{align*}
        x'_i = x_i &= \frac{x_i \exp \round{\eta g_i}}{\sum_{i'\in [d]} x_{i'} \exp \round{\eta x_{i'}}} \quad (i\in[d]), \\
        y' &= \proj{Y}{y - \eta g } ,
    \end{align*}
    are the unique solutions to
    \begin{align*}
        \max_{\xBar \in \simplex{[d]}}&~ \xBar^T g - \frac{1}{\eta} \kl(\xBar,x),\\
        \min_{\yBar \in Y}&~ \yBar^T g + \frac{1}{2\eta} \norm{\yBar - y}^2,
    \end{align*}
    respectively.
\end{lemma}

\begin{proof}
    For the primal variable, the derivation of exponentiated gradient is standard, see, e.g., \citet[Section 6.6]{orabona2019modern}. 
    
    For the dual variable, the derivation of projected gradient descent simply follows from the first-order optimality criterion and convexity of the objective.
\end{proof}

\end{document}